\documentclass[11pt]{article}    
\usepackage[margin=1in]{geometry}
\usepackage{setspace}

\usepackage{amsthm,amsmath, amssymb, amsfonts, amscd, xspace, pifont,mathtools}
\usepackage{graphicx}
\usepackage{algorithm,algorithmic}
\usepackage{float}
\usepackage{epsfig, psfrag, pstricks}
\usepackage{float, fancybox, fullpage}
\usepackage{delarray}
\usepackage{hyperref}
\usepackage{url}
\usepackage{anysize}
\usepackage{multirow}
\usepackage{latexsym,ifthen}
\usepackage{subfigure}
\usepackage{anysize}
\usepackage[round]{natbib} 

\newcommand{\xddots}{%
	\raise 4pt \hbox {.}
	\mkern 6mu
	\raise 1pt \hbox {.}
	\mkern 6mu
	\raise -2pt \hbox {.}
}
\newcommand{\ze}{Z\kern-0.45emZ}
\newcommand{\esp}{I\kern-0.37emE}
\newcommand{\N}{I\kern-0.37emN}
\newcommand{\real}{I\kern-0.37emR}

\newtheorem{proposition}{Proposition}
\newtheorem{theorem}{Theorem}
\newtheorem{lemma}{Lemma}

\newtheorem{example}{Example}

\newtheorem{remark}{Remark}
\newtheorem{assumption}{Assumption}

\usepackage{mathrsfs}

\title{Distribution-free Contextual Dynamic Pricing}

\author
{
	Yiyun Luo\thanks{PhD Student, Department of Statistics and Operations Research, The University of North Carolina at Chapel Hill. Email: yiyun851@live.unc.edu.},
	Will Wei Sun\thanks{Assistant Professor, Krannert School of Management, Purdue University. Email: sun244@purdue.edu.}, and
	Yufeng Liu\thanks{Professor, Department of Statistics and Operations Research, Department of Genetics, Department of Biostatistics, Carolina Center for Genome Sciences, Lineberger Comprehensive Cancer Center, The University of North Carolina at Chapel Hill. Email: yfliu@email.unc.edu.}
}

\date{}

\begin{document}
	\maketitle
	
	{\singlespacing
		\begin{abstract}
			Contextual dynamic pricing aims to set personalized prices based on sequential interactions with customers. At each time period, a customer who is interested in purchasing a product comes to the platform. The customer's valuation for the product is a linear function of contexts, including product and customer features, plus some random market noise. The seller does not observe the customer's true valuation, but instead needs to learn the valuation by leveraging contextual information and historical binary purchase feedbacks. Existing models typically assume full or partial knowledge of the random noise distribution. In this paper, we consider contextual dynamic pricing with unknown random noise in the linear valuation model. Our distribution-free pricing policy learns both the contextual function and the market noise simultaneously. A key ingredient of our method is a novel perturbed linear bandit framework, where a modified linear upper confidence bound algorithm is proposed to balance the exploration of market noise and the exploitation of the current knowledge for better pricing. We establish the regret upper bound and a matching lower bound of our policy in the perturbed linear bandit framework and prove a sub-linear regret bound in the considered pricing problem. Finally, we demonstrate the superior performance of our policy on simulations and a real-life auto-loan dataset.
		\end{abstract}
	}

	\vspace{+0in}
	
\begin{quote} \small
\end{quote} \normalsize

	\newpage
	\setcounter{page}{1}

\section{Introduction}
\label{sec:1}

Contextual dynamic pricing aims to design an online pricing policy adaptive to product features, customer characteristics, and marketing environment \citep{huang2021value}. It has been widely used in industries such as hospitality, tourism, entertainment, retail, electricity, and public transportation \citep{den2015dynamic}. A successful dynamic pricing algorithm involves both pricing and learning to maximize the revenues. Upon receiving sequential customer responses, the algorithm continuously updates its knowledge on the customer purchasing behavior and sets a price accordingly. Such online statistical learning differs from traditional supervised or unsupervised learning in its adaptive and sequential manner. 

The key learning objective in dynamic pricing is the willingness-to-pay (demand) of a customer, i.e., the probability of a customer making a buying decision. With full knowledge of the demand, the seller can set optimal prices that yield the maximum expected revenues. However, it is common that the seller knows little about the demand prior to the pricing procedure. Such an unknown demand case has been studied extensively in dynamic pricing \citep{besbes2009dynamic,keskin2014dynamic,cheung2017dynamic,chen2019dynamic,cesa2019dynamic,den2020discontinuous}. In this case, one critical task is to balance the tradeoff between exploration and exploitation, where exploration aims for more customer demand knowledge and exploitation maximizes the revenue based on the current knowledge. Two major influential factors for a customer's willingness-to-pay are the price offered by the seller as well as the customer's valuation of the product. In this paper, we consider a widely adopted linear valuation model \citep{javanmard2019dynamic,golrezaei2019incentive}. Given the contextual covariate $x$, e.g., product features, customer characteristics, and marketing environment, the customer's valuation $v(x)$ for the product is $v(x) = x^{\top}\theta_{0}+z$. Here, the first component represents the linear effect of the covariates $x$ with an unknown parameter $\theta_{0}$, and the second component models a market noise $z$ drawn from an unknown distribution $F$. After observing the price $p$ set by the seller, the customer buys the product if $v(x)$ exceeds $p$ and otherwise leaves without purchasing. 

Existing contextual dynamic pricing models assume partial or full knowledge of the market noise distribution $F$. For example, \cite{javanmard2019dynamic} assumes a known $F$ for their RMLP method and considers $F$ to belong to a log-concave family for their RMLP-2 policy. In spite that knowing $F$ simplifies the pricing process and improves learning accuracy, it can be restrictive and unrealistic in practice. It is essential to tackle the contextual dynamic pricing problem with an unknown $F$. Importantly, it may happen in practice that not all relevant contexts can be observed and such unobserved contexts may lead to a complex noise term. For example, the heterogeneity among customers may lead to a noise that is a mixture of many distributions, beyond the log-concave family. In our auto loan dataset studied in Section \ref{sec:5}, the estimated Probability Distribution Functions (PDF) of the noise term in four states are clearly not log-concave, as shown in Figure \ref{fig:20}.

\begin{figure}[h]
	\centering
	\subfigure{
		\begin{minipage}[t]{0.24\linewidth}
			\centering
			\includegraphics[width=1.55in]{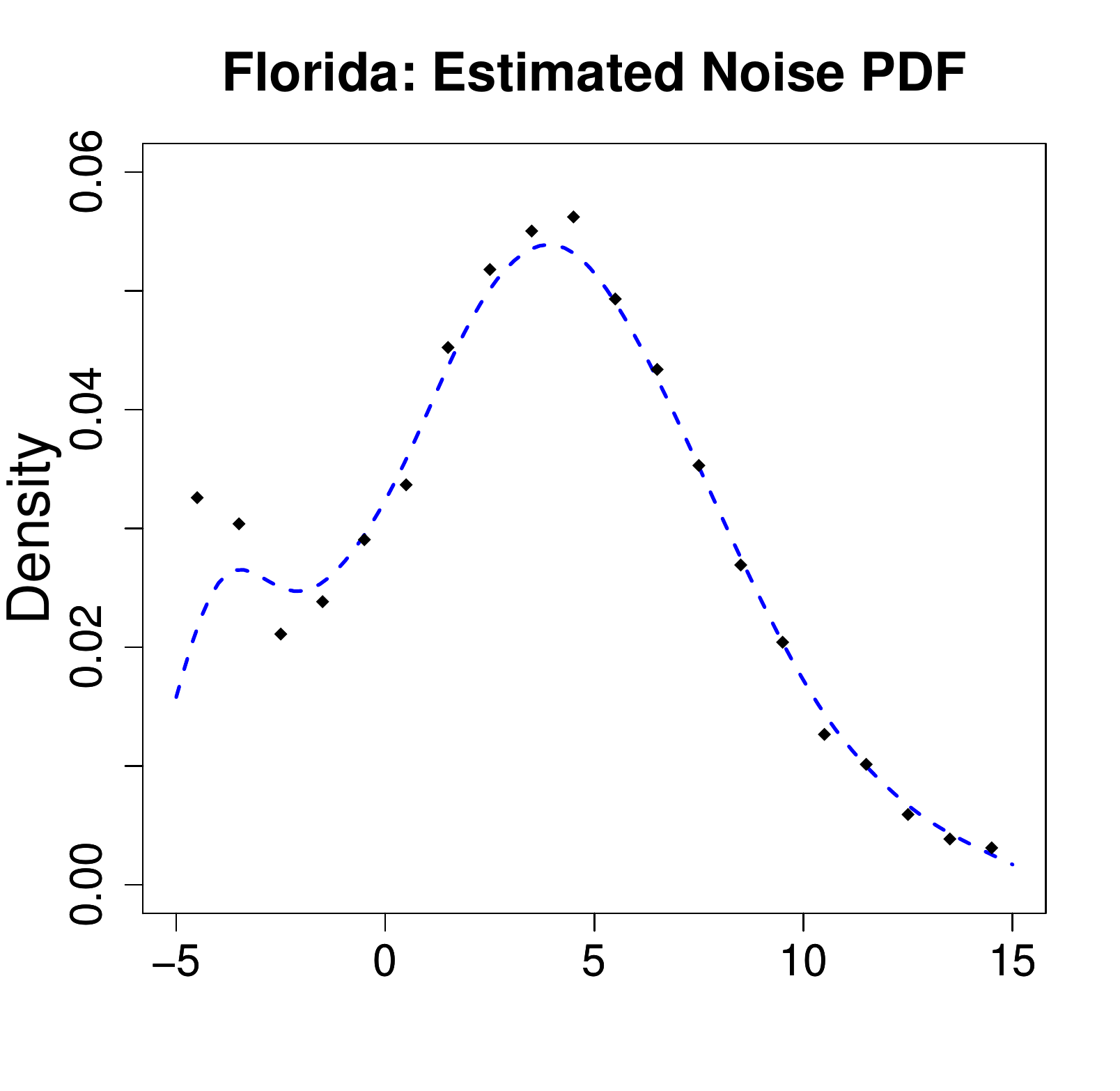}
		\end{minipage}%
	}%
	\subfigure{
		\begin{minipage}[t]{0.24\linewidth}
			\centering
			\includegraphics[width=1.55in]{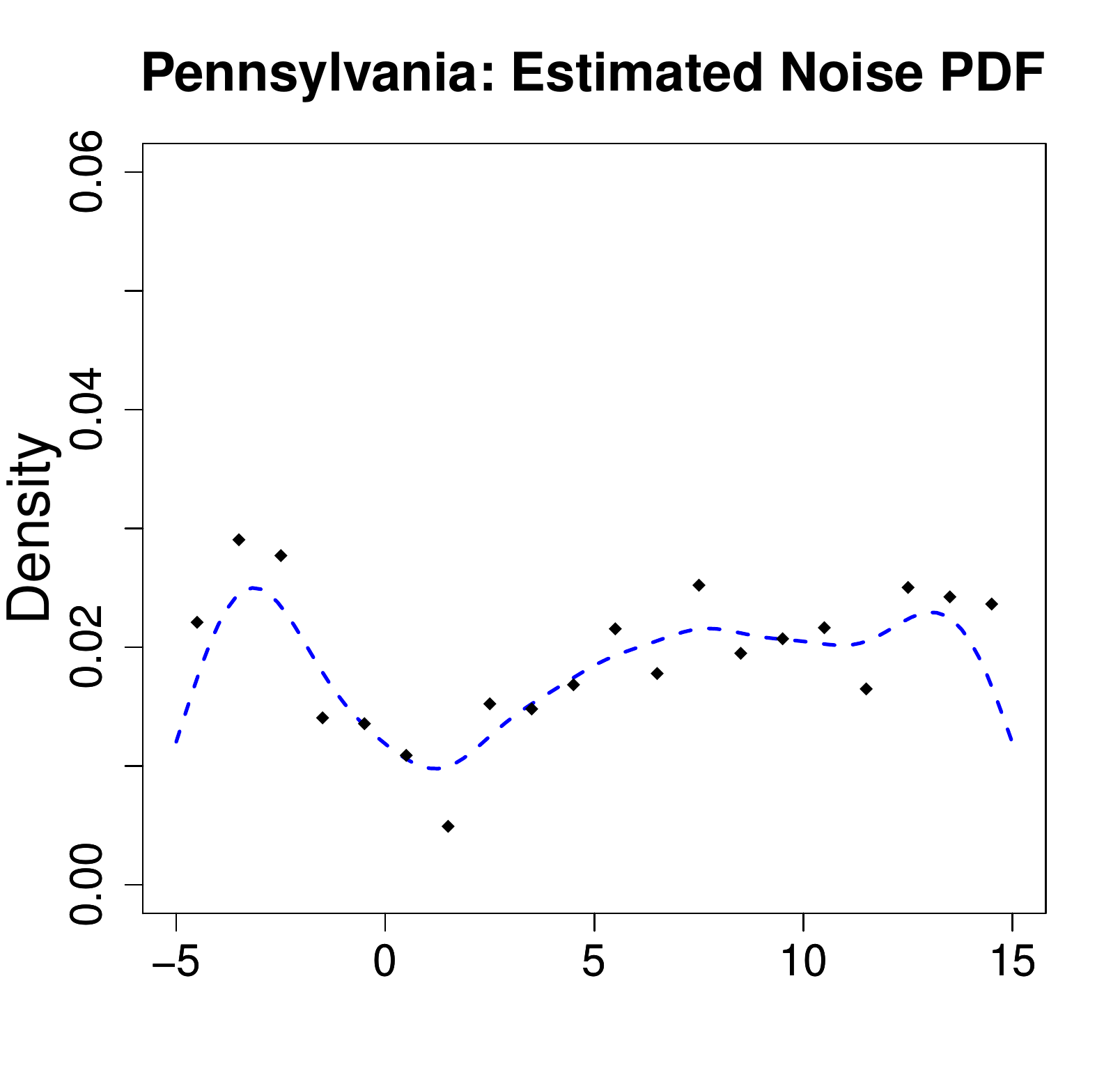}
		\end{minipage}%
	}%
	\subfigure{
		\begin{minipage}[t]{0.24\linewidth}
			\centering
			\includegraphics[width=1.55in]{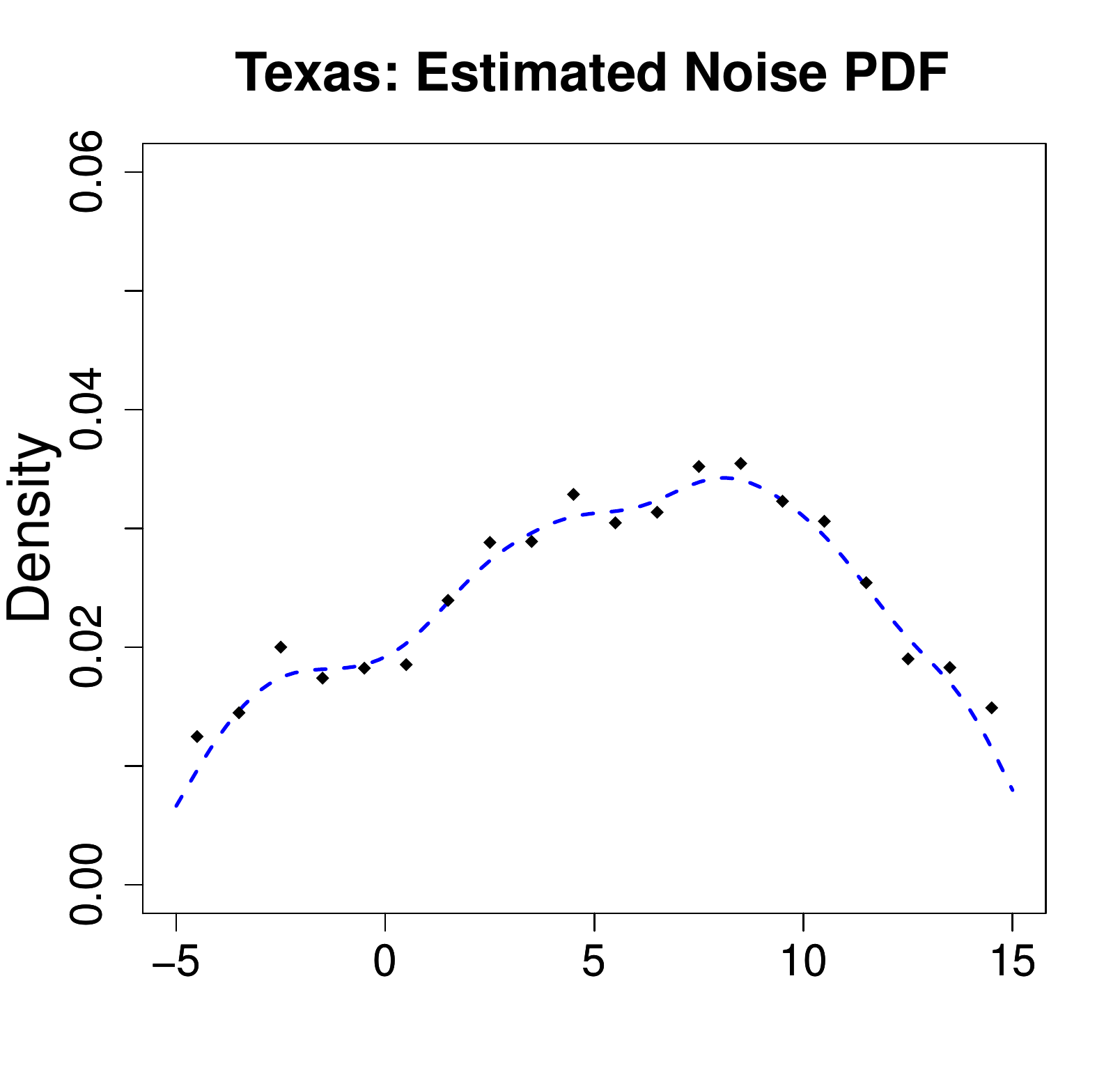}
		\end{minipage}%
	}%
	\subfigure{
		\begin{minipage}[t]{0.24\linewidth}
			\centering
			\includegraphics[width=1.55in]{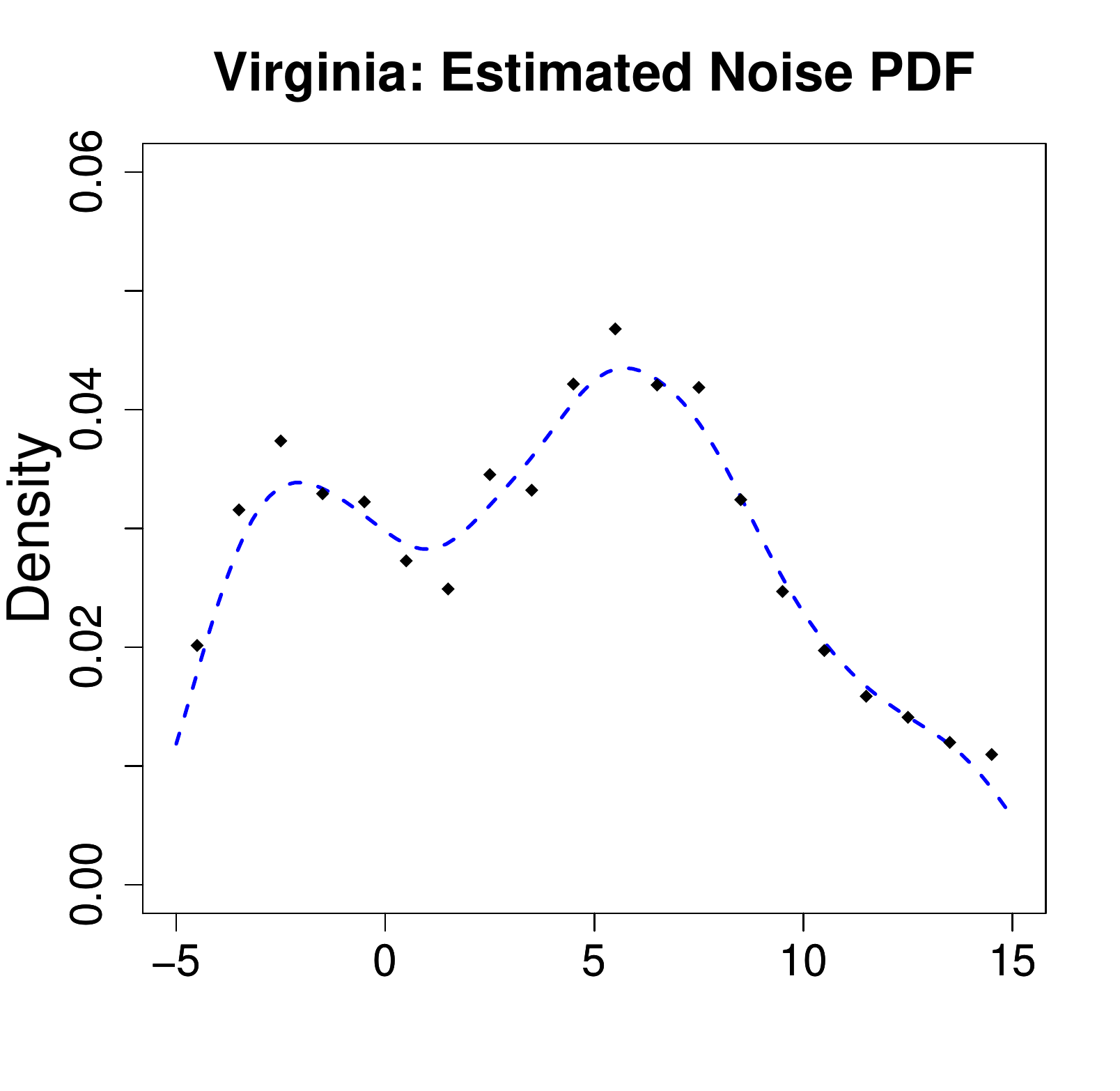}
		\end{minipage}%
	}%
	\centering
	\caption{Estimated noise PDFs for four states in our auto loan real application.}
	\label{fig:20}
\end{figure}

In this paper, we propose a DIstribution-free Pricing (DIP) policy to tackle the contextual dynamic pricing problem with unknown $\theta_{0}$ and unknown $F$. DIP employs a doubling trick \citep{lattimore2020bandit} in its framework, which cuts the time horizon into episodes in order to reduce the correlations across data and handle the unknown horizon length. At the beginning of each episode, by formulating the $\theta_{0}$ estimation into a classification problem in which no prior knowledge of $F$ is required, our DIP policy adopts the logistic regression to estimate $\theta_{0}$ using data in the previous episode. Given such an estimate, we then translate our single-episode pricing problem into a newly-proposed Perturbed Linear Bandit (PLB). PLB can be considered as an extension of the classic linear bandit \citep{abbasi2011improved,chu2011contextual,agrawal2013thompson}, and is also of independent interests. Interestingly, the ``perturbation level" of the translated PLB can be specified as proportional to the $\ell_{1}$ error of the given $\theta_{0}$ estimate. A modified Linear Upper Confidence Bound (M-LinUCB) algorithm, serving as an essential part of DIP, is proposed for our translated PLB to unify the learning of $F$ and exploitation of the learnt knowledge to set prices. 

In addition to the methodological contribution, we also establish regret analysis of our DIP policy. The regret, as the expected loss  of revenues with respect to the clairvoyant policy, is widely used to evaluate the performance of a pricing policy. We first prove a $T_{0}$-period regret of $\tilde{O}(\sqrt{T_{0}}+C_{p}T_{0})$ for M-LinUCB on a general PLB with $C_{p}$ representing the perturbation level. The decomposition of a sub-linear term and a linear term is analogous to the regret in misspecified linear bandits \citep{lattimore2020learning,pacchiano2020model,foster2020adapting}. Importantly, we also show that the linear dependence of $T_0$ is unavoidable by establishing a matching lower bound for our perturbed linear bandit. We then apply this result to the specific PLB formulation of our single-episode pricing problem to obtain the regret bound for each episode. Finally we obtain the regret bound for the entire $T$ horizon, which consists of an $\tilde{O}(T^{2/3})$ sub-linear term and an extra term related to the estimation error for $\theta_{0}$. The latter term is dominated by the sub-linear term in a broad range of scenarios, which is well supported by our experiments. In summary, our sub-linear $\tilde{O}(T^{2/3})$ regret upper bound implies that the average regret per time period vanishes as the time horizon tends to infinity. Because our problems involves both unknown linear parameter $\theta_{0}$ and unknown noise distribution $F$, we conjecture the obtained $\tilde{O}(T^{2/3})$ rate is close to the optimal rate. 

Finally, we demonstrate the superior performance of our policy on extensive simulations and a real-life auto-loan dataset by comparing our DIP policy to RMLP and RMLP-2 \citep{javanmard2019dynamic}. Due to the restrictive condition on $F$, RMLP is not satisfactory when a moderate misspecification of $F$ occurs. Despite being more robust than RMLP, RMLP-2 inevitably leads to a linear regret when the noise distribution is beyond log-concave. On the other hand, our DIP policy is robust to unknown complex noise distributions. In a real-life auto-loan dataset, our DIP policy is shown to largely improve the cumulative regret of the benchmark RMLP-2 method in learning customer's purchasing behavior of auto loans. Specifically, DIP has an $80\%$ improvement over RMLP-2 in the cumulative regret over the considered time horizon. Such an improvement keeps increasing when the total time horizon increases. See Figure \ref{fig:8} and Section \ref{sec:5} for more details.

\subsection{Related Work}

\textbf{Non-contextual dynamic pricing.} For non-contextual dynamic pricing without covariates, \cite{besbes2009dynamic,wang2014close,besbes2015surprising,chen2018primal} designed policies to handle a nonparametric model while \cite{besbes2009dynamic,broder2012dynamic,den2014simultaneously,keskin2014dynamic} considered parametric models. Furthermore, \cite{besbes2011minimax,den2015tracking,keskin2017chasing} investigated the time-varying unknown demand setting. In addition, the Upper Confidence Bound (UCB) idea \citep{auer2002finite,abbasi2011improved} has been used in different non-contextual instances \citep{kleinberg2003value,misra2019dynamic,wang2021multimodal}. However, all these approaches do not incorporate the covariates into pricing policy. Therefore, our model and technical tools are fundamentally different.

\noindent\textbf{Contextual dynamic pricing.} Dynamic pricing with covariates has garnered significant interest among researchers. As \cite{mueller2018low,javanmard2020multi,chen2021statistical} focused on the multi-product setting, most of the contextual dynamic pricing literature \citep{qiang2016dynamic,javanmard2017perishability,mao2018contextual,nambiar2019dynamic,bastani2019meta,ban2020personalized,cohen2020feature,wang2020uncertainty} considered a single product at each time. \cite{javanmard2019dynamic,golrezaei2019incentive,golrezaei2021dynamic} also considered the linear valuation model as we do in this paper. Similar to us, \cite{golrezaei2019incentive} assumed both the unknown linear effect and noise distribution and thus faced the same challenge of error propagation. They adopted a second price auction mechanism with multiple buyers at each time. One main difference lies in the feedback structure. Namely, they assumed a ``full information" setting that the seller observed all bids and valuations from multiple buyers while we considered a ``bandit" setting that the seller only observed one single buyer's binary purchasing decision. In \cite{javanmard2019dynamic}, their proposed RMLP assumed a known market noise distribution while RMLP-2 assumed a known log-concave family of the noise distribution. Hence their approaches are no longer applicable when the noise distribution is unknown or not log-concave. In addition, by assuming the noise distribution to be in a known ambiguity set, \cite{golrezaei2021dynamic} also established a $\tilde{O}(T^{2/3})$ regret with respect to a robust benchmark defined upon the ambiguity set. In the general unknown noise case, the ambiguity set could be extremely large and hence the robust benchmark could be far from the true optimal policy. In contrast, our DIP policy is adaptive to the general unknown noise case and our regret bound is established by comparing to the true optimal policy. On the other hand, \cite{shah2019semi,chen2021nonparametric} shared similar nonparametric ingredients in the unknown demand function as ours. Specifically, \cite{chen2021nonparametric} considered a general Lipschitz demand and proposed a pricing policy based on adaptive binning of the covariate space \citep{perchet2013multi} with a regret of $\tilde{O}(T^{(2+d_{0})/(4+d_{0})})$, where $d_{0}$ is the dimension of covariates. Thus when $d_{0}\geq 3$, our DIP policy enjoys better performance as we leverage the parametric structure in our dynamic pricing model. \cite{shah2019semi} adopted a log-linear valuation model to handle the unknown nonparametric noise in their semi-parametric model. Their method heavily relies on the special structure of the log-linear valuation model, whose optimal price has desirable separable effects of the unknown linear structure and unknown noise distribution. Hence their approach is not applicable to our pricing model where these two unknown parts entangle with each other. Therefore, techniques used in \cite{shah2019semi,chen2021nonparametric} for handling nonparametric components in the demand function are very different from the newly-proposed PLB framework of our DIP policy.

\noindent\textbf{Bandit algorithms.} Our pricing policy is also related to bandit algorithms \citep{bubeck2012regret,lattimore2020bandit,foster2020beyond} which address the balance between exploration and exploitation. In particular, our perturbed linear bandit is related to misspecified linear bandits \citep{lattimore2020learning,pacchiano2020model,foster2020adapting} and non-stationary linear bandits \citep{cheung2018hedging,russac2019weighted,zhao2020simple}. An interesting finding is that, by leveraging the special structure of the perturbed linear bandit formulation of our dynamic pricing problem, we achieve a better and more precise regret bound for our proposed policy, compared to a direct application of much complex existing algorithms for misspecified or non-stationary linear bandits. See Remarks \ref{rm:1} and \ref{rm:2} for more discussions.

\subsection{Notation and Paper Organization}

We adopt the following notations throughout the article. Let $[T] = \{1,\ldots,T\}$. For a vector $\beta \in \mathbb R^d$, let $\|\beta\|_{\infty} = \max_{j}|\beta_j|$ and $\|\beta\|_1 = \sum_{j=1}^d |\beta_j|$ denote its max norm and $\ell_1$ norm, respectively. For two sequences $a_n, b_n$, we say $a_n = O(b_n)$ if $a_n \le C b_n$ for some positive constant $C$, $a_n = \tilde{O}(b_n)$ if $a_n = O(b_n)$ that ignores a logarithm term, and $a_n = \Omega(b_n)$ if $a_n \ge C b_n$ for some positive constant $C$.

The rest of the paper is organized as follows. In Section \ref{sec:2}, we introduce the methodology of our proposed DIP policy along with the perturbed linear bandit formulation of the pricing problem. In Section \ref{sec:3}, we develop regret bounds for a general perturbed linear bandit problem and employ it to establish the regret bound of our DIP policy. In Section \ref{sec:4}, we demonstrate the superior performance of DIP on various synthetic datasets and in Section \ref{sec:5}, we apply DIP to a real-life auto loan dataset. We conclude our work along with some future directions in Section \ref{sec:6}. Most technical proofs are collected in the Supplementary Material. 

\section{Methodology}
\label{sec:2}

In this section, we discuss the contextual dynamic pricing problem setting and then introduce our DIP policy which involves a general perturbed linear bandit formulation.

\subsection{Problem Setting}
\label{sec:2_1}

In contextual dynamic pricing, a potential customer who is interested in purchasing a product arrives at the platform at each period $t\in[T] = \{1,\dots,T\}$, and the seller observes a covariate $x_{t}\in \mathcal{X}\subseteq \mathbb{R}^{d_{0}}$ representing the product features and customer characteristics. Similar to \cite{javanmard2019dynamic,golrezaei2019incentive,shah2019semi,chen2021nonparametric}, we assume $||x_t||_{\infty}\leq 1,\forall x_t\in\mathcal{X}$. Given $x_{t}$, the customer's valuation of the product $v_{t} = v(x_{t}) = x_{t}^{\top}\theta_{0} + z_{t}$ is a sum of a linear function of $x_{t}$ and a market noise $z_{t}$. We assume $\{z_{t}\}_{t\in[T]}$ are drawn i.i.d. from an unknown distribution with Cumulative Distribution Function (CDF) $F$. If the customer's valuation $v_{t}$ is higher than the price $p_{t}$ set by the seller, the sale happens and the seller collects a revenue of $p_{t}$. Otherwise, the customer leaves and the seller receives no revenue. Let $y_{t}=1_{\{v_{t}\geq p_{t} \}}$ denote whether the customer buys the product. By the aforementioned sales mechanism, it follows
\[
y_{t} = \begin{cases}
1 & \text{ if } v_{t}\geq p_{t},\text{ with probability }1-F(p_{t}-x_{t}^{\top}\theta_{0});\\
0 & \text{ if } v_{t}<p_{t},\text{ with probability }F(p_{t}-x_{t}^{\top}\theta_{0}),\\
\end{cases}
\]
and the reward $Z_{t} = p_{t}y_{t} = p_{t}1_{\{v_{t}\geq p_{t} \}}$. Then the triplet $(x_{t},p_{t},y_{t})$ records the information of the pricing procedure at time $t$.

Given the above customer choice model and the covariate $x$, the expected reward of a setting price $p$ is $p(1-F(p-x^{\top}\theta_{0}))$. We define the optimal price $p^{*}(x)$ as that maximizing $p(1-F(p-x^{\top}\theta_{0}))$, which is an implicit function of the covariate and dependent on both the unknown $\theta_{0}$ and $F$. By dynamically setting prices and observing binary feedbacks, we collect instant revenues and meanwhile gather more information to estimate $\theta_{0}, F$ and $p^{*}(x)$. An important feature of this process is the tradeoff between exploration and exploitation where we shall well balance between exploiting the current knowledge for larger immediate revenues and exploring more information for better future revenues.

We next introduce the notion of regret for evaluating a pricing policy. Denote $$p^{*}_{t} = p^{*}(x_{t}) = \arg\max_{p>0}p(1-F(p-x_{t}^{\top}\theta_{0}))$$ as the optimal price at time $t$. Then the regret $r_{t}$ at time $t$ is defined as the loss of reward by setting the price $p_{t}$ compared to the optimal price $p_{t}^{*}$, i.e., 
\begin{equation}
r_{t} = p_{t}^{*}(1-F(p_{t}^{*}-x_{t}^{\top}\theta_{0}))-p_{t}(1-F(p_{t}-x_{t}^{\top}\theta_{0})).
\label{eqn:regret}
\end{equation}
The $T$-period cumulative regret across the horizon is defined as $R_{T} = \sum_{t=1}^{T}r_{t}$. We obtain the expected cumulative regret $\mathbb{E}(R_{T})$ by taking the expectation with respect to the randomness of data and the potential randomness of the pricing policy. The goal of our contextual dynamic pricing is to decide the price $p_t$ for covariate $x_t$ at time $t$, by utilizing all historical data $\{(x_{s},p_{s},y_{s}), s=1,\ldots,t-1\}$, in order to minimize the expected cumulative regret.

\subsection{DIP Algorithm}
\label{sec:2_2}

Our proposed DIP policy enjoys a simple framework as an Outer Algorithm nested with the Inner Algorithm A and the Inner Algorithm B. The Inner Algorithm A is designed for estimating $\theta_{0}$, and the Inner Algorithm B is the essential part that fully exploits the perturbed linear bandit formulation of our single-episode pricing problem and implements the UCB idea to resolve the tradeoff between exploration and exploitation.

\subsubsection{Outer Algorithm}

In online learning, the total time horizon $T$ is typically unknown. To address this problem, we adopt a doubling trick widely used in online learning and bandit algorithms \citep{lattimore2020bandit} to cut the horizon into episodes. After the first warm-up episode and starting from the second episode, we set the length of the next episode as double of the current one until the horizon ends. The number of episodes $n = n(T,\alpha_{1},\alpha_{2})$ and their lengths denoted as $\{\ell_{k} = \ell_{k}(T,\alpha_{1},\alpha_{2})\}_{k\in[n]}$ are functions of the total horizon length $T$ and the first two episodes' lengths $\alpha_{1},\alpha_{2}$. Figure \ref{fig:1} demonstrates the case when the total time horizon is cut into 5 episodes via the doubling trick. 

\begin{figure}[h]
	\centering
	\includegraphics[width=6in]{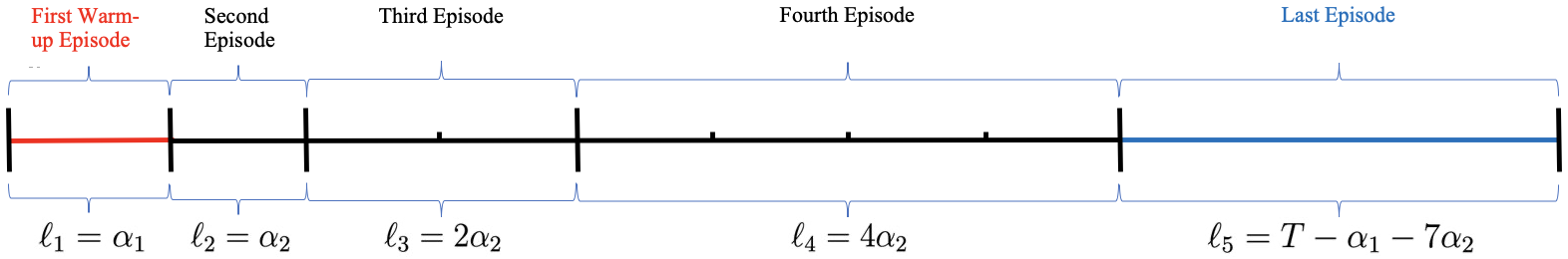}
	\caption{An illustration of cutting total time horizon utilizing the doubling trick.}
	\label{fig:1}	
\end{figure}

We present the outline of our DIP policy as the generic Outer Algorithm in Algorithm 1. In the first warm-up episode, DIP performs random exploration to set random prices at each time period. Then DIP alternates between Inner Algorithm A to obtain an estimate of $\theta_{0}$ and Inner Algorithm B to set prices. Specifically, Inner Algorithm A uses all data from episode $k-1$ to obtain an estimate $\hat{\theta}_{k-1}$ of $\theta_{0}$; then Inner Algorithm B takes $\hat{\theta}_{k-1}$ as an input to sequentially set prices for all time periods in episode $k$, which then forms all triplets of covariates, prices and customer responses in episode $k$ for future $\theta_{0}$ estimation by Inner Algorithm A. Another advantage of the horizon cutting strategy is the reduction of correlation across the pricing procedure.

\begin{algorithm}[h]
	\caption{Generic Outer Algorithm}\label{alg:0}
	\begin{algorithmic}[1]
		\STATE \textbf{Input:} \textbf{(arrives over time)} covariates $\{x_{t}\}_{t\in[T]}$
		\STATE Denote the episodes yielded by the doubling trick as $\mathcal{E}_{1},\dots,\mathcal{E}_{n}$. 
		\STATE \textbf{For} $t\in\mathcal{E}_{1}$, \textbf{do}
		\STATE \hspace{0.075in} Set a price $p_{t}$ randomly from $(0,p_{\max})$ and receive a binary response $y_{t}$.
		\STATE \textbf{For} episode $k = 2,3,\dots,n$, \textbf{do}
		\STATE \hspace{0.075in} With input data $\{(x_{t},p_{t},y_{t})\}_{t\in \mathcal{E}_{k-1}}$, apply Inner Algorithm A on this dataset to update\\ \hspace{0.075in} an estimate $\hat{\theta}_{k}$ of $\theta_{0}$;
		\STATE \hspace{0.075in} With input $\hat{\theta}_{k-1}$ as the estimate of $\theta_{0}$, apply Inner Algorithm B on $\mathcal{E}_{k}$ to sequentially\\
		\hspace{0.075in} set a price $p_{t}$ and receive a binary response $y_{t}$ for all $t\in \mathcal{E}_{k}$.
	\end{algorithmic}
\end{algorithm}
\subsubsection{Inner Algorithm A}

We  now introduce the Inner Algorithm A designed for estimating $\theta_{0}$. It uses all data $(x_{t},p_{t},y_{t})$ from the $(k-1)$-th episode to obtain an estimate $\hat{\theta}_{k-1}$ for future pricing in the $k$-th episode. For simplicity, we introduce its generic version with $[T_{0}]=\{1,\dots,T_{0}\}$ representing the $(k-1)$-th episode horizon. Since $y_{t}$ is binary and invoked by $x_{t},p_{t}$ through $\mathbb{P}(y_{t} = 1) = 1-F(p_{t}-x_{t}^{\top}\theta_{0})$, we obtain
\[
\begin{cases}
\mathbb{P}(y_{t} = 1)  > \frac{1}{2}, \text{ if }F^{-1}(\frac{1}{2})+x_{t}^{\top}\theta_{0}-p_{t}>0;\\
\mathbb{P}(y_{t} = 1)  = \frac{1}{2}, \text{ if }F^{-1}(\frac{1}{2})+x_{t}^{\top}\theta_{0}-p_{t}=0;\\
\mathbb{P}(y_{t} = 1)  < \frac{1}{2}, \text{ if }F^{-1}(\frac{1}{2})+x_{t}^{\top}\theta_{0}-p_{t}<0.\\
\end{cases}
\]
Therefore, we can form a classification problem with responses $y_{t}$ and covariates $(1,x_{t}^{\top},p_{t})^{\top}$ for $t\in [T_{0}]$. It admits a Bayes decision boundary $\{u: (F^{-1}(\frac{1}{2}),\theta_{0}^{\top},-1)u = 0 \}$ which involves the unknown parameter $\theta_{0}$. Thus we can estimate the linear decision boundary and extract an estimate of $\theta_{0}$ by applying a linear classification method. In this paper, we use logistic regression, which yields an estimate $(\hat{c},\hat{\beta}^{\top},\hat{b})$ of $(F^{-1}(\frac{1}{2}),\theta_{0}^{\top},-1)$ up to a constant factor. Thus $-\frac{\hat{\beta}}{\hat{b}}$ is a natural estimate of $\theta_{0}$. Similar to \cite{javanmard2019dynamic}, we assume $||\theta_{0}||_{1}$ is upper bounded by a known constant $W$. By projecting $-\frac{\hat{\beta}}{\hat{b}}$ onto the $\ell_{1}$-ball $\Theta = \{\theta\in\mathbb{R}^{d_{0}}:||\theta||_{1}\leq W \}$, we can obtain our final estimate denoted as $\hat{\theta} = \text{Proj}_{\Theta}(-\frac{\hat{\beta}}{\hat{b}})$. Such a projection has a closed-form solution as $\text{Proj}_{\Theta}(-\frac{\hat{\beta}}{\hat{b}})= \mathcal{T}_{\rho_{\min}}(-\frac{\hat{\beta}}{\hat{b}})$, where $\mathcal{T}_{\rho}(v) = \text{sgn}(v)(|v|-\rho)_{+}$ is the soft-thresholding operator and $\rho_{\min} = \min\{\rho:||\mathcal{T}_{\rho}(-\frac{\hat{\beta}}{\hat{b}})||_{1}\leq W\}$. Here the assumption of constant $W$ is purely for theoretical purpose and our policy is very robust to the value of $W$ in the empirical studies. The generic Inner Algorithm A is summarized in Algorithm 2. 

\begin{algorithm}[h]
	\caption{Generic Inner Algorithm A}\label{alg:2}
	\begin{algorithmic}[1]
		\STATE \textbf{Input:} $\{(x_{t},p_{t},y_{t})\}_{t\in[T_{0}]}$, $W$
		\STATE Use logistic regression to obtain the minimizer
		\[
		(\hat{c},\hat{\beta}^{\top},\hat{b})= \arg\min_{(c,\beta^{\top},b)}\sum_{t=1}^{T_{0}}\log(1+\exp((2y_{t}-1)(c,\beta^{\top},b)(1,x_{t}^{\top},p_{t})^{\top})).
		\]
		\STATE Estimate $\theta_{0}$ by $\hat{\theta} = \text{Proj}_{\Theta}(-\frac{\hat{\beta}}{\hat{b}})$, where $\Theta = \{\theta\in\mathbb{R}^{d_{0}}:||\theta||_{1}\leq W \}$. 
	\end{algorithmic}
\end{algorithm}

Under the same assumption of a known upper bound $W$ of $||\theta_{0}||_{1}$, RMLP and RMLP-2 in \cite{javanmard2019dynamic} estimated $\theta_{0}$ via the maximum likelihood type of method by assuming some knowledge on $F$. In comparison, our approach achieves robust $\theta_{0}$ estimation without knowledge of a potentially complex-shaped $F$. It is worth mentioning that the logistic regression used in Algorithm 2 can be replaced by other linear classification methods, e.g., large-margin classifiers \citep{wang2008probability}. We choose logistic regression for its simplicity and superior numerical performance.

\subsubsection{Inner Algorithm B}

Next we introduce the Inner Algorithm B designed for setting prices. Taking $\hat{\theta}_{k-1}$ obtained by Inner Algorithm A as an input, it sequentially sets prices for all time periods in episode $k$. For ease of presentation, we introduce a generic version by using $\hat{\theta}$ to represent $\hat{\theta}_{k-1}$ and $T_{0}$ to represent the length of the episode $k$.

Based on our model in Section \ref{sec:2_1}, the knowledge of the expected reward $p(1-F(p-x_{t}^{\top}\theta_{0}))$ plays a critical role in deciding the best price at time $t$. Given the current estimate $\hat{\theta}$, we would need to evaluate $\{p(1-F(p-x_{t}^{\top}\hat{\theta}))\}$ over $p\in(0,p_{\max})$. Here we assume there is a known upper bound $p_{\max}$ of our pricing problem. This assumption is very mild in real applications and was also used in \cite{javanmard2019dynamic,chen2021nonparametric}. By the condition $||x_t||_{\infty}\leq 1$, we have $p-x_{t}^{\top}\hat{\theta}\in G(\hat{\theta}) = [-||\hat{\theta}||_{1},p_{\max}+||\hat{\theta}||_{1}]$. Therefore the evaluation of the expected reward is reduced to evaluate $1-F$ on $G(\hat{\theta})$. When $F$ is Lipschitz continuous and no other global smoothness is assumed, it is sufficient to evaluate $1-F$ on several well-chosen discrete points in $G(\hat{\theta})$ to leverage the finite data for better pricing. In this paper, we utilize the discretization idea \citep{kleinberg2003value,weed2016online} to cut $G(\hat{\theta})$ into $d$ same-length subintervals with the set of their midpoints $\mathcal{M} = \{m_{1},\dots,m_{d} \}$. Here $d$ is a parameter that possibly depends on the horizon length $T_{0}$. When $T_{0}$ is large, it would be reasonable to set a bigger $d$ for a denser discretization and hence larger exploration spaces. We leave the detailed discussion on the choice of $d$ to the theoretical analysis of DIP in Section \ref{sec:3}. Our aim is then to dynamically set prices and evaluate $1-F$ on $\mathcal{M}$.

\begin{figure}
	\centering
	\includegraphics[width=3in]{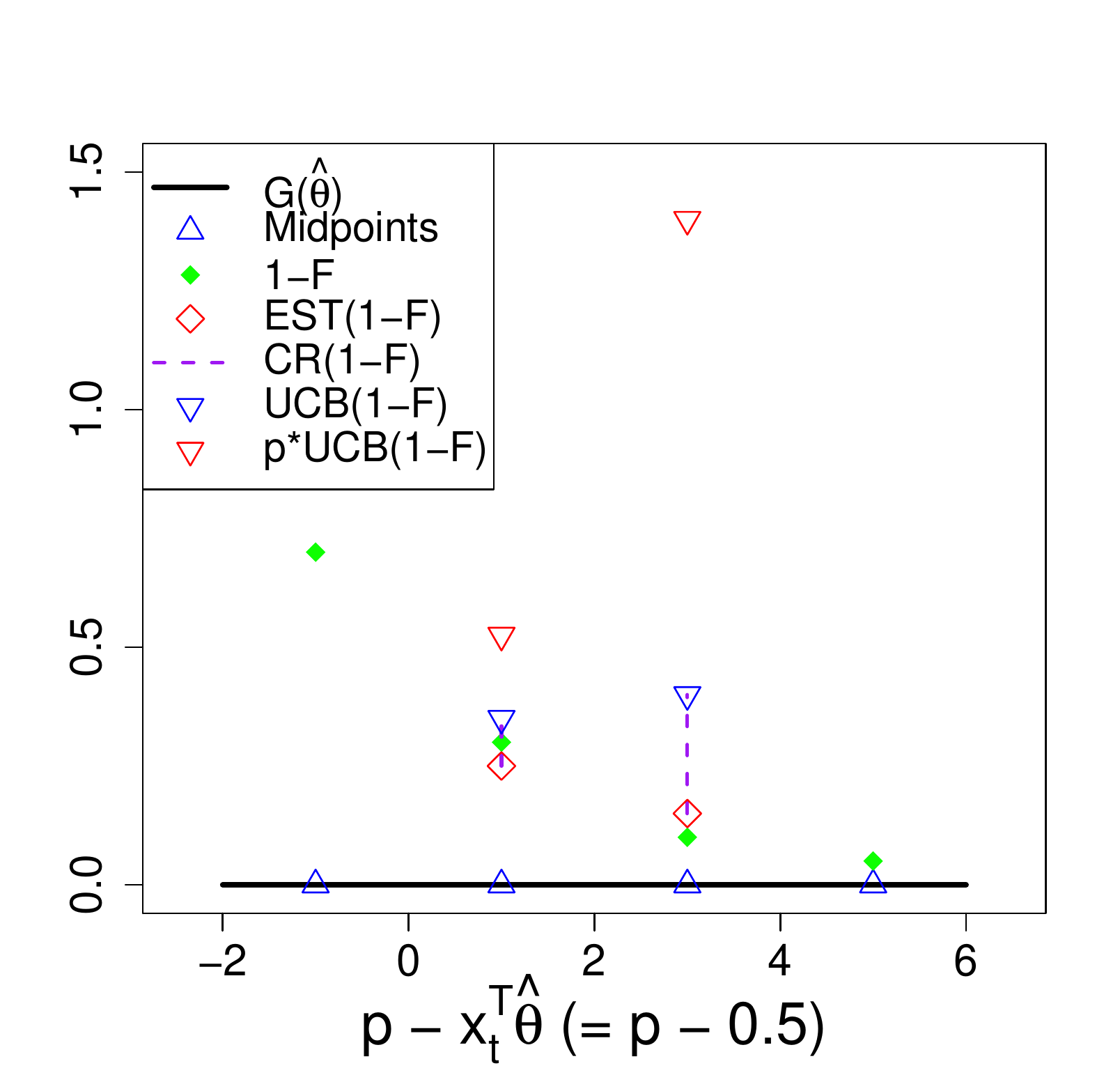}
	\caption{An illustration of Inner Algorithm B via Example 1.}
	\label{fig:0}
\end{figure}

\begin{example}[\textbf{Discretization}]\label{exp:1}
	We introduce a toy example to better illustrate our pricing policy. We couple each part of our pricing strategy with its corresponding realization in this toy example. All quantities that will be introduced in our pricing policy for this specific example are displayed in Figure \ref{fig:0}. Consider a two-dimensional covariate $x_{t} = (0.3,0.2)^{\top}$ at time $t$. Assume we have an estimation $\hat{\theta} = (1,1)^{\top}$ and $p_{\max} = 4$. Then the interval for discretization is $G(\hat{\theta}) = [-||\hat{\theta}||_{1},p_{\max}+||\hat{\theta}||_{1}] = [-2,6]$ represented by the black solid line in Figure \ref{fig:0}. If $d = 4$, we discretize $G(\hat{\theta})$ into subintervals $[-2,0],[0,2],[2,4]$, and $[4,6]$. Their midpoints $m_{1}=-1,m_{2}=1,m_{3}=3,m_{4}=5$, represented by blue hollow triangles on the black line in Figure \ref{fig:0}, form the set $\mathcal{M} = \{-1, 1, 3, 5 \}$. We will continue this example later. 
\end{example}

To achieve the mutual reinforcement of pricing and evaluation of $1-F$ on $\mathcal{M}$, we restrict the set price $p_{t}$ at time $t$ into a carefully constructed candidate set $\mathcal{S}_{t} = \{m_{j}+x_{t}^{\top}\hat{\theta}|j\in [d],m_{j}+x_{t}^{\top}\hat{\theta}\in(0,p_{\max})\}$. The key feature for any price $p \in \mathcal{S}_{t}$ is that $p-x_{t}^{\top}\hat{\theta}$ exactly equals to a midpoint in $\mathcal{M}$. We now illustrate why pricing in $\mathcal{S}_{t}$ and evaluation of $1-F$ on $\mathcal{M}$ can enhance each other. For any price $p = m_{j}+x_{t}^{\top}\hat{\theta}\in \mathcal{S}_{t}$, we can leverage our current knowledge of $1-F(m_{j})$ to obtain an estimate of its expected reward $p(1-F(p-x_{t}^{\top}\theta_{0}))$ as $p(1-F(p-x_{t}^{\top}\hat{\theta})) = p(1-F(m_{j}))$. Thus a better evaluation of $1-F$ on $\mathcal{M}$ improves our pricing decision from $\mathcal{S}_{t}$. On the other hand, when we set one price $p_{t} = m_{j}+x_{t}^{\top}\hat{\theta}$ from $\mathcal{S}_{t}$, we observe a binary response $y_{t}\sim\text{Ber}(1-F(m_{j}+x_{t}^{\top}\hat{\theta}-x_{t}^{\top}\theta_{0}))\approx \text{Ber}(1-F(m_{j}))$ which then improves our knowledge of $1-F(m_{j})$. Upon this observation, we say that we pull arm $j$ at time $t$ if we set $p_{t} = m_{j}+x_{t}^{\top}\hat{\theta}$. Then pulling arm $j$ yields more knowledge for $1-F$ on $m_{j}$. Thus we define the available arm set at time $t$ as $\mathcal{B}_{t} = \{j\in [d]: \exists p\in \mathcal{S}_{t} \text{ such that }  p= m_{j}+x_{t}^{\top}\hat{\theta}\}$, which varies over time as $x_{t}^{\top}\hat{\theta}$ changes over time.

\setcounter{example}{0}

\begin{example}[\textbf{Continued, Construct candidate sets}] We construct the candidate sets $\mathcal{S}_{t}$ based on the discretized set $\mathcal{M} = \{m_{1}=-1,m_{2}=1,m_{3}=3,m_{4}=5 \}$. As $x_{t}^{\top}\hat{\theta} = 0.5$, we obtain $\mathcal{S}_{t} =\{m_{j}+x_{t}^{\top}\hat{\theta}|m_{j}+x_{t}^{\top}\hat{\theta}\in(0,p_{\max})\} = \{m_{2}+0.5,m_{3}+0.5\}= \{1.5,3.5 \}$ since $m_{1}+x_{t}^{\top}\hat{\theta} = -0.5$ and $m_{4}+x_{t}^{\top}\hat{\theta} = 5.5$ are out of the range $(0,p_{\max})$. In this case, the arm set at time $t$ is $\mathcal{B}_{t} = \{j\in [d]: \exists p\in \mathcal{S}_{t} \text{ such that }  p= m_{j}+x_{t}^{\top}\hat{\theta}\} = \{2,3 \}$.
\end{example}

Restricted on $\mathcal{S}_{t}$, there is a clear tradeoff between exploration and exploitation for our pricing problem. A pure exploration tends to pull less-pulled arms in $\mathcal{B}_{t}$ and may set many suboptimal prices while a pure exploitation may continuously pull suboptimal arms due to lack of knowledge of other arms. To balance between exploration and exploitation, we utilize the principle of optimism in the face of uncertainty \citep{lattimore2020bandit} to construct an upper confidence bound (UCB), which calls for both an estimation $\text{EST}_{t}(1-F(m_{j}))$ for $1-F(m_{j})$ and a Confidence Radius (CR) $\text{CR}_{t}(1-F(m_{j}))$ of this estimation at the beginning of time $t$. We can accomplish this goal using all the past data yielded by pulling arm $j$. We leave the specific forms of $\text{EST}_{t}(1-F(m_{j}))$ and $\text{CR}_{t}(1-F(m_{j}))$ to the next subsection as they emerge naturally from the perturbed linear bandit formulation of our single-episode pricing problem. Then we select $p_{t}= m_{j}+x_{t}^{\top}\hat{\theta}\in\mathcal{S}_{t}$ with the largest optimism estimation $p_{t}\text{UCB}_{t}(1-F(m_{j}))$, where $\text{UCB}_{t}(1-F(m_{j}))= \text{EST}_{t}(1-F(m_{j})) + \text{CR}_{t}(1-F(m_{j}))$ is an optimism estimation of $1-F(m_{j})$. This optimism estimation addresses the exploration-exploitation tradeoff since a large UCB could result in either exploring a less-pulled arm with a large CR or exploiting an optimal arm with a large mean estimation.


\setcounter{example}{0}

\begin{example}[\textbf{Continued, Set prices}] As the available arm set is $\mathcal{B}_{t} = \{2,3 \}$ at time $t$, we only require knowledge of $1-F(m_{2})$ and $1-F(m_{3})$ to compare between two candidate prices $m_{2}+x_{t}^{\top}\hat{\theta}$ and $m_{3}+x_{t}^{\top}\hat{\theta}$. To emphasize this, in Figure \ref{fig:0}, we only show $\{\text{EST}_{t}(1-F(m_{j}))\}_{j=2,3}$ (red hollow diamonds) and $\{\text{CR}_{t}(1-F(m_{j}))\}_{j=2,3}$ (lengths of purple dashed line) at two midpoints $m_{2} = 1$ and $m_{3} = 3$. Summing them up leads to the optimism estimations $\{\text{UCB}_{t}(1-F(m_{j}))\}_{j=2,3}$ represented by blue hollow inverted triangles. Multiplying them by their corresponding prices $m_{2}+x_{t}^{\top}\hat{\theta} = 1.5$ and $m_{3}+x_{t}^{\top}\hat{\theta} = 3.5$, we obtain their optimism expected reward estimations represented by red hollow inverted triangles, which are used to form our pricing decisions. Based on the illustration in Figure \ref{fig:0}, we will set the price $p_{t} = 3.5$, i.e., $m_{3}+x_{t}^{\top}\hat{\theta}$, since $1.5\text{UCB}_{t}(1-F(m_{2}))< 3.5\text{UCB}_{t}(1-F(m_{3}))$.
\end{example}


We summarize the generic Inner Algorithm B for one episode in Algorithm \ref{alg:3}. 

\begin{algorithm}[h!]
	\caption{Generic Inner Algorithm B}\label{alg:3}
	\begin{algorithmic}[1]
		\STATE \textbf{Input:} \textbf{(arrives over time)} covariates $\{x_{t}\}_{t\in[T_{0}]}$, $\hat{\theta}$, discretization number $d$, and other inputs required to construct the specific forms of $\{\text{UCB}_{t}(1-F(m_{j}))\}_{j\in[d]}$.
		\STATE Cut the interval $G(\hat{\theta}) = [-||\hat{\theta}||_{1},p_{\max}+||\hat{\theta}||_{1}]$ into $d$ same-length intervals and denote their midpoints as $m_{1},\dots,m_{d}$. 
		\STATE \textbf{For} time $t = 1,\dots,T_{0}$, \textbf{do}
		\STATE \hspace{0.075in} Construct the candidate price set $\mathcal{S}_{t} = \{m_{j}+x_{t}^{\top}\hat{\theta}|j\in [d],m_{j}+x_{t}^{\top}\hat{\theta}\in(0,p_{\max})\}$;
		\STATE \hspace{0.075in} Determine the arm set $\mathcal{B}_{t} = \{j\in [d]: \exists p\in \mathcal{S}_{t} \text{ such that }  p= m_{j}+x_{t}^{\top}\hat{\theta}\}$;
		\STATE \hspace{0.075in} Calculate $\text{UCB}_{t}(1-F(m_{j}))$ for $j\in \mathcal{B}_{t}$ in (\ref{eq:4});
		\STATE \hspace{0.075in} Calculate $j_{t} \in \arg\max_{j\in \mathcal{B}_{t}}(m_{j}+x_{t}^{\top}\hat{\theta})\text{UCB}_{t}(1-F(m_{j}))$;
		\STATE \hspace{0.075in} Set a price $p_{t}=m_{j_{t}}+x_{t}^{\top}\hat{\theta}$ and receive a binary response $y_{t}$.
	\end{algorithmic}
\end{algorithm}

\subsubsection{Perturbed Linear Bandit}

In this subsection, we first introduce a perturbed linear bandit (PLB) framework, then show that our single-episode pricing problem can be formulated as a PLB. Furthermore, the proposed M-LinUCB for PLB is shown to be equivalent to the Inner Algorithm B with a specific UCB construction.

We say that the reward $Z_{t}$, the parameter $\xi_{t}$, the action set $\mathcal{A}_{t}$ form a perturbed linear bandit with a perturbation constant $C_{p}$ if $Z_{t} = \langle \xi_{t},A_{t}\rangle +\eta_{t}$ with any selected action $A_{t}\in\mathcal{A}_{t}$ and $||\xi_{s}-\xi_{t}||_{\infty}\leq C_{p}$ for any $s,t$. Here $\eta_{t}$ is a sub-Gaussian conditional on the filtration $\mathcal{F}_{t-1} = \sigma(\xi_{1}, A_{1},Z_{1},\dots,\xi_{t}, A_{t})$. Note that the condition on the linear parameters $\xi_{t}$'s implies the existence of a $\xi^{*}$ such that $||\xi_{t}-\xi^{*}||_{\infty}\leq \frac{C_{p}}{2}$ for any $t$. Thus the linear parameter $\xi_{t}$ regulating the reward structure at time $t$ can be viewed as a perturbation from a ``central" parameter $\xi^{*}$. Note that the linear bandit \citep{abbasi2011improved,chu2011contextual,agrawal2013thompson} is a special zero-perturbation PLB with $\xi_{t} = \xi^{*}$ for any $t$.

Now we introduce the perturbed linear bandit formulation of our single-episode pricing problem with time horizon $[T_{0}]$. We first specify the linear parameter $\xi_{t} = (1-F(m_{1}+x_{t}^{\top}\hat{\theta}-x_{t}^{\top}\theta_{0}),\dots,1-F(m_{d}+x_{t}^{\top}\hat{\theta}-x_{t}^{\top}\theta_{0}))^{\top} \in \mathbb R^d$, which turns out to regulate the reward at time $t$ as shown in Lemma \ref{lem:1} below. Note that for any price $m_{j}+x_{t}^{\top}\hat{\theta}\in\mathcal{S}_{t}$, the $j$-th element of $\xi_{t}$ is exactly the purchasing probability of the customer faced with this price. Further define $\xi^{*} = (1-F(m_{1}),\dots,1-F(m_{d}) )^{\top}$ as the ``central" parameter. Then by Lemma \ref{lem:1} below, $\xi_{t}$'s can be viewed as perturbations from $\xi^{*}$. It is interesting to see that the perturbations indeed originate from the difference between the estimate $\hat{\theta}$ and the true $\theta_{0}$, and may change with covariates $x_{t}$'s.

To transform price setting into action selection, we define a mapping from any price $p =m_{j}+x_{t}^{\top}\hat{\theta} \in \mathcal{S}_{t}$ to a vector $Q_{t}(p)\in\mathbb{R}^{d}$ with $Q_{t}(p)_{j} = m_{j}+x_{t}^{\top}\hat{\theta}$ and $Q_{t}(p)_{i}=0, \forall i\neq j$. Namely, $Q_{t}$ maps a price $p=m_{j}+x_{t}^{\top}\hat{\theta} \in\mathcal{S}_{t}$ to a vector with a single nonzero $j$-th element $p$. Further define a vector set $\mathcal{A}_{t} = \{Q_{t}(p):p\in\mathcal{S}_{t}\}$. Then $Q_{t}$ is a one-to-one mapping from $S_{t}$ to $\mathcal{A}_{t}$ and $Q_{t}^{-1}$ is well-defined. To proceed, we define the price-action coupling by $A_{t} = Q_{t}(p_{t})$. Then setting any price $p_{t}\in\mathcal{S}_{t}$ means selecting an action $A_{t}=Q_{t}(p_{t})\in\mathcal{A}_{t}$ and vice versa. With all these preparations, the following Lemma \ref{lem:1} rigorously forms our single-episode pricing problem into a perturbed linear bandit.

\begin{assumption}\label{ass:1}
	$F$ is Lipschitz with the Lipschitz constant $L$.
\end{assumption}

\begin{lemma}\label{lem:1}
	Under Assumption \ref{ass:1}, $||\xi_{t}-\xi^{*}||_{\infty}\leq L||\hat{\theta}-\theta_{0}||_{1},\forall t\in [T_{0}]$. Moreover, under the price-action coupling $A_{t} = Q_{t}(p_{t})$, the reward $Z_{t} = p_{t}1_{\{v_{t}\geq p_{t}\}}$, the parameter $\xi_{t}$ and the action set $\mathcal{A}_{t}$ form a perturbed linear bandit with a perturbation constant $2L||\hat{\theta}-\theta_{0}||_{1}$.
\end{lemma}

Lemma \ref{lem:1} implies that the perturbation is proportional to the $\ell_{1}$ estimation error $||\hat{\theta}-\theta_{0}||_{1}$. If the estimate $\hat{\theta}=\theta_{0}$, then $\xi_{t} = \xi^{*}$ with zero perturbation and the PLB reduces to a classic linear bandit. On the other hand, a worse $\hat{\theta}$ implies a larger perturbation, thus incurring more difficulty in solving the PLB and potentially leading to a larger regret.


According to Lemma \ref{lem:1}, $Z_{t} = A_{t}^{\top}\xi_{t}+\eta_{t}$ with $||\xi_{t}-\xi^{*}||_{\infty}\leq L||\hat{\theta}-\theta_{0}||_{1}$, and hence $\xi^*$ can be estimated from historical data. Similar to that in linear bandit \citep{lattimore2020bandit}, we employ the ridge estimator $\hat{\xi}_{t-1} = V_{t-1}(\lambda)^{-1}\sum_{s=1}^{t-1}A_{s}Z_{s}$, where $V_{t-1}(\lambda) = \lambda I+\sum_{s=1}^{t-1}A_{s}A_{s}^{\top}$ with the tuning parameter $\lambda>0$. Note that in Algorithm 3, we use $j_{t}$ to denote the arm pulled at time $t$. Let $\mathcal{U}_{t-1,j} = \{s:1\leq s\leq t-1,j_{s} = j \}$. Since $A_{s}$'s have a single nonzero element and $V_{t-1}(\lambda)$ is a diagonal matrix, we obtain the explicit form for the $j$-th element of $\hat{\xi}_{t-1}$ as $\hat{\xi}_{t-1,j}= \frac{\sum_{s\in\mathcal{U}_{t-1,j}}p_{s}^{2}y_{s}}{\lambda+\sum_{s\in\mathcal{U}_{t-1,j}}p_{s}^{2}}$, which serves as the estimate $\text{EST}_{t}(1-F(m_{j}))$ for $1-F(m_{j})=\xi^{*}_{j}$.

In order to construct a UCB using the principle of optimism in the face of uncertainty, we then compute a confidence radius $\text{CR}_{t}(1-F(m_{j}))$ of the above estimate $\text{EST}_{t}(1-F(m_{j}))=\hat{\xi}_{t-1,j}$. The common confidence set $\mathcal{C}_{t}(\beta_{t}) = \{\xi\in\mathbb{R}^{d}: ||\xi- \hat{\xi}_{t-1}||_{V_{t-1}(\lambda)}^{2}\leq \beta_{t} \}$ yields a marginal confidence radius for each $\hat{\xi}_{t-1,j}$.  Due to the simple form of $V_{t-1}(\lambda)$, we obtain an explicit form $\text{CR}_{t}(1-F(m_{j})) =\sqrt{\frac{\beta_{t}}{\lambda + \sum_{s\in\mathcal{U}_{t-1,j}}p_{s}^{2}}}$. Then we obtain the UCB as required in Inner Algorithm B,
\begin{equation}\label{eq:4}
\text{UCB}_{t}(1-F(m_{j})) = \frac{\sum_{s\in\mathcal{U}_{t-1,j}}p_{s}^{2}y_{s}}{\lambda+\sum_{s\in\mathcal{U}_{t-1,j}}p_{s}^{2}} + \sqrt{\frac{\beta_{t}}{\lambda + \sum_{s\in\mathcal{U}_{t-1,j}}p_{s}^{2}}}.
\end{equation}

Motivated by the linear bandit \citep{lattimore2020bandit}, we specify the parameter $\beta_{t} =\beta_{t}^{*}= p_{\max}^{2}(1\vee (\frac{1}{p_{\max}}\sqrt{\lambda d}+\sqrt{2\log(\frac{1}{\delta})+d\log(\frac{d\lambda+(t-1)p_{\max}^{2}}{d\lambda}) })^{2})$. Here $1-\delta$ is the confidence level and $\delta = \frac{1}{T_{0}}$ is a typical choice \citep{lattimore2020bandit} with known $T_{0}$. Thus we use $\delta = \frac{1}{2^{k-2}\ell_{2}}$ for the application of Inner Algorithm B to the $k$-th episode with an expected length of $2^{k-2}\ell_{2}$. Now we are ready to present the full version of our DIP policy as Algorithm \ref{alg:4}. In summary, DIP well organizes two sub-algorithms across episodes, one applying classification for linear parameter estimation and the other adapting the UCB idea for online pricing.

\begin{algorithm}[h!]
	\caption{DIP for Contextual Dynamic Pricing}\label{alg:4}
	\begin{algorithmic}[1]
		\STATE \textbf{Input:} \textbf{(at time 0)} $\alpha_{1},\alpha_{2},p_{\max},C,\lambda,W$
		\STATE \textbf{Input:} \textbf{(arrives over time)} covariates $\{x_{t}\}_{t\in[T]}$
		\STATE \textbf{For} time $t=1,\dots,\ell_{1} (= \alpha_{1})$, \textbf{do}
		\STATE \hspace{0.075in} Set a price $p_{t}$ randomly from $(0,p_{\max})$ and receive a binary response $y_{t}$. 
		\STATE \textbf{For} episodes $k=2,3,\dots,n(=n(T,\alpha_{1},\alpha_{2}))$, \textbf{do}
		\STATE \hspace{0.075in} Apply Inner Algorithm A with the input data $\{(x_{t},p_{t},y_{t}) \}_{\sum_{i=1}^{k-2}\ell_{i}+1\leq t\leq \sum_{i=1}^{k-1}\ell_{i}}$ and $W$ \\
		\hspace{0.075in} to obtain the estimate $\hat{\theta}_{k-1}$;
		\STATE \hspace{0.075in} Apply Inner Algorithm B on the coming sequential covariates $\{x_{t}\}_{\sum_{i=1}^{k-1}\ell_{i}+1\leq t\leq \sum_{i=1}^{k}\ell_{i}}$, \\
		\hspace{0.075in} with the estimate $\hat{\theta}_{k-1}$, discretization number $d_{k} = C\lceil (2^{k-2}\ell_{2})^{\frac{1}{6}}\rceil$ and the UCB\\
		\hspace{0.075in} construction in (\ref{eq:4}) with $\beta_{t} = \beta_{t}^{*}$ and $\delta = \frac{1}{2^{k-2}\ell_{2}}$.
	\end{algorithmic}
\end{algorithm}

\begin{remark}
	In this remark, we provide the computational complexity of Algorithm \ref{alg:4}. In each episode $k$, the Inner Algorithm A consists of a linear classification procedure and a projection with the complexity of $O(\ell_{k}d_{0})$ and $O(d_{0})$ respectively. Thus it contributes a complexity of $O(d_{0}T+d_{0}\log T) = O(d_{0}T)$ to the total horizon. The Inner algorithm B in episode $k$ first conducts a discretization with the complexity of $O(d_{k}) = O(\ell_{k}^{1/6})$. At its $t$-th iteration, we first locate the estimated linear valuation component $x_{t}^{\top}\hat{\theta}_{k}$ with a complexity of $O(d_{0})$. Then we calculate the expected revenue UCB for each candidate price. Each past data point is used at most once in this procedure by Equation (\ref{eq:4}) and thus the total UCB calculation complexity is $O(t)$. Then the overall complexity of Inner Algorithm B in episode $k$ is $O(d_{k}+d_{0}\ell_{k}+\ell_{k}^{2}) = O(d_{0}\ell_{k}+\ell_{k}^{2})$. Thus Inner Algorithm B contributes a total complexity of $O(d_{0}T+T^{2})$ to the entire horizon. Hence, the computational complexity of the whole DIP policy is $O(d_{0}T+T^{2})$. 
\end{remark}

Finally, we would like to mention that the proposed perturbed linear bandit framework can be used beyond the above contextual dynamic pricing problem. This motivates us to introduce a general algorithm called M-LinUCB in Algorithm \ref{alg:5} for the perturbed linear bandit framework $Z_{t} = \langle \xi_{t},A_{t}\rangle +\eta_{t}$ when any potential action has only one nonzero element. For any vector $v$ with a single nonzero element, denote $\delta(v)$ as the index of this nonzero element. For instance, $\delta((0,1,0)^{\top}) = 2$. Further define $\tilde{\mathcal{B}}_{t} = \{\delta(a): a\in \mathcal{A}_{t}\}$ as the nonzero index set of all potential actions at time $t$ and $\tilde{\mathcal{B}}^{'}_{t} = \{\delta(A_{s}):s\in[t-1]\}$ as the nonzero index set of all past selected actions. Then, bridged by the PLB formulation of our single-episode pricing problem, there exists a close connection between M-LinUCB and Inner Algorithm B formalized in Lemma \ref{lem:2} below.



\begin{algorithm}[h!]
	\caption{M-LinUCB for Perturbed Linear Bandit}\label{alg:5}
	\begin{algorithmic}[1]
		\STATE \textbf{Input:} \textbf{(arrives over time)} action sets $\mathcal{A}_{t}$, $\lambda, \{\beta_{t}\}_{t\in[T]}$
		\STATE \textbf{For} $t = 1,\dots,T$, \textbf{do}
		\STATE \hspace{0.075in} Determine $\tilde{\mathcal{B}}_{t} = \{\delta(a): a\in \mathcal{A}_{t}\}$ and $\tilde{\mathcal{B}}^{'}_{t} = \{\delta(A_{s}):s\in[t-1]\}$.
		\STATE \hspace{0.075in} \textbf{If} $\tilde{\mathcal{B}}_{t}\not\subseteq \tilde{\mathcal{B}}^{'}_{t}$, \textbf{do}
		\STATE \hspace{0.225in} Choose an arbitrary $A_{t}\in\mathcal{A}_{t}$ such that $\delta(A_{t})\notin \tilde{\mathcal{B}}^{'}_{t}$.
		\STATE \hspace{0.075in} \textbf{If} $\tilde{\mathcal{B}}_{t}\subseteq \tilde{\mathcal{B}}^{'}_{t}$, \textbf{do}
		\STATE \hspace{0.225in} \textbf{For} $a\in\mathcal{A}_{t}$, \textbf{do}
		\STATE \hspace{0.375in} Calculate $\text{LinUCB}_{t}(a) = \max_{\xi\in \mathcal{C}_{t}(\beta_{t})}\langle \xi,a\rangle$ where $\mathcal{C}_{t}(\beta_{t}) = \{\xi\in\mathbb{R}^{d}: ||\xi- $\\
		\hspace{0.375in} $\hat{\xi}_{t-1}||_{V_{t-1}(\lambda)}^{2}\leq \beta_{t} \}$ and $\hat{\xi}_{t-1} = V_{t-1}(\lambda)^{-1}\sum_{s=1}^{t-1}A_{s}Z_{s},V_{t-1}(\lambda) = \lambda I+\sum_{s=1}^{t-1}A_{s}A_{s}^{\top}$.
		\STATE \hspace{0.225in} Choose $A_{t} \in \arg\max_{a\in\mathcal{A}_{t}}\text{LinUCB}_{t}(a)$.	
		\STATE \hspace{0.075in} Receive a reward $Z_{t}$.
	\end{algorithmic}
\end{algorithm}


\begin{lemma}\label{lem:2}
	Applying Algorithm \ref{alg:5} to the PLB formulation of our single-episode pricing problem with $\beta_{t} = \beta_{t}^{*}= p_{\max}^{2}(1\vee (\frac{1}{p_{\max}}\sqrt{\lambda d}+\sqrt{2\log(\frac{1}{\delta})+d\log(\frac{d\lambda+(t-1)p_{\max}^{2}}{d\lambda}) })^{2})$ yields Algorithm \ref{alg:3} using the UCB construction (\ref{eq:4}) with $\beta_{t} = \beta_{t}^{*}$.
\end{lemma}

Therefore, Inner Algorithm B (Algorithm \ref{alg:3}) can be viewed as the ``projection'' of M-LinUCB onto our single-episode pricing problem. In the remaining part of this paper, without further specifications, we refer to Algorithms 3 and 5 as the ones mentioned in Lemma \ref{lem:2}.

\section{Theory}
\label{sec:3}

In this section, we establish the regret bound for the proposed DIP policy. As DIP divides the total time horizon into episodes, we conduct the regret analysis on a single episode and then merge them together. For the single-episode pricing problem, our discretization procedure leads to a natural decomposition of the regret into a discrete part and a continuous part. One key technical contribution is the proof of the discrete-part regret, which is shown via the equivalent regret of M-LinUCB for the corresponding PLB framework.

In our single-episode regret analysis, we denote the total horizon as $[T_{0}]$ and use $\hat{\theta}$ as the input for Algorithm \ref{alg:3}. In Algorithm \ref{alg:3}, we restrict the price in a discrete candidate set $\mathcal{S}_{t}$, thus yielding a ``discrete" best price $\tilde{p}_{t}^{*}$ in $\mathcal{S}_{t}$, i.e., $\tilde{p}^{*}_{t} \in \arg\max_{p\in \mathcal{S}_{t}}p(1-F(p-x_{t}^{\top}\theta_{0}))$. Thus the regret $r_t$ in $(\ref{eqn:regret})$ can be rewritten as 
$$
\underbrace{\tilde{p}^{*}_{t}(1-F(\tilde{p}^{*}_{t}-x_{t}^{\top}\theta_{0}))-p_{t}(1-F(p_{t}-x_{t}^{\top}\theta_{0}))}_{r_{t,1}}+\underbrace{p^{*}_{t}(1-F(p^{*}_{t}-x_{t}^{\top}\theta_{0}))-\tilde{p}^{*}_{t}(1-F(\tilde{p}^{*}_{t}-x_{t}^{\top}\theta_{0}))}_{r_{t,2}}.
$$ 
The first part $r_{t,1}$ is the reward loss with respect to the discrete best price $\tilde{p}_{t}^{*}$. The second part $r_{t,2}$ is the regret of setting $\tilde{p}_{t}^{*}$. Denote their sums as $R_{T_{0},1} = \sum_{t=1}^{T_{0}}r_{t,1}$ and $R_{T_{0},2} = \sum_{t=1}^{T_{0}}r_{t,2}$, which are the discrete-part and continuous-part regrets respectively. Then bounding the cumulative regret $R_{T_{0}}=R_{T_{0},1}+R_{T_{0},2}$ reduces to bounding $R_{T_{0},1}$ and $R_{T_{0},2}$ separately. As discussed before, the discrete-part regret would be shown to be the same as the regret under the equivalent PLB framework and then investigated by utilizing newly-developed regret bounds for the PLB setting. For the continuous-part regret, we adopt the following concavity assumption on the general expected revenue function defined as $f_{q}(p) = p(1-F(p-q))$. Note that  the single-step continuous-part regret $r_{t,2}$ can then be rewritten as $f_{x_{t}^{\top}\theta_{0}}(p^{*}_{t})-f_{x_{t}^{\top}\theta_{0}}(\tilde{p}^{*}_{t})$. 
\begin{assumption}\label{ass:2}
	There exists a constant $C$ such that for any $q=x^{\top}\theta_{0}$ and $x\in\mathcal{X}$, we have $f_{q}(p^{*}(x))-f_{q}(p)\leq C(p^{*}(x)-p)^{2}, \forall p\in[0,p_{\max}]$.
\end{assumption}
Assumption \ref{ass:2} requires that the reward difference between the overall best price and any other price can be bounded by a constant multiplying their quadratic difference. Given the global continuity of $F$, Assumption \ref{ass:2} indicates a uniform control of $f_{x^{\top}\theta_{0}}(p)$ over the local neighborhoods of the maximizers $p^{*}(x)$. In Proposition \ref{prop:2}, by applying the Taylor's theorem with the Lagrange remainder, we provide a sufficient condition for Assumption \ref{ass:2}. Nevertheless, Assumption \ref{ass:2} does not require any global smoothness of $F$. The derived regret bound would still hold for locally erratic $F$'s as long as Assumption \ref{ass:2} is satisfied.


\begin{proposition}\label{prop:2}
	Assumption \ref{ass:2} holds if $F^{''}(\cdot)$ is bounded on $[-||\theta_{0}||_{1},p_{\max}+||\theta_{0}||_{1}]$.
\end{proposition}

Now we present our main result in the following Theorem \ref{thm:3}. It provides a regret upper bound over the entire horizon. 


\begin{theorem}\label{thm:3}
	Under Assumptions \ref{ass:1} -- \ref{ass:2}, the DIP policy yields the expected regret
	\[
	\mathbb{E}(R_{T}) =  \tilde{O}(T^{2/3}) +4p_{\max}L\sum_{k=2}^{n}2^{k-2}\ell_{2}\mathbb{E}||\hat{\theta}_{k-1}-\theta_{0}||_{1}.
	\]
\end{theorem}

Theorem \ref{thm:3} demonstrates how the estimation errors $||\hat{\theta}_{k}-\theta_{0}||_{1}$ affect the regret upper bound for DIP policy. If the estimates $\{\hat{\theta}_{k}\}_{k\in[n-1]}$ are perfectly accurate, the second term vanishes and the overall regret is of $\tilde{O}(T^{2/3})$. In general, if $\mathbb{E}||\hat{\theta}_{k}-\theta_{0}||_{1} =  O(\ell_{k}^{-\alpha})$ for some $0<\alpha\leq\frac{1}{2}$, we can conclude that $\sum_{k=2}^{n}2^{k-2}\ell_{2}\mathbb{E}||\hat{\theta}_{k-1}-\theta_{0}||_{1}= O(T^{1-\alpha})$ by the doubling construction. Then the overall regret is $\tilde{O}(T^{\frac{2}{3}\vee (1-\alpha)})$. Since we use the adaptive pricing data in the previous episode to estimate $\theta_{0}$, it is challenging to derive the exact rate of convergence for the estimation. In spite of this theoretical difficulty, we conduct a simulation study in Section \ref{sec:sim_estimation_error} to numerically demonstrate that the convergence rate of $||\hat{\theta}_{k}-\theta_{0}||_1$ is between $1/3$ and $1/2$ and hence the $\tilde{O}(T^{2/3})$ overall regret bound can be practically achieved.

\begin{remark}\label{rm:lower-bound}
At the first glance, the obtained regret upper bound is worse than the typical $\Omega(T^{1/2})$ lower bound in linear bandit \citep{lattimore2020bandit} and dynamic pricing with the  known noise distribution \citep{javanmard2019dynamic}. However, we would like to point out that our problem involves both unknown linear parameter $\theta_{0}$ and unknown noise distribution $F$. We conjecture that our obtained regret upper bound is close to the lower bound in our setting. To see it, \cite{chen2021nonparametric} considered a nonparametric pricing problem and proved an $\Omega(T^{(d_{0}+2)/(d_{0}+4)})$ lower bound under some additional smoothness and concavity assumptions, where $d_{0}$ is the dimension of the nonparametric component. With a one-dimensional nonparametric component $F$ in our pricing problem, their results suggest an $\Omega(T^{3/5})$ lower bound, which is higher than the typical $\Omega(T^{1/2})$ rate. However, their constructed instances do not fit into our considered pricing problem with an additional linear structure $x^{\top}\theta_{0}$. The additional unknown $\theta_{0}$ makes the lower bound derivation harder and we leave it for future work.

\end{remark}

\begin{remark}\label{rm:3}	
After our initial submission, there are two recent papers \citep{fan2021policy,xu2022towards} considering a similar dynamic pricing problem with the unknown noise distribution. In \cite{fan2021policy}, the authors considered an $m(\geq 2)$ times continuously differentiable $F$ and proposed an explore-then-commit-type policy which achieved an $\tilde{O}(T^{\frac{2m+1}{4m-1}})$ regret upper bound. Both our assumed Lipschitz and concavity assumptions are satisfied under their condition of twice continuously differentiable $F$ ($m=2$). Thus, even under stronger assumptions, their proved $\tilde{O}(T^{5/7})$ regret for $m=2$ is still worse than our main $\tilde{O}(T^{2/3})$ regret term. In \cite{xu2022towards}, the authors considered an adversarial setting and proposed a ``D2-EXP4'' policy that achieved a regret of $\tilde{O}(T^{3/4})$. By fully utilizing the smoothness and concavity of the noise distribution, our proposed DIP policy achieves an improvement to $\tilde{O}(T^{2/3})$ for our main regret term. 

\end{remark}

In the next two subsections, we first do some preparations by developing the regret bounds for the general perturbed linear bandit. Then we provide a proof outline for our main Theorem \ref{thm:3} by utilizing the proved PLB results.

\subsection{Regret Bounds for Perturbed Linear Bandit}
We consider a PLB setting with the reward model $Z_{t} = \langle \xi_{t},A_{t}\rangle + \eta_{t}$ which satisfies the following conditions. 
\begin{description}
	\item[\textbf{Condition 1}] For any $t\in \mathbb{N}^{+}$ and $a\in \mathcal{A}_{t}$, $|\langle \xi_{t},a\rangle|\leq 1$. 
	
	\item[\textbf{Condition 2}] For any $t\in\mathbb{N}^{+}$, $||\xi_{t}||_{\infty}\leq C_{1}$. 
	
	\item[\textbf{Condition 3}] For any $t\in\mathbb{N}^{+}$ and $a\in\mathcal{A}_{t}$, $||a||_{0} = 1$ and $||a||_{2}\leq a_{\max}$ for a constant $a_{\max}$. 
	
	\item[\textbf{Condition 4}] For any $t\in\mathbb{N}^{+}$, $\eta_{t}$ is a $1$-conditionally sub-Gaussian random variable, i.e., $\mathbb{E}(\exp(\alpha\eta_{t})|\mathcal{F}_{t-1})\leq \exp(\frac{\alpha^{2}}{2})$, where $\mathcal{F}_{t-1} = \sigma(\xi_{1},A_{1},Z_{1},\dots, \xi_{t},A_{t})$.
\end{description}

\begin{remark}
	Condition 1 ensures a constant regret upper bound at each time and is commonly adopted in linear bandit \citep{lattimore2020bandit}. Condition 2 assumes the bounded infinity norm of $\xi_{t}$. Condition 3 implies there is only one nonzero element bounded in absolute value for any action. This holds for our PLB formulation since any action vector $Q_{t}(p)$ with $p=m_{j}+x_{t}^{\top}\hat{\theta} \in\mathcal{S}_{t}$ has a single nonzero $j$-th element $p\in(0,p_{\max})$. Condition 4 implies that the noise is sub-Gaussian conditional on all the past parameters, actions, rewards, as well as the current parameter and action. The perturbed linear bandit formulation of our single-episode pricing problem satisfies all these conditions.
\end{remark}

We develop the following Lemma \ref{lem:3} to establish the regret bound for such a PLB setting.

\begin{lemma}\label{lem:3}
	Consider the PLB satisfying Conditions 1-4 with a perturbation $C_{p}$. With probability at least $1-\delta$, Algorithm \ref{alg:5} with $\beta_{t} = \tilde{\beta}_{t} =1\vee (C_{1}\sqrt{\lambda d}+\sqrt{2\log(\frac{1}{\delta})+d\log(\frac{d\lambda+(t-1)a_{\max}^{2}}{d\lambda}) })^{2}$ has the regret bound
	\[
	R^{PLB}_{T_{0}}\leq 2\sqrt{2dT_{0}\tilde{\beta}_{T_{0}}\log(\frac{d\lambda+T_{0}a_{\max}^{2}}{d\lambda})}+  2a_{\max}C_{p}T_{0}+2d.
	\]
\end{lemma}

\noindent\textbf{Proof Sketch}: We construct a new sequence of ``shadow'' linear parameters $\{\dot{\xi}_{t}\}_{2\leq t\leq T_{0}}$ and control the ``pseudo-regret" $\sum_{t=2}^{T_{0}}\langle \dot{\xi}_{t},\dot{A}_{t}-A_{t}\rangle$ with the sub-linear order $\tilde{O}(\sqrt{T_{0}})$, where $\dot{A}_{t} =\arg\max_{a\in \mathcal{A}_{t}}\langle \dot{\xi}_{t},a\rangle$. By proving closeness of $\dot{\xi}_{t}$ and $\xi_{t}$ for all $t$, we can bound the difference between the true regret and pseudo-regret by a linear term proportional to the perturbation $C_{p}$. The detailed construction of $\{\dot{\xi}_{t}\}_{2\leq t\leq T_{0}}$ and rigorous proofs are deferred to Section A of the Appendix. $\hfill\blacksquare$


As shown in Lemma \ref{lem:3}, the second term in the regret upper bound is proportional to the perturbation $C_{p}$. When $C_{p} = 0$, this linear term vanishes and the final regret bound matches that of the classic linear bandit. Reversely, the perturbed linear bandit would become intractable when $C_p$ is too large. Interestingly, by Lemma \ref{lem:1}, this perturbation constant in the PLB formulation of our single-episode pricing problem is proportional to $||\hat{\theta}-\theta_{0}||_{1}$. This matches the intuition that a larger estimation error would lead to more revenue loss. Lemma \ref{lem:3} is of independent interest since it provides an informative regret bound for the PLB problem.

\begin{remark}\label{rm:1}
	Our proposed PLB can be viewed as a misspecified linear bandit \citep{lattimore2020learning,pacchiano2020model,foster2020adapting} with a misspecification level $\epsilon_{*} = a_{\max}C_{p}/2$, where the latter has a general regret of $\tilde{O}(d\sqrt{T_{0}}+\epsilon_{*}\sqrt{d}T_{0})$. In comparison, by leveraging Condition 3, we prove a regret of $\tilde{O}(d\sqrt{T_{0}}+a_{\max}C_{p}T_{0}) = \tilde{O}(d\sqrt{T_{0}}+\epsilon_{*}T_{0})$ for the simple M-LinUCB algorithm under our PLB setting. To see it, the key in our proof is the closeness property $||\dot{\xi}_{t}-\xi_{t}||_{\infty}\leq C_{p}$ of our constructed shadow parameters $\dot{\xi}_{t} = V_{t-1}^{+}\sum_{s=1}^{t}A_{s}A_{s}^{\top}\xi_{s}$. The proof of this property relies on the fact that the Moore-Penrose inverse $V_{t-1}^{+}$ of $V_{t-1} = V_{t-1}(0)$ is a diagonal matrix, which is a direct result of the condition $||a||_{0} = 1$ in Condition 3. Importantly, this $\sqrt{d}$ improvement is critical for us to derive the final regret rate of our pricing problem. 
\end{remark}

\begin{remark}\label{rm:2}
	Non-stationary linear bandits (NLB) \citep{cheung2018hedging,russac2019weighted,zhao2020simple} also allow changing linear parameters $\xi_{t}$ but design policies to adapt to the smooth variations $B_{T_{0}}=\sum_{t=1}^{T_{0}-1}||\xi_{t}-\xi_{t+1}||_{2}$. Our PLB setting fits an NLB with linear variations $B_{T_{0}}=O(C_{p}T_{0})$. 
The nonasymptotic results in \cite{cheung2018hedging,zhao2020simple} suggest a regret of $\mathcal{\tilde{O}}(B_{T_{0}}^{1/3}T_{0}^{2/3})=\mathcal{\tilde{O}}(C_{p}^{1/3}T_{0})$ which is only valid for a range of $C_{p}$ (exclusive of zero and dependent on $T_{0}$). In contrast, our proved Lemma \ref{lem:3} provides regret behaviors with a fixed $T_{0}$ for $C_{p}\to 0$, i.e., approaching the classic linear bandit result $\mathcal{\tilde{O}}(\sqrt{T_{0}})$ linearly with $C_{p}$, which is essential for further derivations in our pricing problem. Though some intermediate results in \cite{cheung2018hedging,zhao2020simple} also yield regrets for fixed $T_{0}$ and $C_{p}\to 0$, they suggest worse regrets such as $\mathcal{\tilde{O}}(wC_{p}T_{0}+\frac{T_{0}}{\sqrt{w}})$ ($w$ chosen from $\{1,\dots,T_{0}\}$) and $\mathcal{\tilde{O}}(C_{p}T^{2}_{0}+\sqrt{T_{0}})$ when applied to our PLB setting, which will inevitably deteriorate the performance guarantee for our pricing problem. 
	
\end{remark}

Next we prove an $\Omega(C_pT_{0})$ regret lower bound for the PLB with a perturbation $C_p$. This implies that the linear term in the upper bound is inevitable due to the potentially adversarial perturbations. Define $PB(\tilde{\xi},C_{p}) = \{\xi\in \mathbb{R}^{d}:||\xi-\tilde{\xi}||_{\infty}\leq \frac{C_p}{2} \}$ as a parameter set with respect to a ``central" parameter $\tilde{\xi}$ and a perturbation quantification $C_p$. 

\begin{proposition}\label{prop:1}
	For any PLB algorithm $\mathcal{A}^{*}$, any $\tilde{\xi}$ with all positive elements and $\frac{C_p}{2}<\min_{i\in [d]}\tilde{\xi}_{i}$, there exists a PLB with parameters $(\xi_{1},\dots,\xi_{t},\dots)$ and action sets $(\mathcal{A}_{1},\dots,\mathcal{A}_{t},\dots)$ satisfying $\xi_{t}\in PB(\tilde{\xi},C_p),\forall t\in\mathbb{N}^{+}$ and a constant $C_{0}$ only dependent on $\tilde{\xi}$ such that \[
	\mathbb{E}(R^{PLB}_{T_{0}}(\mathcal{A}^{*}))\geq C_{0}C_pT_{0}, \forall T_{0}\in\mathbb{N}^{+}.
	\]
\end{proposition}

\subsection{Proof Outline for Theorem \ref{thm:3}}

To prove Theorem \ref{thm:3}, we first prove regret upper bounds for each episode and then merge them together. In the following Proposition \ref{thm:2}, we prove a high probability regret bound as well as an expected regret bound for our pricing policy in a single episode. Specifically, the expected regret is bounded by a sub-linear $\tilde{O}(T^{2/3})$ term and a linear term proportional to the $\ell_1$ estimation error $||\hat{\theta}-\theta_{0}||_{1}$. As DIP applies Algorithm 3 to  the $k$-th episode with $\hat{\theta} = \hat{\theta}_{k-1}$, we obtain Theorem \ref{thm:3} by applying Proposition \ref{thm:2} to each episode. 

\begin{proposition}\label{thm:2}
	Under Assumptions \ref{ass:1} -- \ref{ass:2}, with probability at least $1-\delta$, applying Algorithm 3 on the single-episode pricing problem yields the total regret $R_{T_{0}}$ satisfying
	\[
	R_{T_{0}}\leq 2\sqrt{2dT_{0}\beta_{T_{0}}^{*}\log(\frac{d\lambda+T_{0}p_{\max}^{2}}{d\lambda})} +  4p_{\max}L||\hat{\theta}-\theta_{0}||_{1}T_{0} + C\frac{T_{0}}{d^{2}}+2dp_{\max}.
	\]
	Moreover, by setting $\delta = \frac{1}{T_{0}},d = C\lceil T_{0}^{1/6}\rceil$ and taking the expectation, we have $\mathbb{E}(R_{T_{0}}) = \tilde{O}(T_{0}^{2/3}) + 4p_{\max}L||\hat{\theta}-\theta_{0}||_{1}T_{0}$.
\end{proposition}

It remains to prove the regret bound in Proposition \ref{thm:2} for the single-episode pricing problem. We conduct the analysis of both the discrete-part and continuous-part regret, and then combine them together. 

\noindent\textbf{Discrete-Part Regret}: By the PLB formulation in Lemma \ref{lem:1} and the one-to-one correspondence between $\mathcal{S}_{t}$ and $\mathcal{A}_{t}$, the best action in $\mathcal{A}_{t}$ is $A^{*}_{t} = Q_{t}(\tilde{p}_{t}^{*})$. Therefore, the selected action $A_{t}=Q_{t}(p_{t})$ yields the regret $\tilde{p}^{*}_{t}(1-F(\tilde{p}^{*}_{t}-x_{t}^{\top}\theta_{0}))-p_{t}(1-F(p_{t}-x_{t}^{\top}\theta_{0}))$ for the PLB, which matches the discrete-part regret $r_{t,1}$. Moreover, Lemma \ref{lem:2} shows that Algorithm \ref{alg:5} yields Algorithm \ref{alg:3} under the price-action coupling. Therefore, we can investigate the regret of Algorithm \ref{alg:5} on the PLB to quantify the discrete-part regret of Algorithm \ref{alg:3}.

We now apply the general regret bound of Lemma \ref{lem:3} to the PLB formulation of our single-episode pricing problem to bound the discrete-part regret. After scaling the rewards, linear parameters and noises by $\frac{1}{p_{\max}}$ as $\tilde{\xi}_{t} = \frac{1}{p_{\max}}\xi_{t}$, $\tilde{Z}_{t} = \frac{1}{p_{\max}}Z_{t}$, $\tilde{\eta}_{t} = \frac{1}{p_{\max}}\eta_{t}$, we obtain the transformed model $\tilde{Z}_{t} = \langle \tilde{\xi}_{t},A_{t}\rangle + \tilde{\eta}_{t}$ with the perturbation constant $\tilde{C}_{p} = \frac{2L||\hat{\theta}-\theta_{0}||_{1}}{p_{\max}}$, which satisfies Conditions 1-4 with $C_{1} = \frac{1}{p_{\max}}$ and $a_{\max} = p_{\max}$. On the other hand, we can prove that applying Algorithm 5 with $\beta_{t} = \beta_{t}^{*}$ on the original PLB is equivalent to applying it with $\beta_{t} = \tilde{\beta}_{t} = \frac{1}{p_{\max}^{2}}\beta^{*}_{t}$ on the transformed model, with their regrets admitting a scaling relationship. By formalizing the above reasoning, we obtain the following Proposition \ref{thm:1}.

\begin{proposition}\label{thm:1}
	Under Assumption \ref{ass:1}, with probability at least $1-\delta$, applying Algorithm 3 on the single-episode pricing problem yields a discrete-part regret $R_{T_{0},1}$ satisfying
	\[
	R_{T_{0},1}\leq 2\sqrt{2dT_{0}\beta_{T_{0}}^{*}\log(\frac{d\lambda+T_{0}p_{\max}^{2}}{d\lambda})}+  4p_{\max}L||\hat{\theta}-\theta_{0}||_{1}T_{0}+2dp_{\max}.
	\]
\end{proposition}

Proposition \ref{thm:1} provides an upper bound of the discrete-part regret on a single episode. The first term is sub-linear as $\tilde{O}(\sqrt{T_{0}})$ while the second term is linear in $T_{0}$ and proportional to the estimation error $||\hat{\theta}-\theta_{0}||_{1}$, which invokes the perturbation in our PLB formulation. The third term will be dominated by the first two terms as we further specify $d$ to yield a best tradeoff between discrete and continuous parts of the regret.

\noindent\textbf{Continuous-Part Regret}: We now discuss how to derive a bound for the continuous-part regret under Assumption \ref{ass:2}. By our discretization approach, $\{m_{i}+x_{t}^{\top}\hat{\theta} \}_{i\in[d]}$ are a sequence of points that ``cover" $[0,p_{\max}]$ with equal adjacent distance $\frac{p_{\max} + 2\|\hat{\theta}\|_1}{d}$. Since $\mathcal{S}_{t} = \{m_{j}+x_{t}^{\top}\hat{\theta}|j\in [d],m_{j}+x_{t}^{\top}\hat{\theta}\in(0,p_{\max})\}$ and $p_{t}^{*}\in(0,p_{\max})$, there must exist a $\dot{p}_{t}\in\mathcal{S}_{t}$ close enough with $p_{t}^{*}$ such that their expected reward difference is $O(\frac{1}{d^{2}})$ according to Assumption \ref{ass:2}. Since the discrete best price $\tilde{p}_{t}^{*}$ outperforms $\dot{p}_{t}$, the unit continuous-part regret $r_{t,2}$ of setting $\tilde{p}_{t}^{*}$ satisfies $r_{t,2} = O(\frac{1}{d^{2}})$. Thus the continous-part regret $R_{T_{0},2}$ in the entire horizon is of the order $O(\frac{T_{0}}{d^{2}})$.

\noindent\textbf{Combination}: We can prove that the right-hand side of the regret result in Proposition \ref{thm:1} has a simpler form of $\tilde{O}(d\sqrt{T_{0}}) + 4p_{\max}L||\hat{\theta}-\theta_{0}||_{1}T_{0}$. Thus the overall regret for the single episode is $\tilde{O}(d\sqrt{T_{0}}+\frac{T_{0}}{d^{2}}) + 4p_{\max}L||\hat{\theta}-\theta_{0}||_{1}T_{0}$. By setting the discretization number $d$ in the order of $T_{0}^{1/6}$, we obtain the single-episode regret bounds in Proposition \ref{thm:2}.

\section{Simulation Study}
\label{sec:4}

We demonstrate the performance of our DIP policy on synthetic datasets and compare it with RMLP and RMLP-2 proposed by \cite{javanmard2019dynamic}. The implementation details of DIP, RMLP and RMLP-2 are provided in Section B of the Supplement. 

Let $\Phi(\mu,\sigma^{2})$ denote the CDF of $N(\mu,\sigma^{2})$ distribution. For the first six examples, we consider a scalar covariate $x_{t}\overset{\text{i.i.d.}}{\sim} \text{Unif}[0,1]$ and set $\theta_{0} = 30$. The CDF $F$ of the noise distribution is designed as follows, where Examples 1 and 5 are motivated from the real application in Section \ref{sec:5}. Their probability density functions (PDFs) are shown in Figure \ref{fig:4}. 
\begin{description}
	\item[Example 1.] The true $F = \frac{1}{2}\Phi(-4,6) + \frac{1}{2}\Phi(4,6)$. 
	\item[Example 2.] The true $F = \frac{1}{3}\Phi(-6,\frac{\pi^{2}}{3}) + \frac{1}{3}\Phi(-1,\frac{\pi^{2}}{3})+\frac{1}{6}\Phi(1,\frac{\pi^{2}}{3}) + \frac{1}{6}\Phi(6,\frac{\pi^{2}}{3})$. 	
	\item[Example 3.] The true $F = \frac{1}{4}\Phi(-7,\frac{\pi^{2}}{3}) + \frac{1}{4}\Phi(-3,\frac{\pi^{2}}{3})+\frac{1}{4}\Phi(3,\frac{\pi^{2}}{3}) + \frac{1}{4}\Phi(7,\frac{\pi^{2}}{3})$.
	\item[Example 4.] The true $F(\cdot) = \tilde{F}(\cdot + \text{mean}(\tilde{F}))$ where $\tilde{F} = \frac{1}{3}\Phi(-3,\frac{\pi^{2}}{3}) + \frac{2}{3}\Phi(3,\frac{\pi^{2}}{3})$.
	\item[Example 5.] The true $F(\cdot) = \tilde{F}(\cdot + \text{mean}(\tilde{F}))$ where $\tilde{F} = \frac{1}{2}\Phi(-5,\frac{25\pi^{2}}{3}) + \frac{1}{2}\Phi(5,\frac{4\pi^{2}}{3})$. 
	\item[Example 6.] The true $F = \frac{1}{2}\Phi(-2.5,5) + \frac{1}{2}\Phi(2.5,5)$.
\end{description}
\begin{figure}[h!]
	\centering
	\vspace{-1em}
	\subfigure{
		\begin{minipage}[t]{0.32\linewidth}
			\centering
			\includegraphics[width=2in]{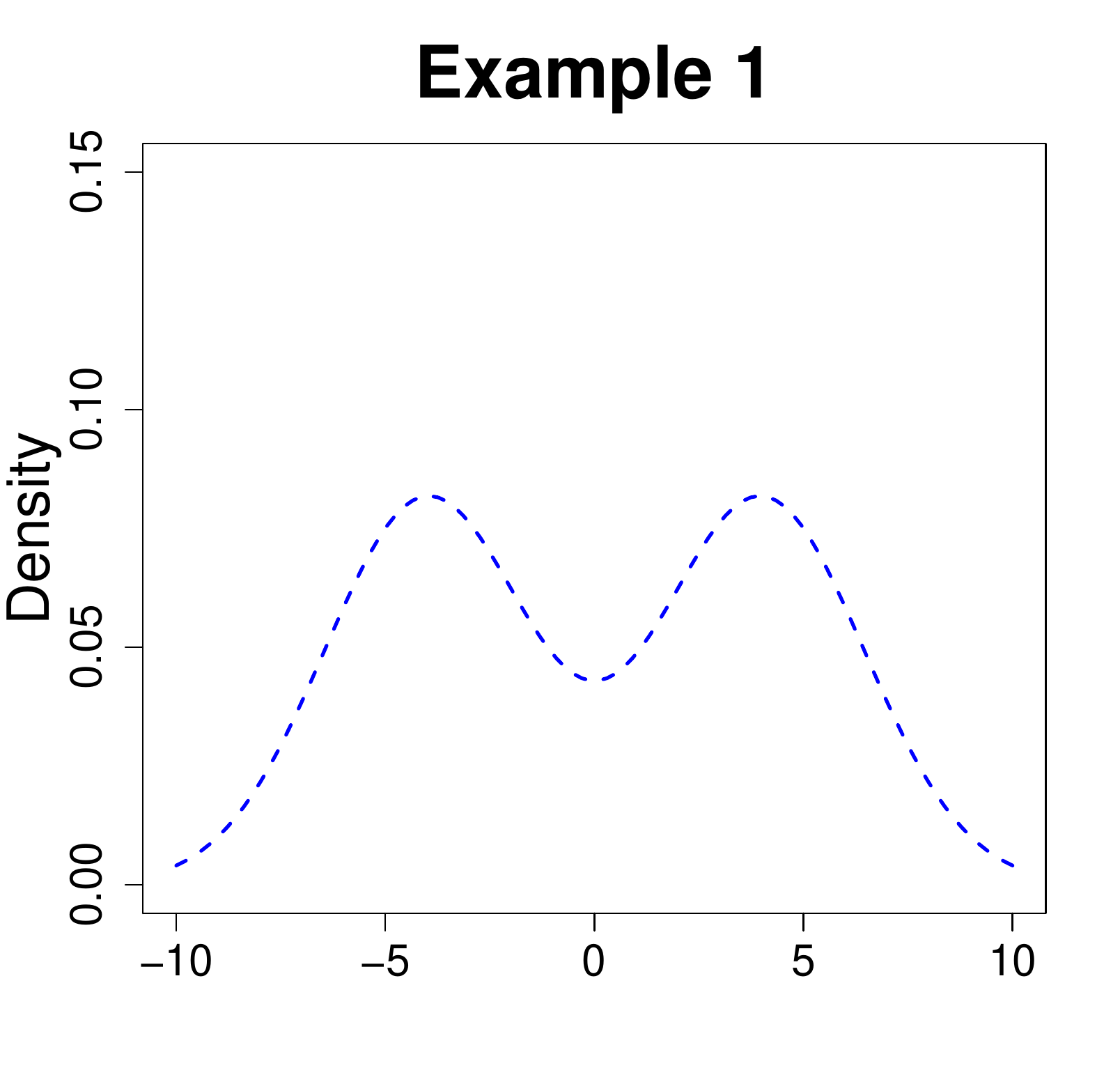}
		\end{minipage}%
	}%
	\subfigure{
		\begin{minipage}[t]{0.32\linewidth}
			\centering
			\includegraphics[width=2in]{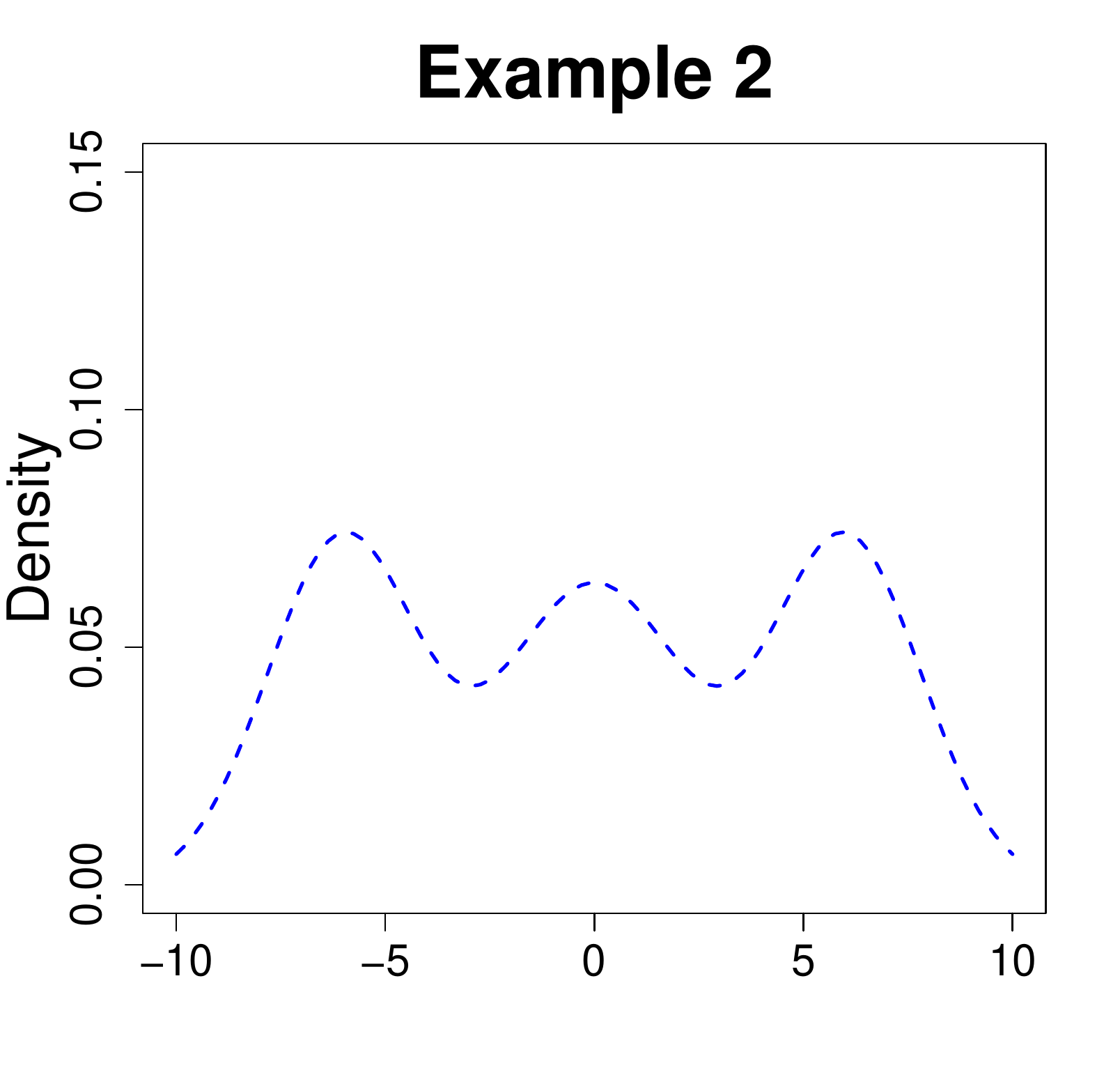}
		\end{minipage}%
	}%
	\subfigure{
		\begin{minipage}[t]{0.32\linewidth}
			\centering
			\includegraphics[width=2in]{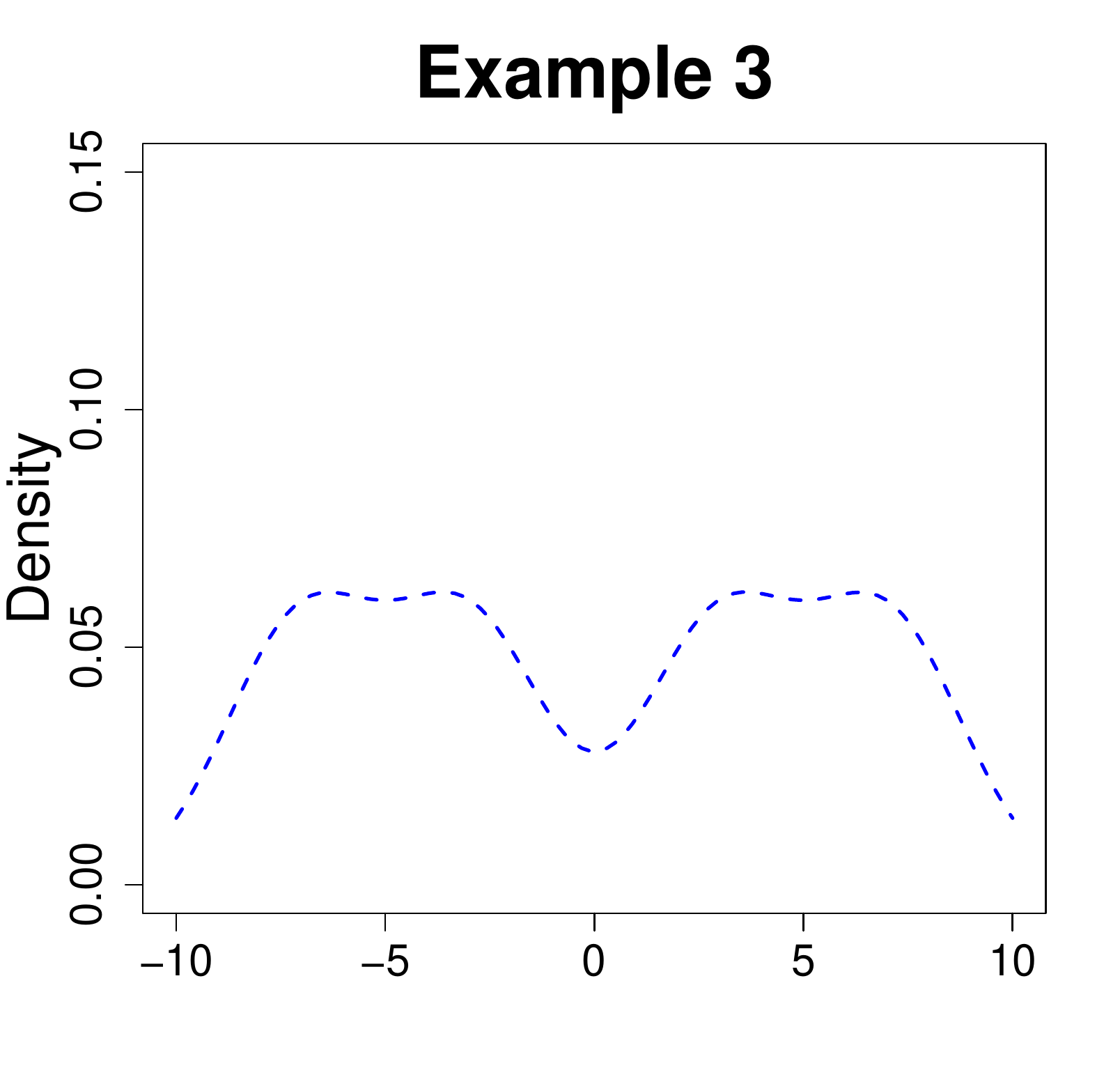}
		\end{minipage}%
	}%
	\centering
	\label{fig:4a}
	\centering
	\subfigure{
		\begin{minipage}[t]{0.32\linewidth}
			\centering
			\includegraphics[width=2in]{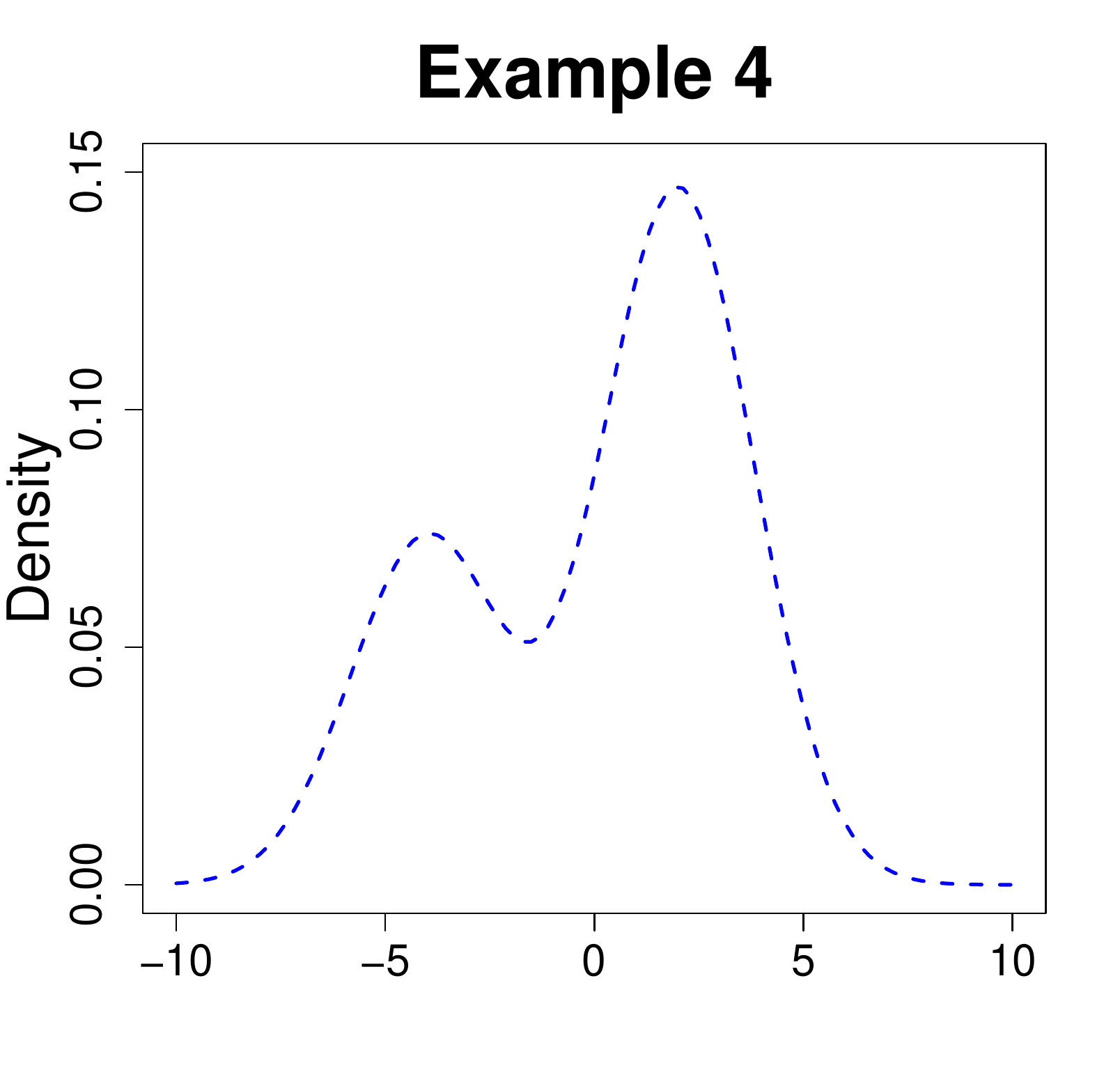}
		\end{minipage}%
	}%
	\subfigure{
		\begin{minipage}[t]{0.32\linewidth}
			\centering
			\includegraphics[width=2in]{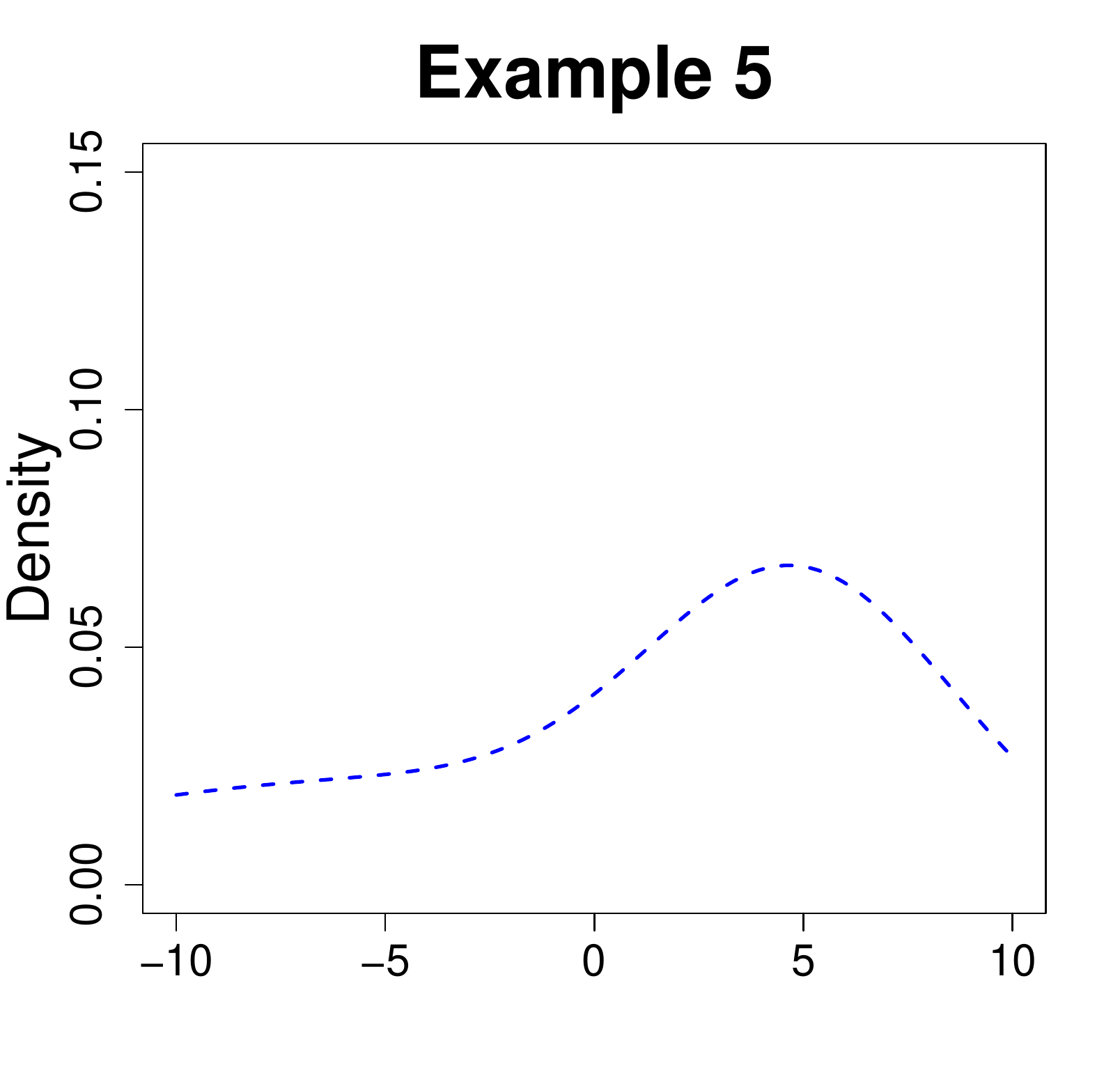}
		\end{minipage}%
	}%
	\subfigure{
		\begin{minipage}[t]{0.32\linewidth}
			\centering
			\includegraphics[width=2in]{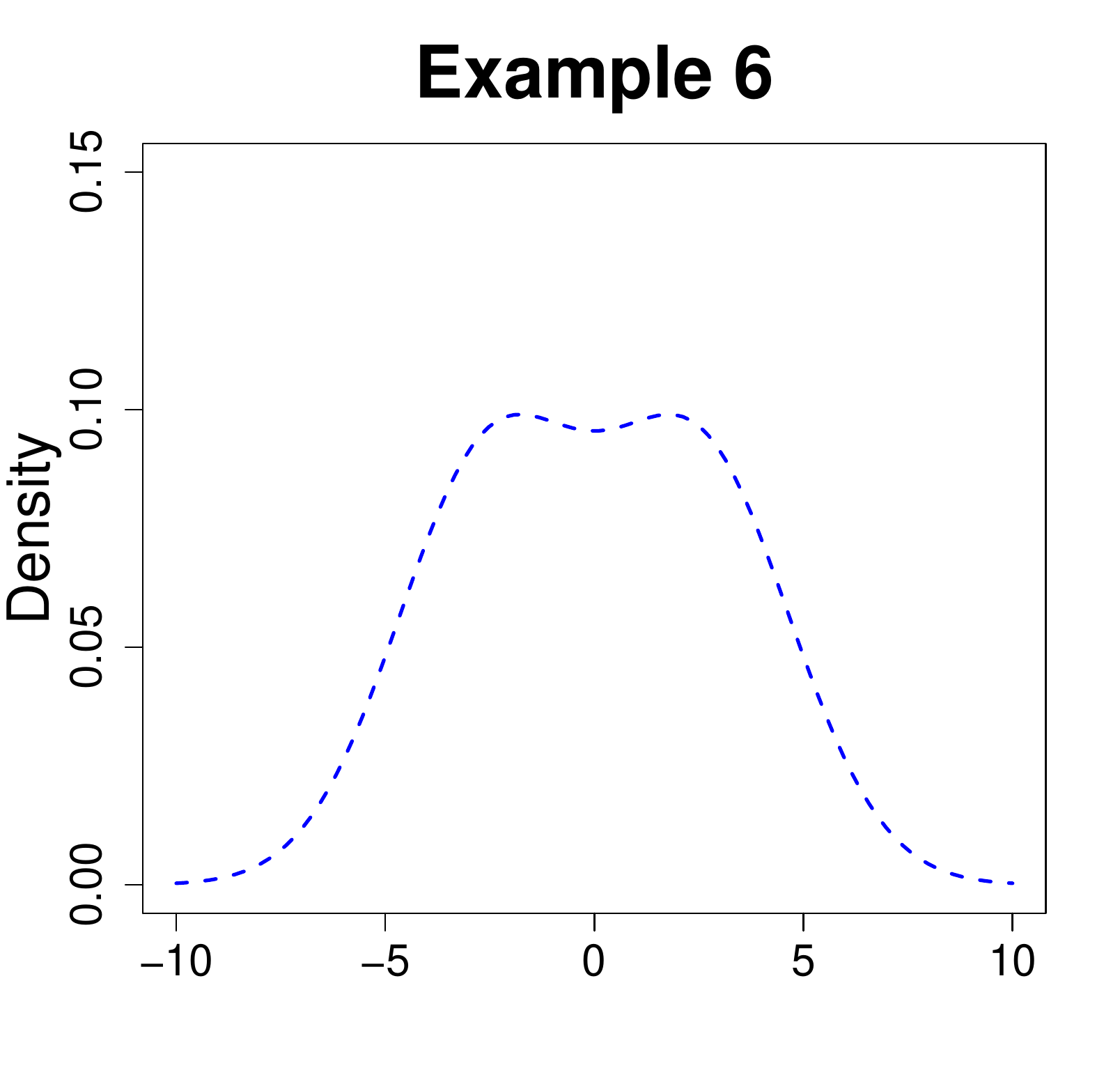}
		\end{minipage}%
	}%
	\centering
	\vspace{-1em}
	\caption{PDFs of the noise distribution in Examples 1 -- 6.}
	\label{fig:4}
\end{figure}
As shown in Figure \ref{fig:4}, Examples 1 -- 3 have symmetric PDFs with two, three and four modes, respectively, while Examples 4 -- 5 have asymmetric PDFs with two modes and one single mode, respectively. Example 6 has two peaks but is close to a single-mode normal distribution. 

We compute the mean and confidence interval of cumulative regrets over 100 replications. As shown in Figure \ref{fig:5}, DIP outperforms both RMLP and RMLP-2 for Examples 1-5. In Example 6, RMLP and RMLP-2 perform better than DIP as the noise distribution $F$ is close to a normal distribution, which aligns with their model assumption. Due to the misspecification of $F$, the cumulative regrets for both RMLP and RMLP-2 exhibit clear linear patterns. The performance deterioration of RMLP and RMLP-2 becomes severe in Example 5, where the PDF is asymmetric and has heavy tails. On the other hand, our DIP policy gradually learns $F$ in the pricing process and achieves sub-linear cumulative regrets in all examples. These examples illustrate the severity of noise distribution misspecification of RMLP and RMLP-2 and hence the superior performance of the proposed robust pricing policy. 
\begin{figure}[h!]
	\centering
	\subfigure{
		\begin{minipage}[t]{0.32\linewidth}
			\centering
			\includegraphics[width=2in]{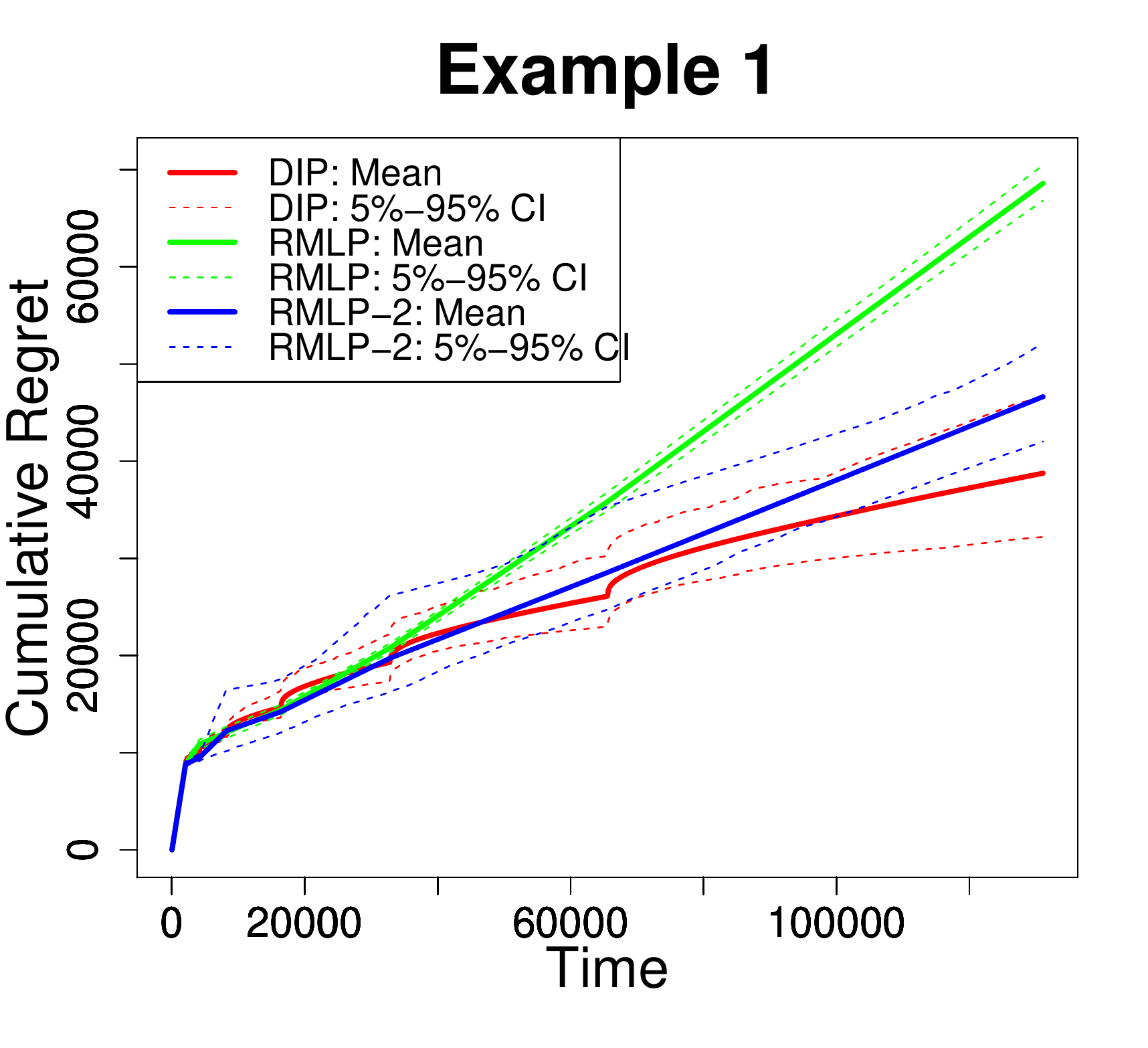}
		\end{minipage}%
	}%
	\subfigure{
		\begin{minipage}[t]{0.32\linewidth}
			\centering
			\includegraphics[width=2in]{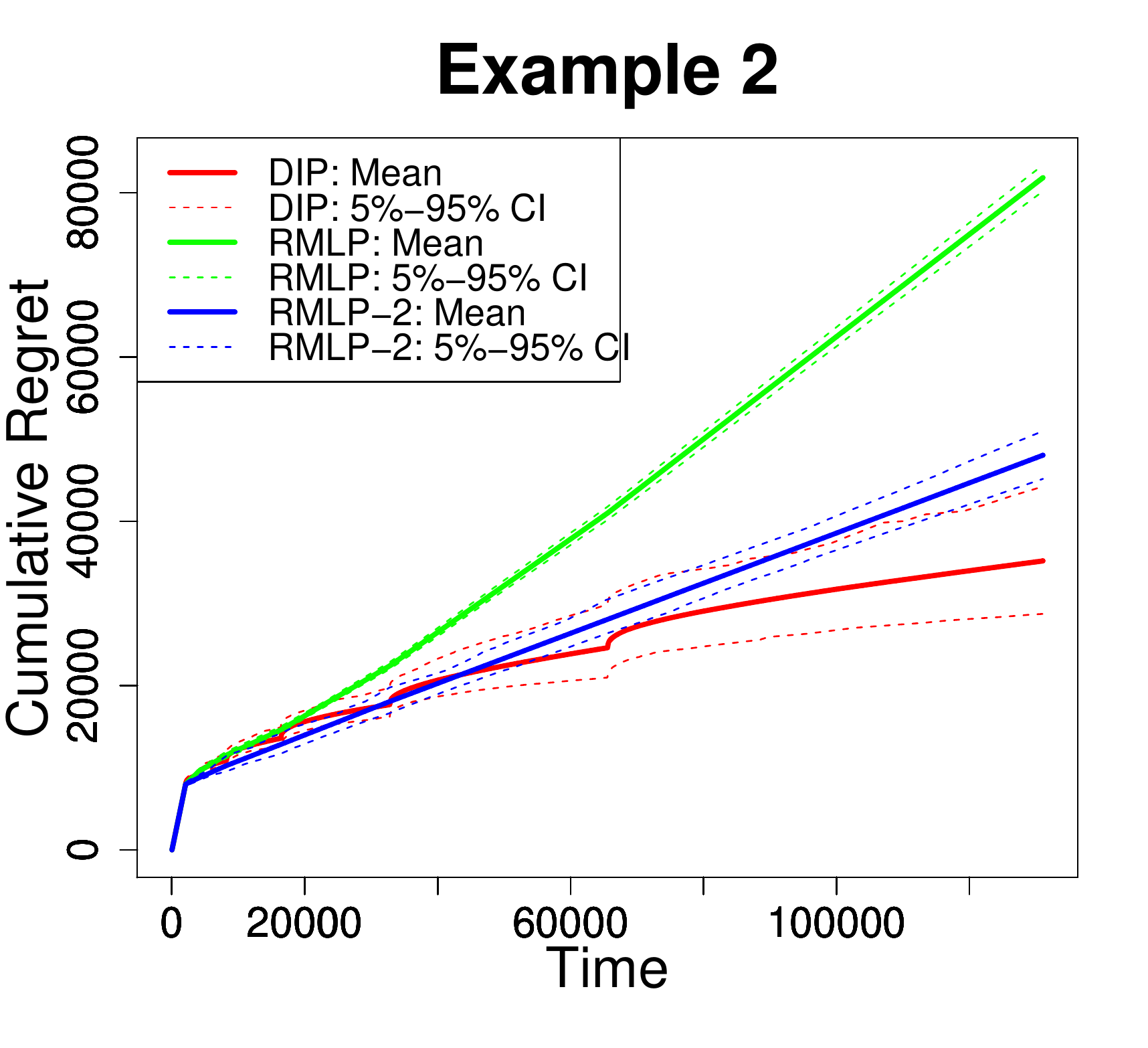}
		\end{minipage}%
	}%
	\subfigure{
		\begin{minipage}[t]{0.32\linewidth}
			\centering
			\includegraphics[width=2in]{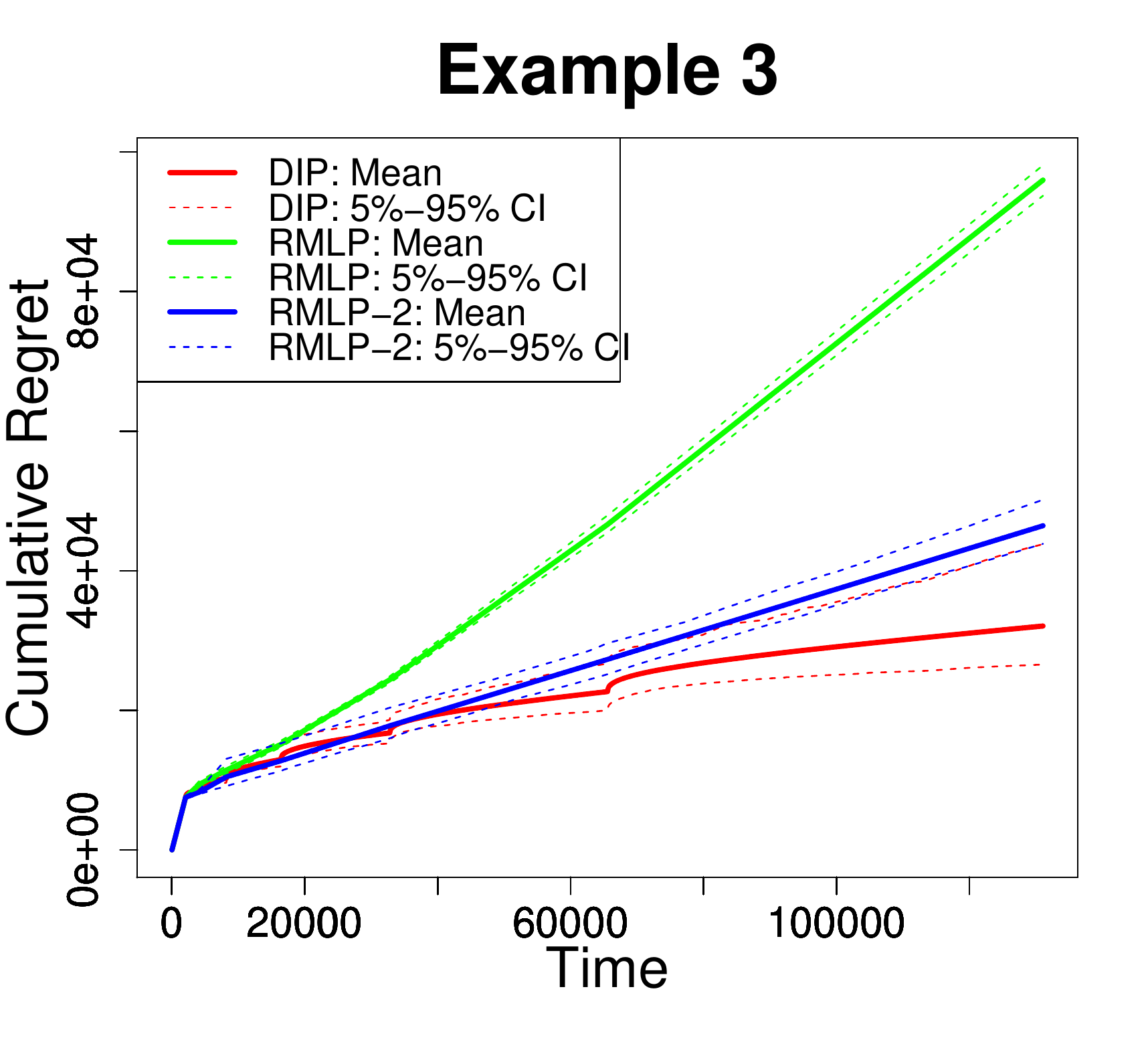}
		\end{minipage}%
	}%
	\centering
	\label{fig:5a}
	\medskip
	\centering
	\subfigure{
		\begin{minipage}[t]{0.32\linewidth}
			\centering
			\includegraphics[width=2in]{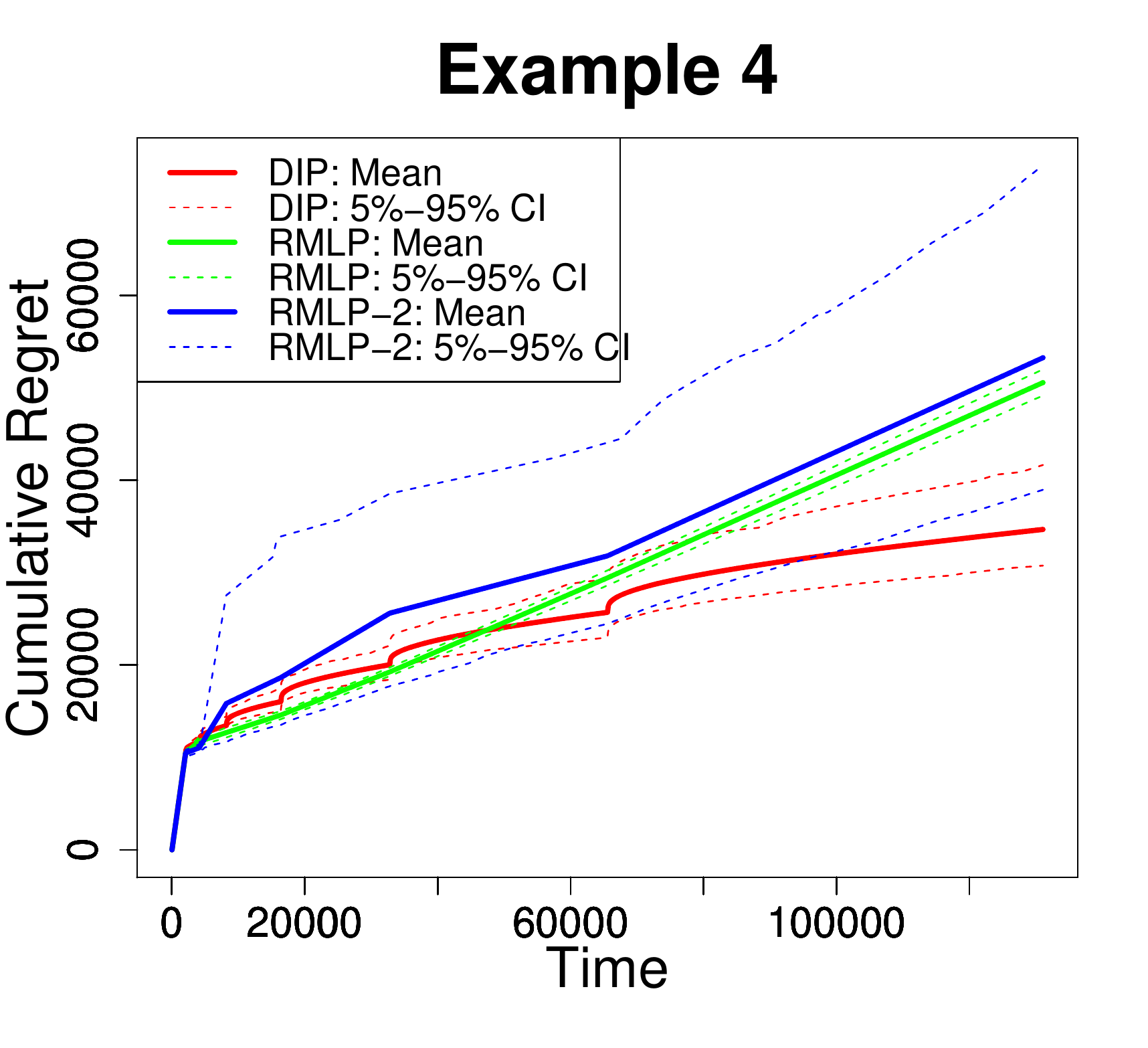}
		\end{minipage}%
	}%
	\subfigure{
		\begin{minipage}[t]{0.32\linewidth}
			\centering
			\includegraphics[width=2in]{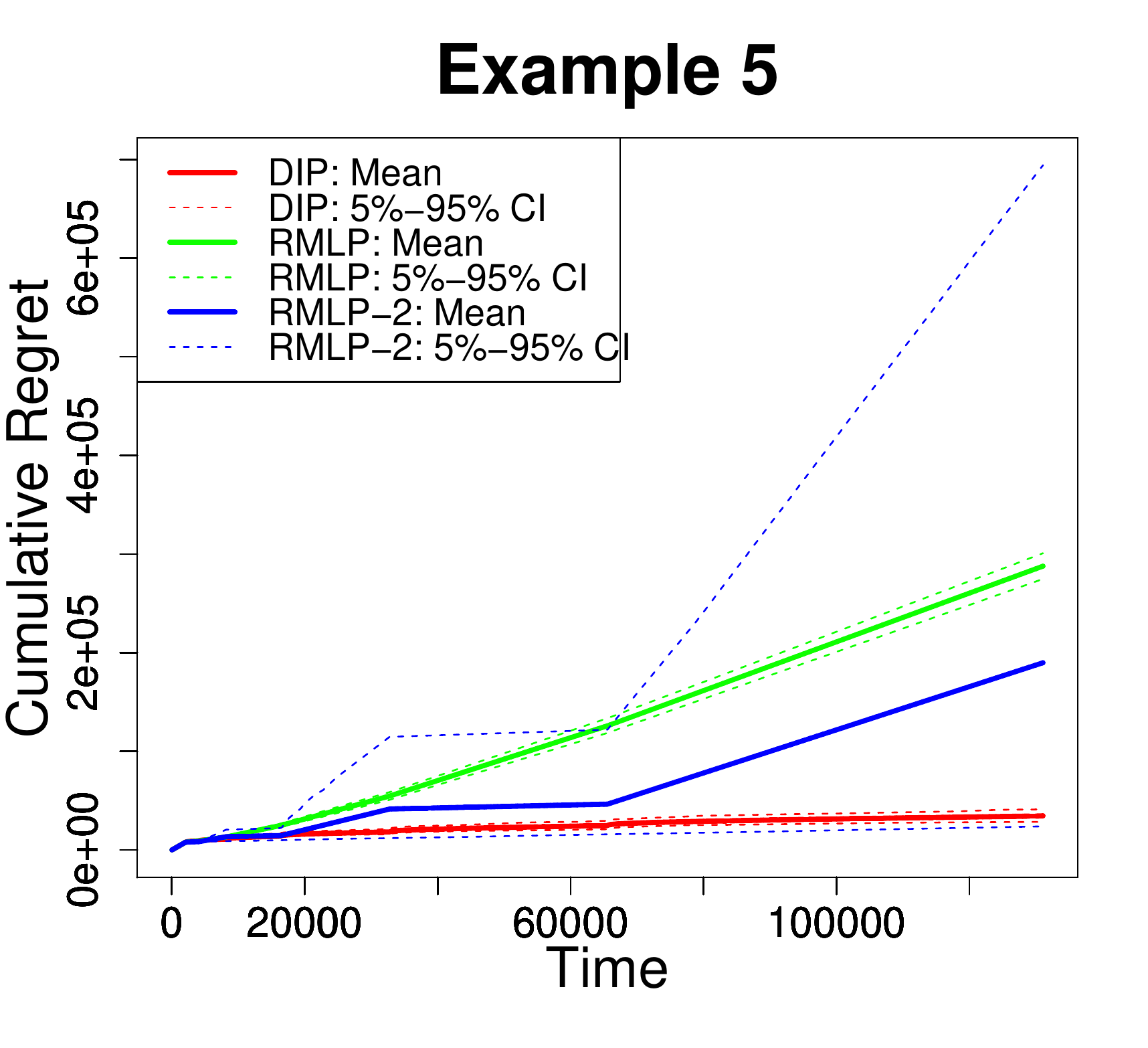}
		\end{minipage}%
	}%
	\subfigure{
		\begin{minipage}[t]{0.32\linewidth}
			\centering
			\includegraphics[width=2in]{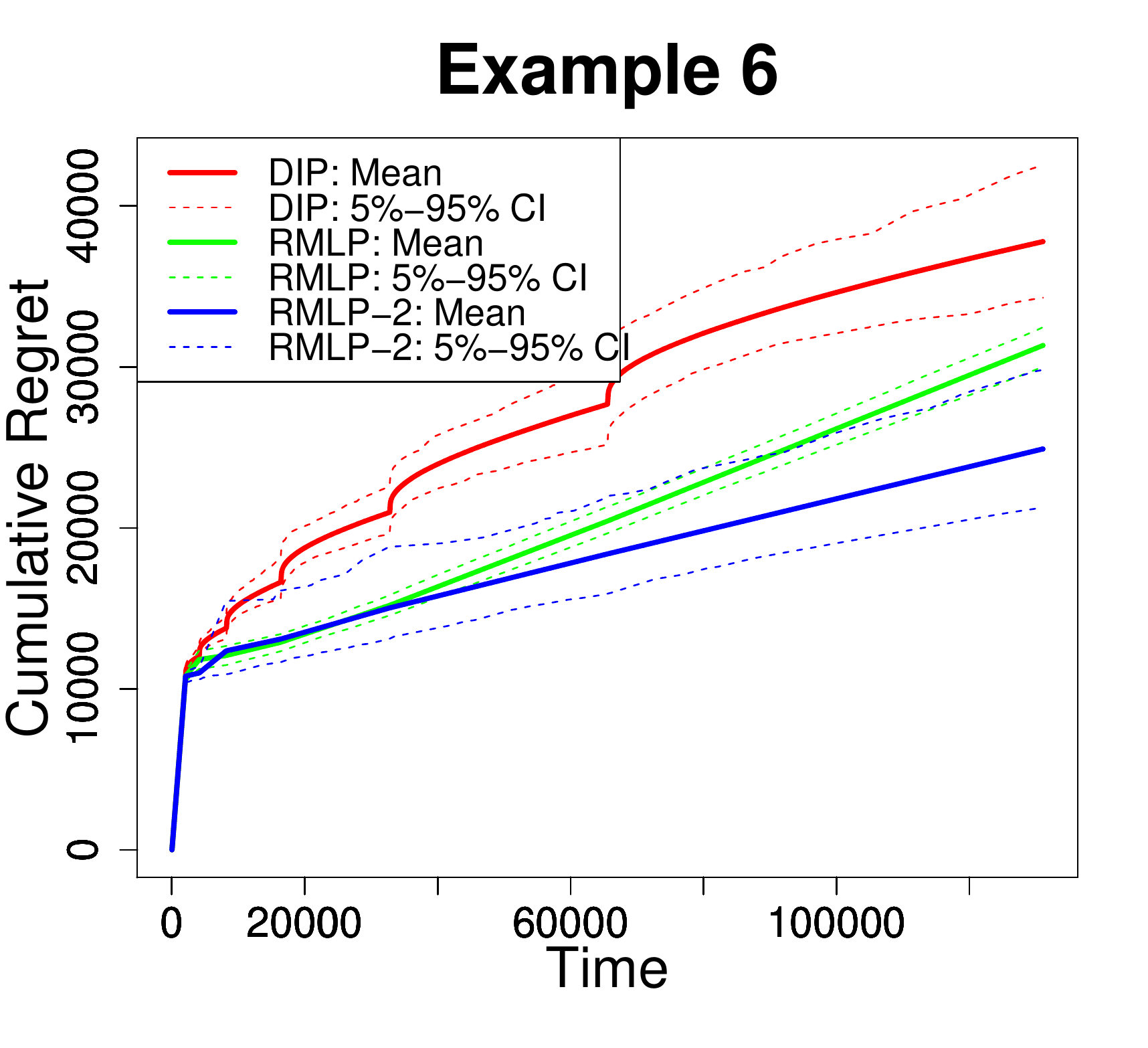}
		\end{minipage}%
	}%
	\centering
	\caption{Regret comparisons of DIP, RMLP and RMLP-2 in Examples 1 -- 6.}
	\label{fig:5}
\end{figure}

Next we show that DIP can still outperform RMLP and RMLP-2 even when the noise distribution $F$ is standard Gaussian $\Phi(0,1)$, which satisfies the log-concave condition assumed by RMLP-2.  In the following Examples 7-9, we set $F= \Phi(0,1)$ and vary the context dimension $d_0$, the true $\theta_{0}$, and the context generation distribution.
\begin{description}
	\item[Example 7.] Dimension $d_0=3$, $\theta_{0} = (10,10,10)^{\top}, x_{t}\overset{\text{i.i.d.}}{\sim} \text{Unif}[0.3,1]^{3}$. 	
	\item[Example 8.] Dimension $d_0=10$, $\theta_{0} = (3,\dots,3)^{\top}, x_{t}\overset{\text{i.i.d.}}{\sim} \text{Unif}[0.1,1]^{10}$. 
	\item[Example 9.] Dimension $d_0=10$, $\theta_{0} = (3,\dots,3)^{\top}, x_{t}\overset{\text{i.i.d.}}{\sim} \text{Unif}[0,1]^{10}$. 
\end{description}

We show the log cumulative regrets of DIP, RMLP and RMLP-2 averaged over 100 replications in Figure \ref{fig:6}. We use log regret because the scale difference between the regrets of these three methods are large for Examples 7-9 mainly due to the unsatisfactory performance of RMLP. For all three methods, we estimate $\theta_{0}$ in each of six episodes and use it for pricing in subsequent episode. Figure \ref{fig:17} shows boxplots of estimation errors $\|\hat{\theta}_k - \theta_0\|_2$ for all six episodes $k=1,\ldots, 6$ and all three methods in Examples 7-9. In Examples 7-8, DIP outperforms RMLP-2 with more stable parameter estimations. In Example 9, the RMLP-2 is relatively stable and delivers better performance than DIP. Note that RMLP performs the worst since it specifies $F$ as $\frac{\exp(x)}{\exp(x)+1}$ with the variance $\frac{\pi^{2}}{3}$, which is quite different from that of the true $F$. Moreover, we find that RMLP-2 sometimes obtains poor estimates and thus incurs large regrets. For instance, as shown in the middle plot of Figure \ref{fig:17} representing Example 8, there is one replication in which RMLP-2 has an estimation error over $80$ in episode 2. Then in this replication, the regret of RMLP-2 in the subsequent episode 3 can be large due to this unsatisfactory estimate. We conjecture that the unstable estimations of RMLP-2 in Examples 7-8 are due to its produced singular price-covariate data in each episode. In Section C of the Supplement, we provide a more detailed discussion on this phenomenon. As a comparison, DIP well balances exploration and exploitation and sets dispersed prices at the beginning of each episode. This helps to generate a healthier data structure leading to more stable estimates.
\begin{figure}[h!]
	\centering
	\subfigure{
		\begin{minipage}[t]{0.32\linewidth}
			\centering
			\includegraphics[width=2in]{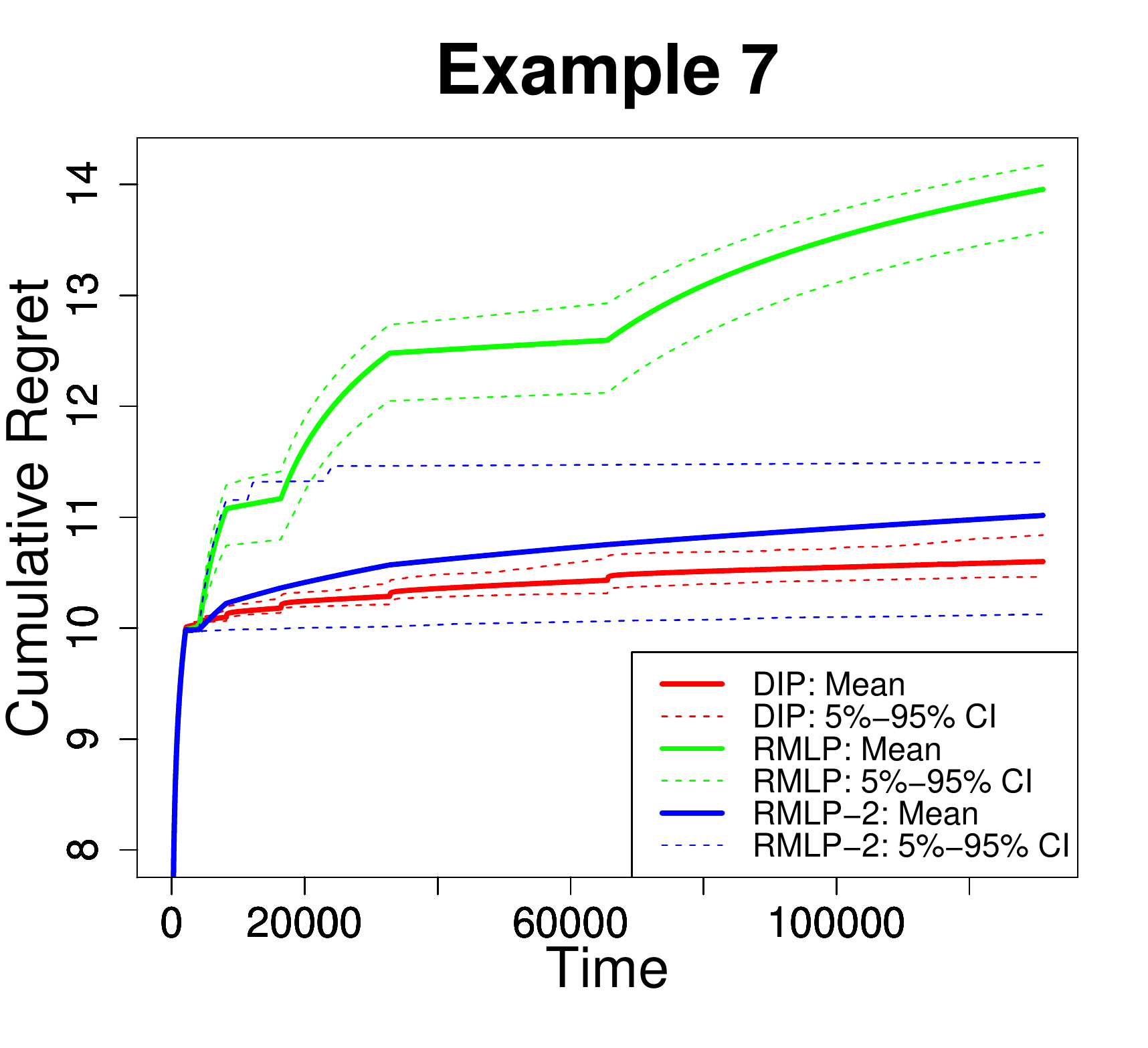}
		\end{minipage}%
	}%
	\subfigure{
		\begin{minipage}[t]{0.32\linewidth}
			\centering
			\includegraphics[width=2in]{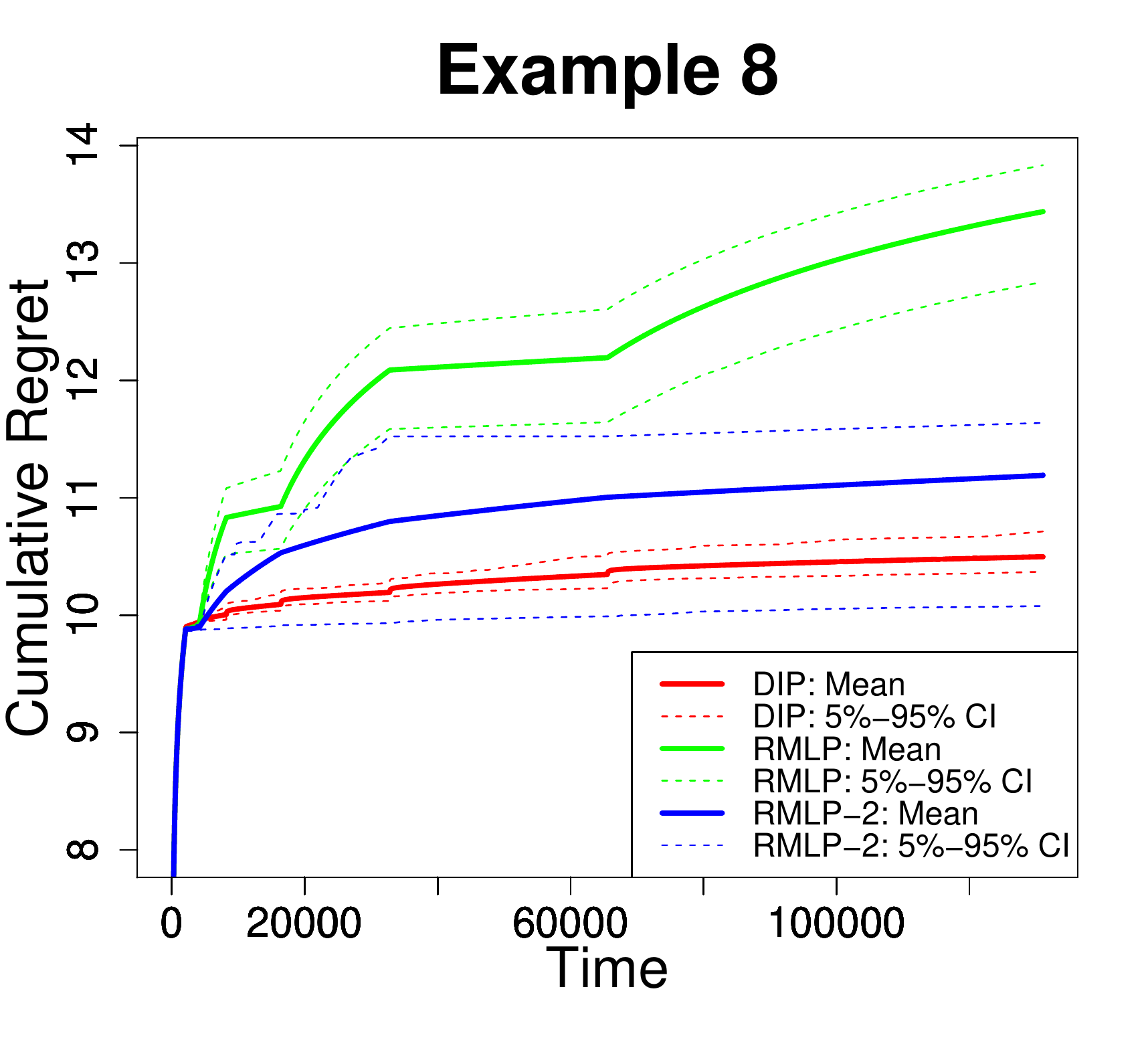}
		\end{minipage}%
	}%
	\subfigure{
		\begin{minipage}[t]{0.32\linewidth}
			\centering
			\includegraphics[width=2in]{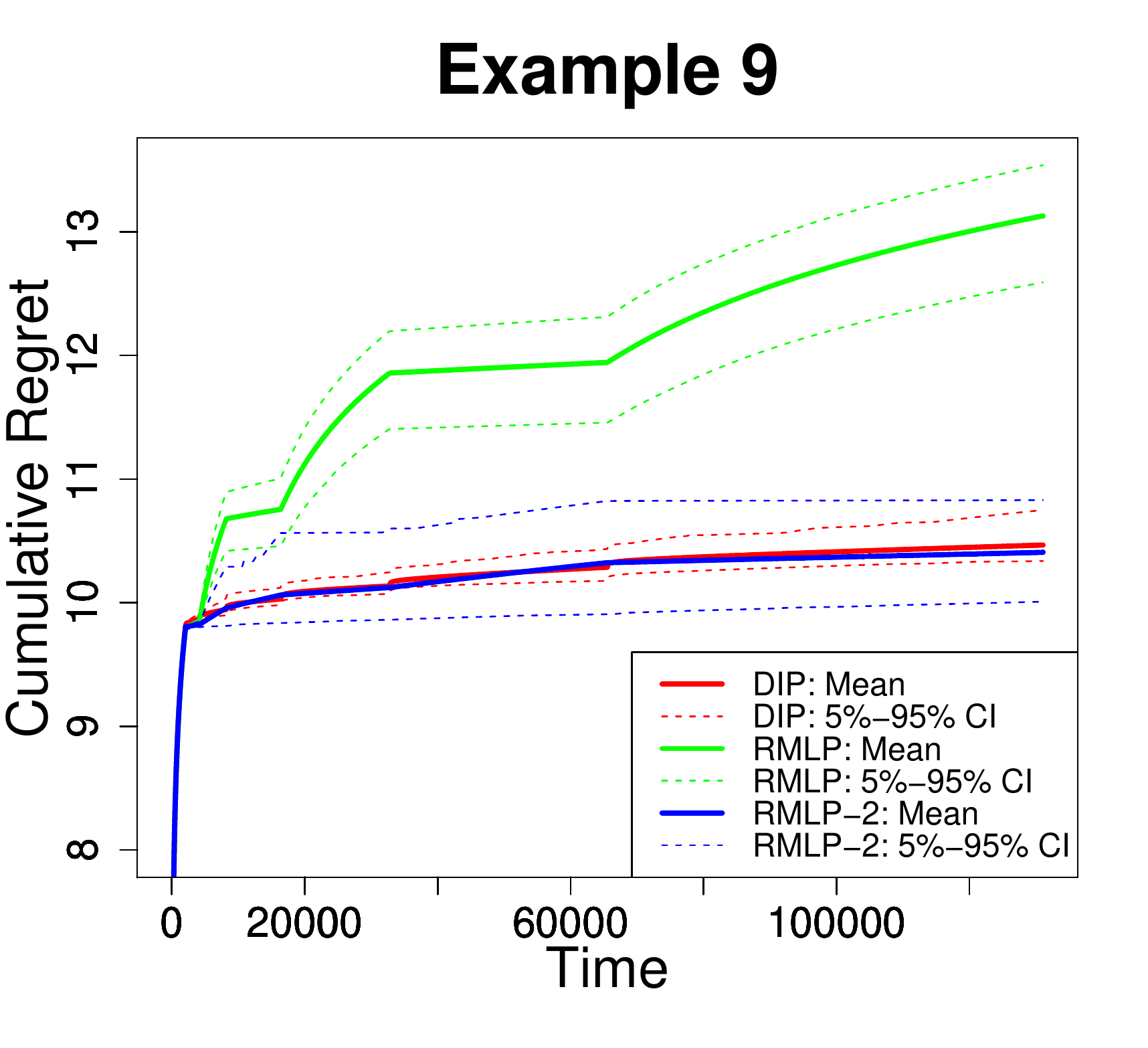}
		\end{minipage}%
	}%
	\centering
	\caption{Log regret comparisons of DIP, RMLP and RMLP-2 in Examples 7 -- 9.}
	\label{fig:6}
\end{figure}

\begin{figure}[h!]
	\centering
	\subfigure{
		\begin{minipage}[t]{0.32\linewidth}
			\centering
			\includegraphics[width=2in]{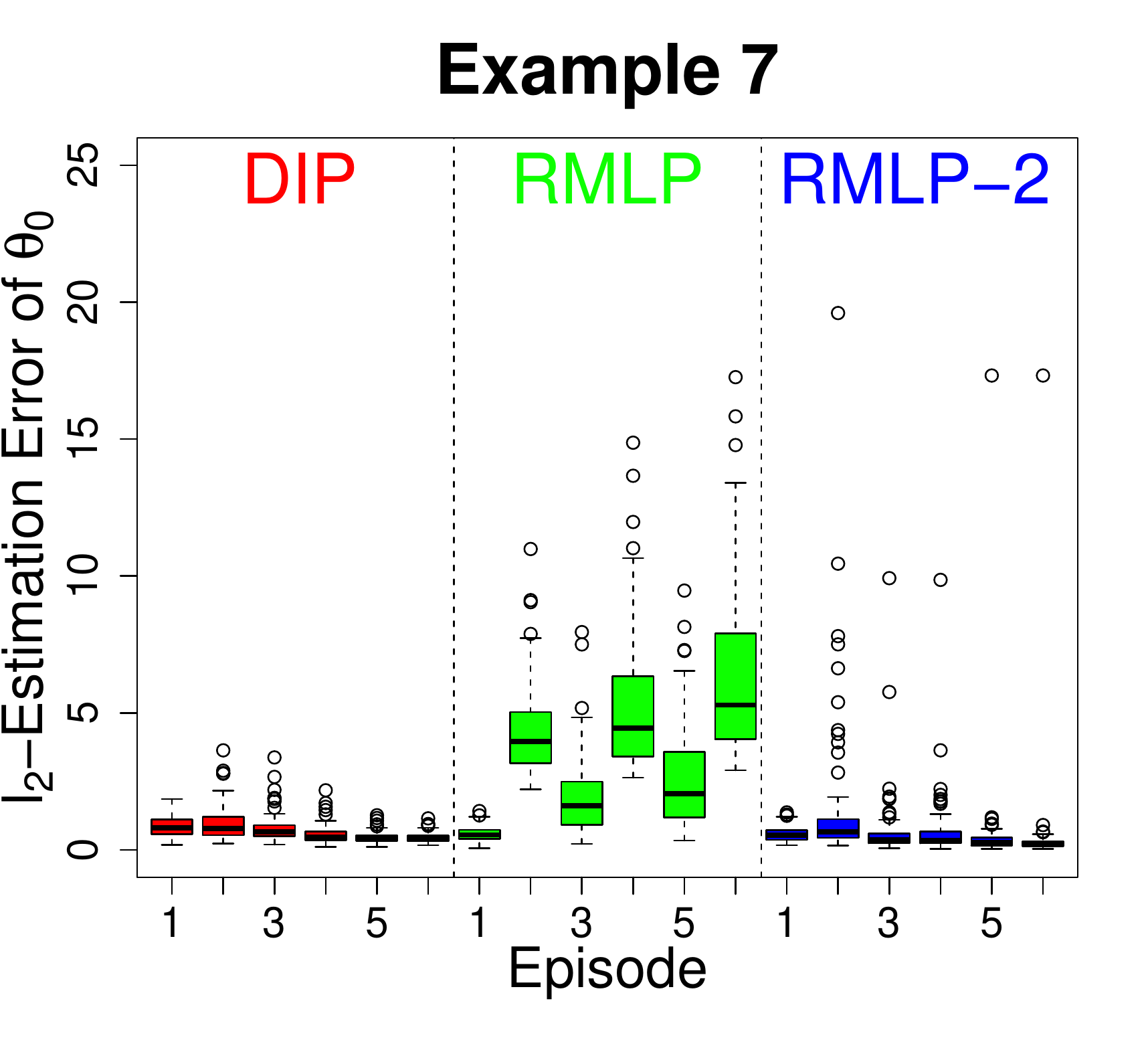}
		\end{minipage}%
	}%
	\subfigure{
		\begin{minipage}[t]{0.32\linewidth}
			\centering
			\includegraphics[width=2in]{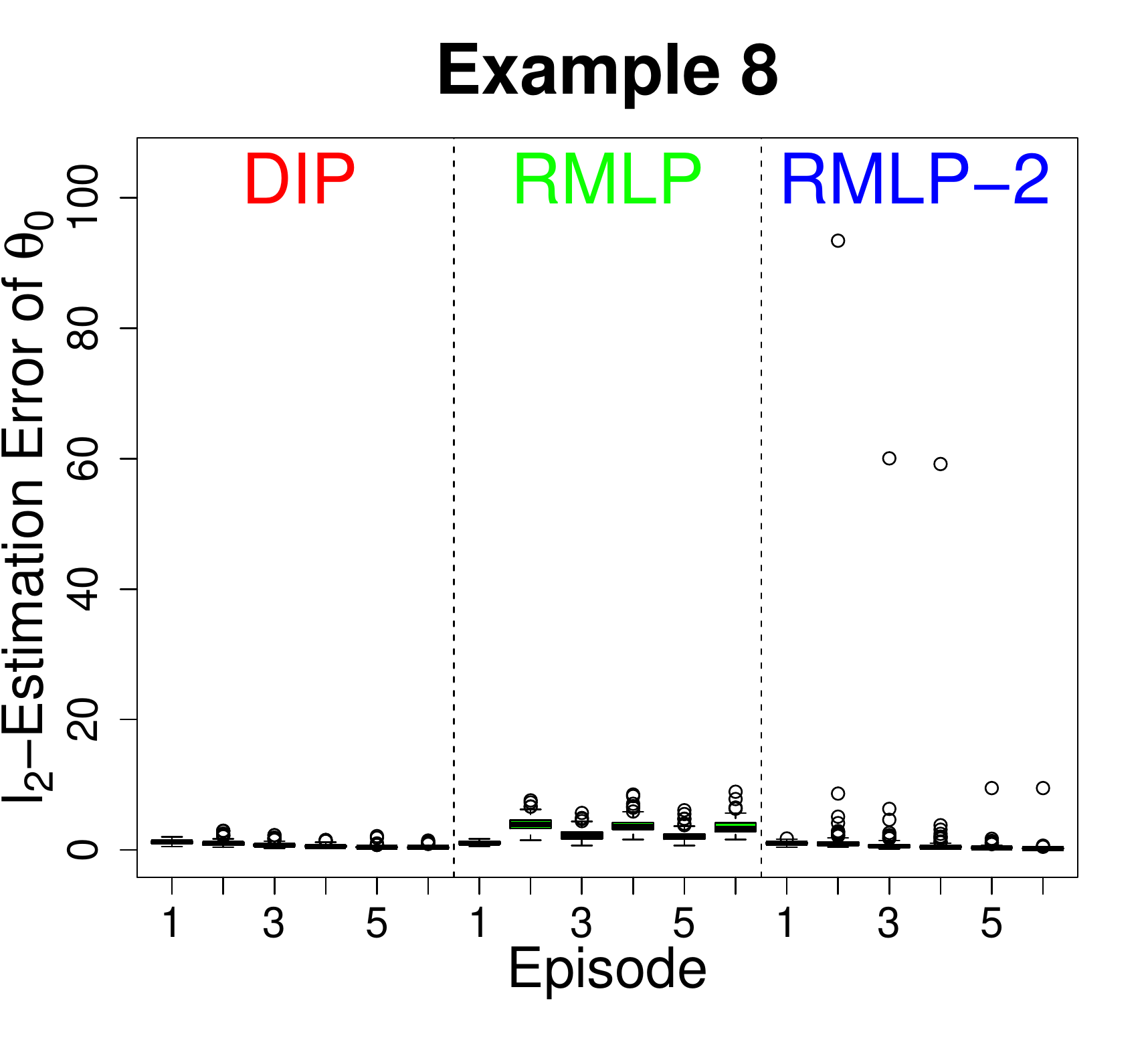}
		\end{minipage}%
	}%
	\subfigure{
		\begin{minipage}[t]{0.32\linewidth}
			\centering
			\includegraphics[width=2in]{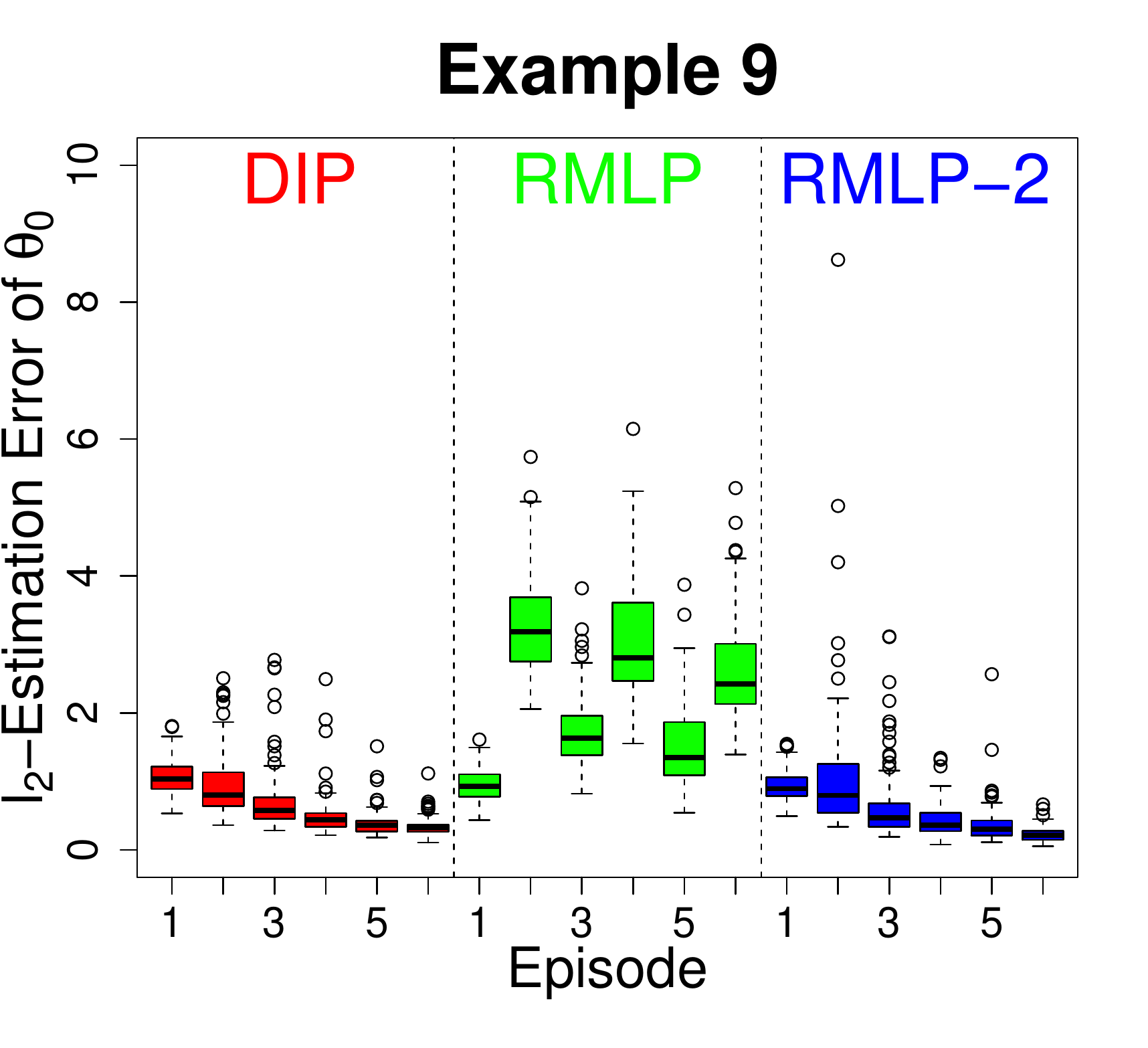}
		\end{minipage}%
	}%
	\centering
	\caption{Estimation errors $\|\hat{\theta}_k - \theta_0\|_2$ of DIP, RMLP and RMLP-2 over six episodes in Examples 7-9.}
	\label{fig:17}
\end{figure}

\subsection{$\ell_{1}$ Estimation Error Convergence}
\label{sec:sim_estimation_error}
Our proved regret upper bound in Theorem \ref{thm:3} involves a term related to $||\hat{\theta}_{k}-\theta_{0}||_{1}$. In our Examples 7 -- 9 of simulations, we have plotted the $\ell_{2}$ estimation errors $||\hat{\theta}_{k}-\theta_{0}||_{2}$. In this subsection, we plot in Figure \ref{fig:converge} the $\ell_{1}$ estimation errors calculated from each episode of Example 7 to investigate its convergence rates. The left panel of Figure \ref{fig:converge} shows a clear decaying trend starting from the second episode. In addition, in the right panel, we plot the $\text{log}_{2}$ average estimation errors over the $\text{log}_{2}$ number of data samples for episodes 2 -- 6. Through a linear fit, we extract a slope of $-0.354$, which implies that the real decaying rate is between $-1/2$ and $-1/3$ and hence $\alpha \in (-1/2,-1/3)$. Thus the $\tilde{O}(T^{2/3})$ overall regret bound in our theorem can be practically achieved.

\begin{figure}[h!]
	\centering
	\subfigure{
		\begin{minipage}[t]{0.4\linewidth}
			\centering
			\includegraphics[width=2.5in]{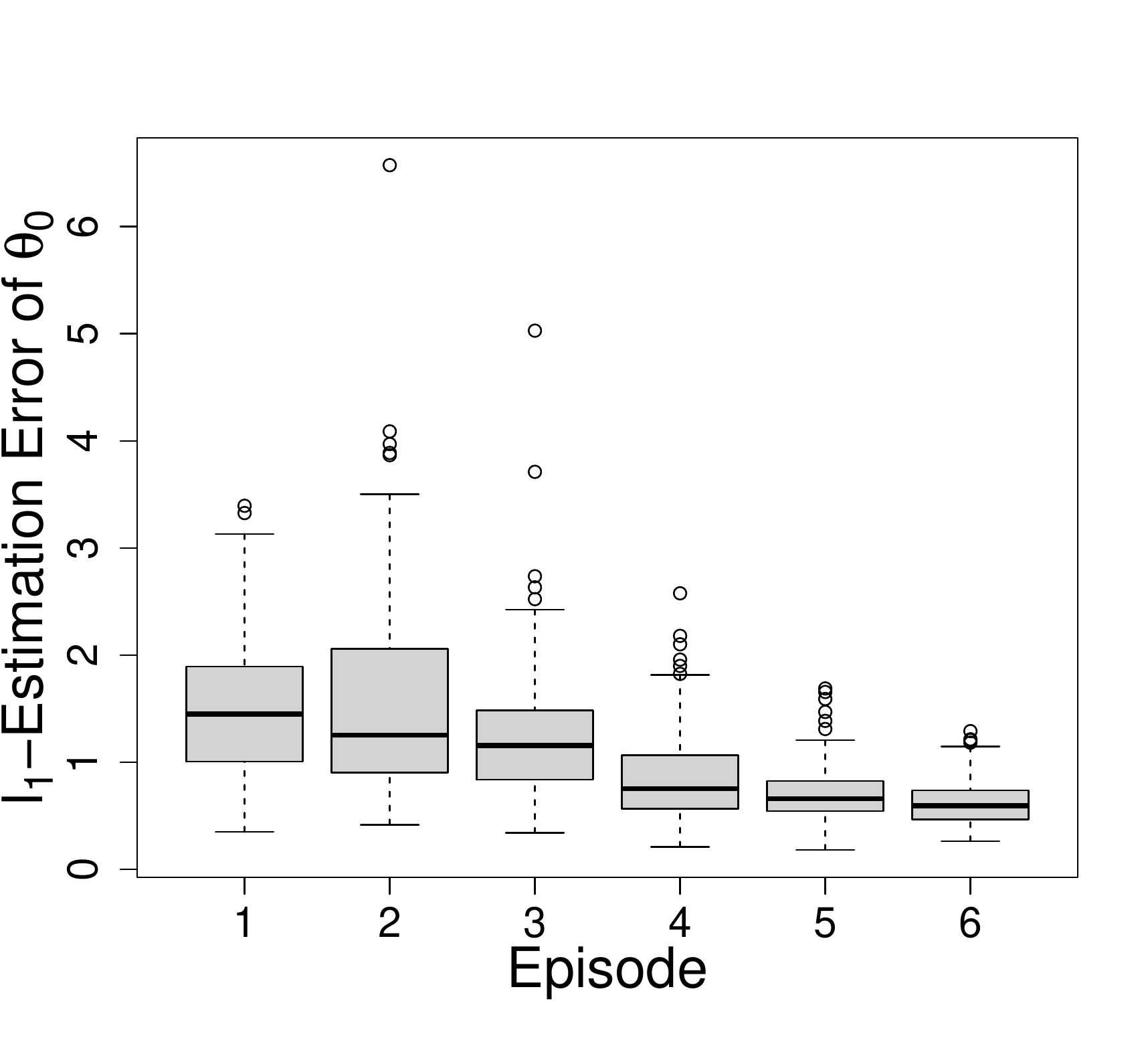}
		\end{minipage}%
	}%
	\subfigure{
		\begin{minipage}[t]{0.4\linewidth}
			\centering
			\includegraphics[width=2.5in]{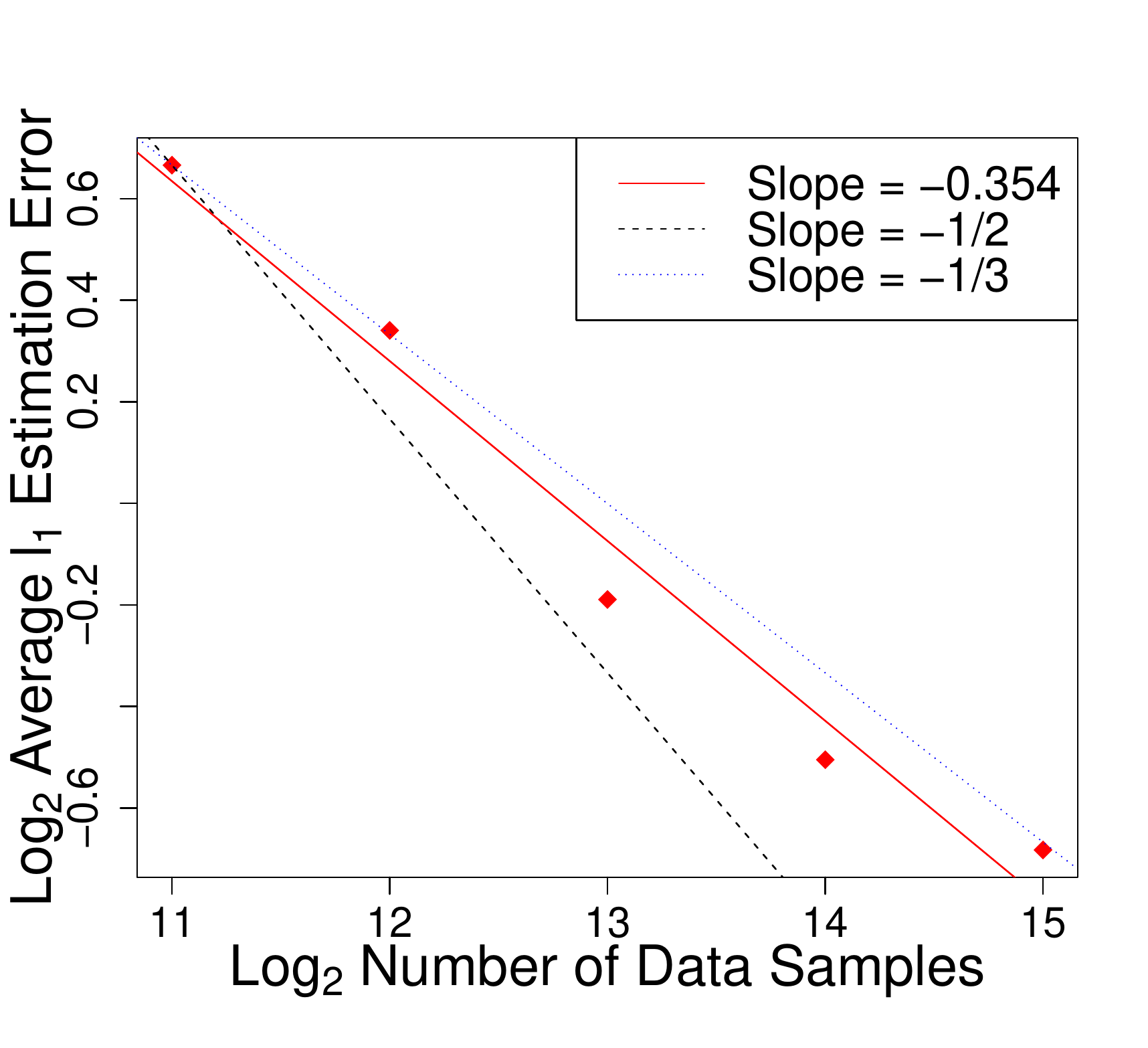}
		\end{minipage}%
	}%
	\centering
	\caption{$\ell_{1}$ estimation error $||\hat{\theta}_{k}-\theta_{0}||_{1}$ of DIP in Example 7.}
	\label{fig:converge}
\end{figure}

\subsection{Heavy-tailed Noise Distributions}
Next we evaluate the performance of our DIP policy on heavy-tailed noise distributions. Let $\text{Cauchy}(\mu,\sigma^{2})$ denote the CDF of the Cauchy distribution with location parameter $\mu$ and scale parameter $\sigma$. We consider the three-dimensional covariates $x_{t}\overset{\text{i.i.d.}}{\sim} \text{Unif}[0.01,1]^{3}$ and set $\theta_{0} = (10,10,10)^{\top}$. The CDF $F$ of the noise distributions are designed as follows. 
\begin{description}
	\item[Example 10.] The true $F = \text{Cauchy}(0,1)$.
	\item[Example 11.] The true $F = \text{Cauchy}(0,3)$.
	\item[Example 12.] The true $F = \frac{1}{2}\text{Cauchy}(-5,6) + \frac{1}{2}\text{Cauchy}(5,6)$.
\end{description}

We repeat 100 times for each example and plot the average accumulative regret curves and their confidence bounds in Figure \ref{fig:10-12reg}. We can see that DIP outperforms both RMLP and RMLP-2 in all these three simulation settings. 
\begin{figure}[h!]
	\centering
	\subfigure{
		\begin{minipage}[t]{0.32\linewidth}
			\centering
			\includegraphics[width=2in]{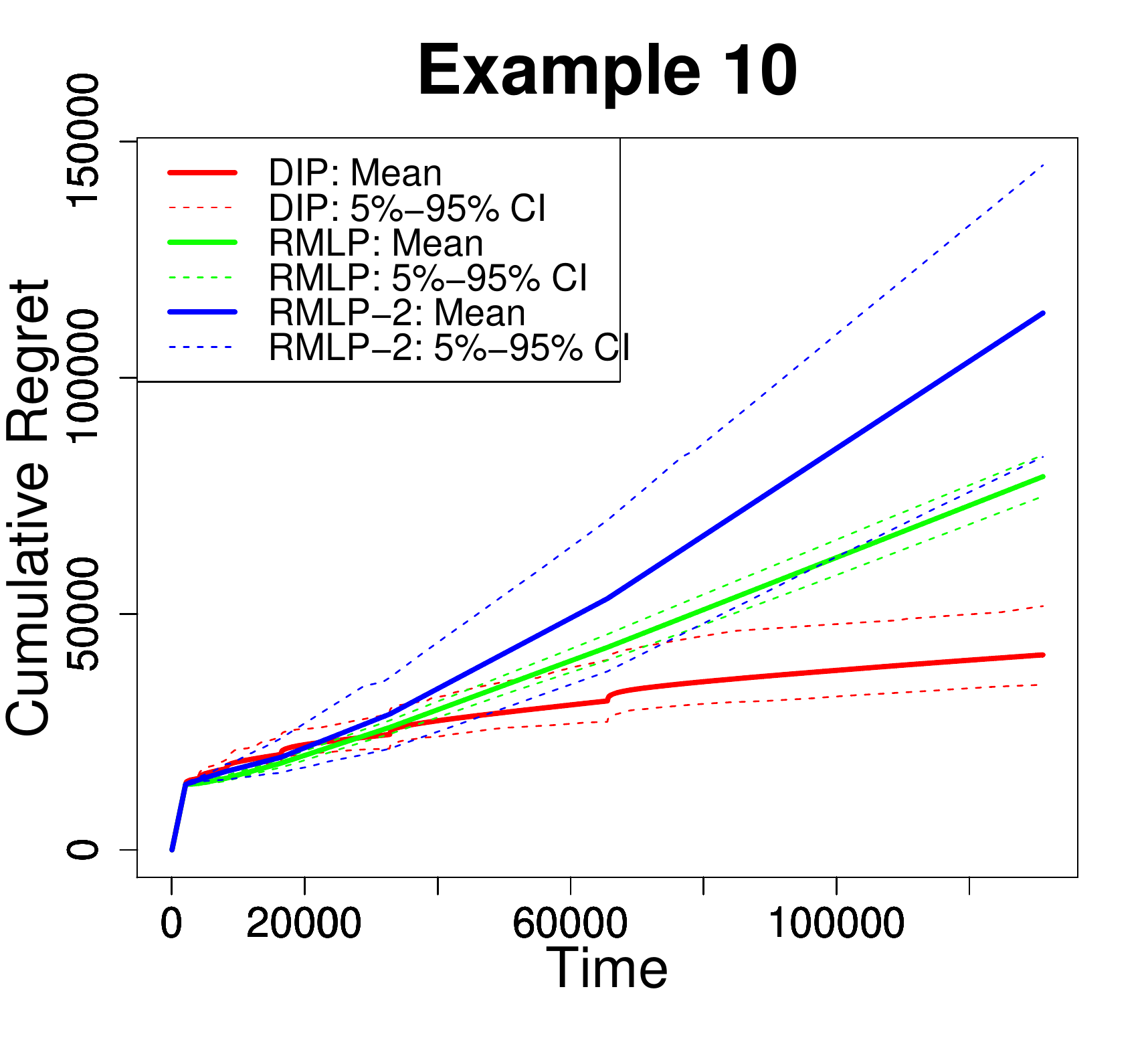}
		\end{minipage}%
	}%
	\subfigure{
		\begin{minipage}[t]{0.32\linewidth}
			\centering
			\includegraphics[width=2in]{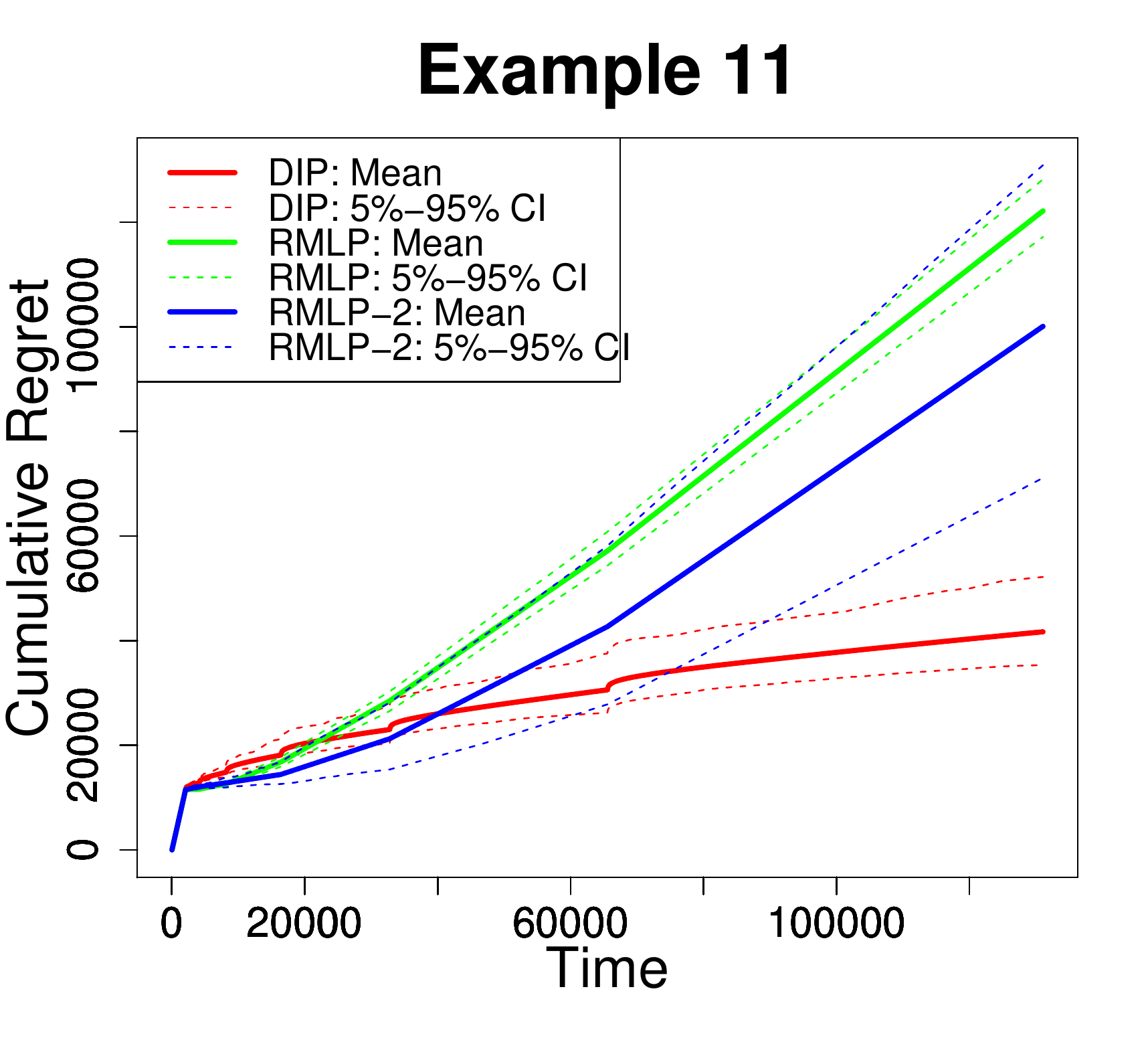}
		\end{minipage}%
	}%
	\subfigure{
		\begin{minipage}[t]{0.32\linewidth}
			\centering
			\includegraphics[width=2in]{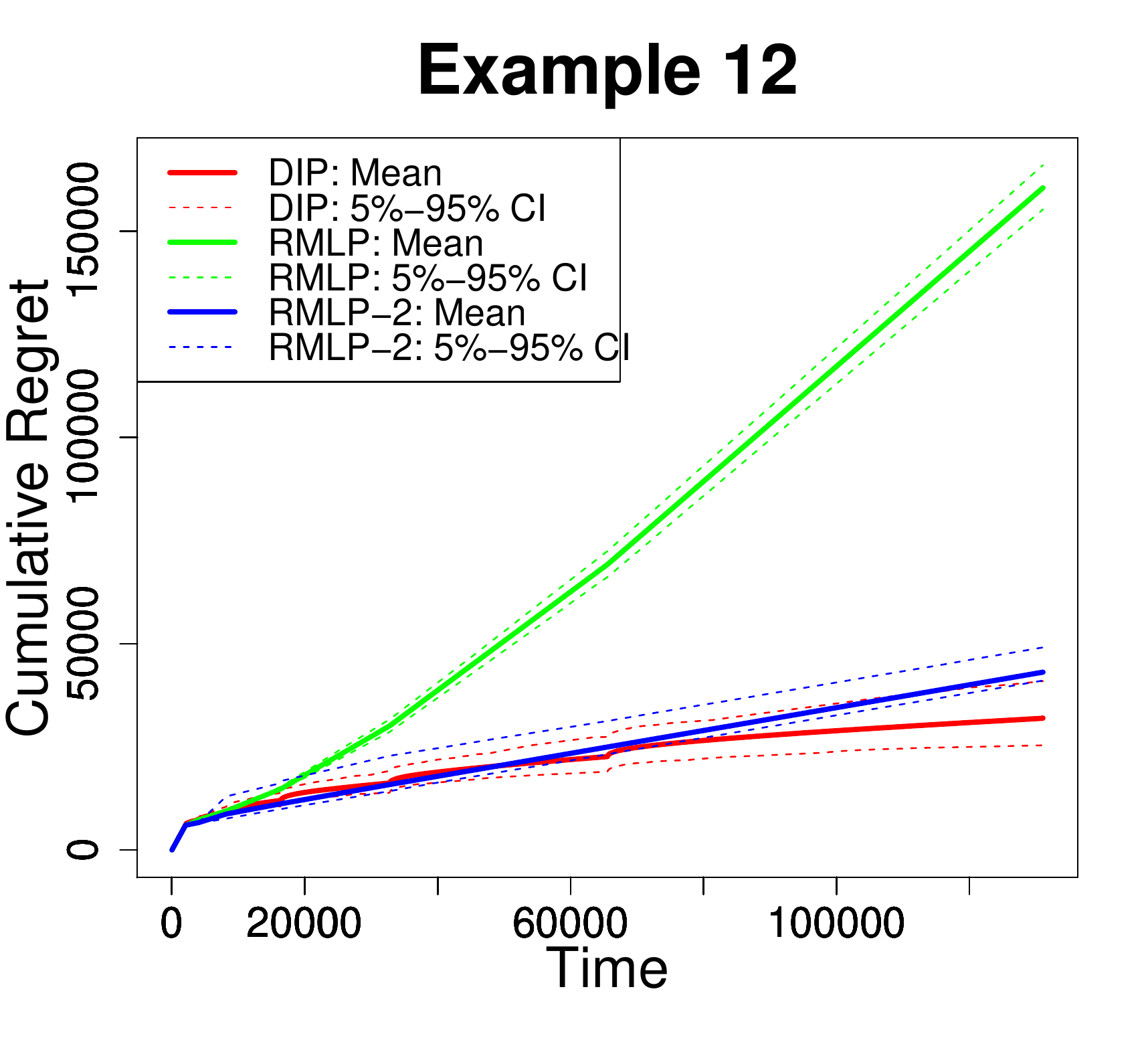}
		\end{minipage}%
	}%
	\centering
	\caption{Regret comparisons of DIP, RMLP and RMLP-2 in Examples 10 -- 12.}
	\label{fig:10-12reg}
\end{figure}

\subsection{Sensitivity Tests}
Our DIP policy relies on three hyper-parameters $\lambda, p_{\max}$ and $C$. Here, $\lambda$ is the regularization parameter of the regression estimation procedure, $p_{\max}$ is an upper bound of any potential optimal prices, and $C$ is the constant in the discretization number $d = C\lceil T_{0}^{1/6}\rceil$. In all our simulation settings, we set $\lambda = 0.1, p_{\max} = 30$ and $C = 20$. In this section, we conduct sensitivity tests to see how their values affect the overall performance of DIP. Here we use the simulation setting in Example 1 as an illustration. In Figure \ref{fig:sensitive}, we include two other values for each of the three parameters, i.e., $\lambda = 0.01,1$, $p_{\max} = 25,35$, $C = 15,25$. The results demonstrate that DIP is relatively robust to the values of these three parameters.

\begin{figure}[h!]
	\centering
	\subfigure{
		\begin{minipage}[t]{0.32\linewidth}
			\centering
			\includegraphics[width=2in]{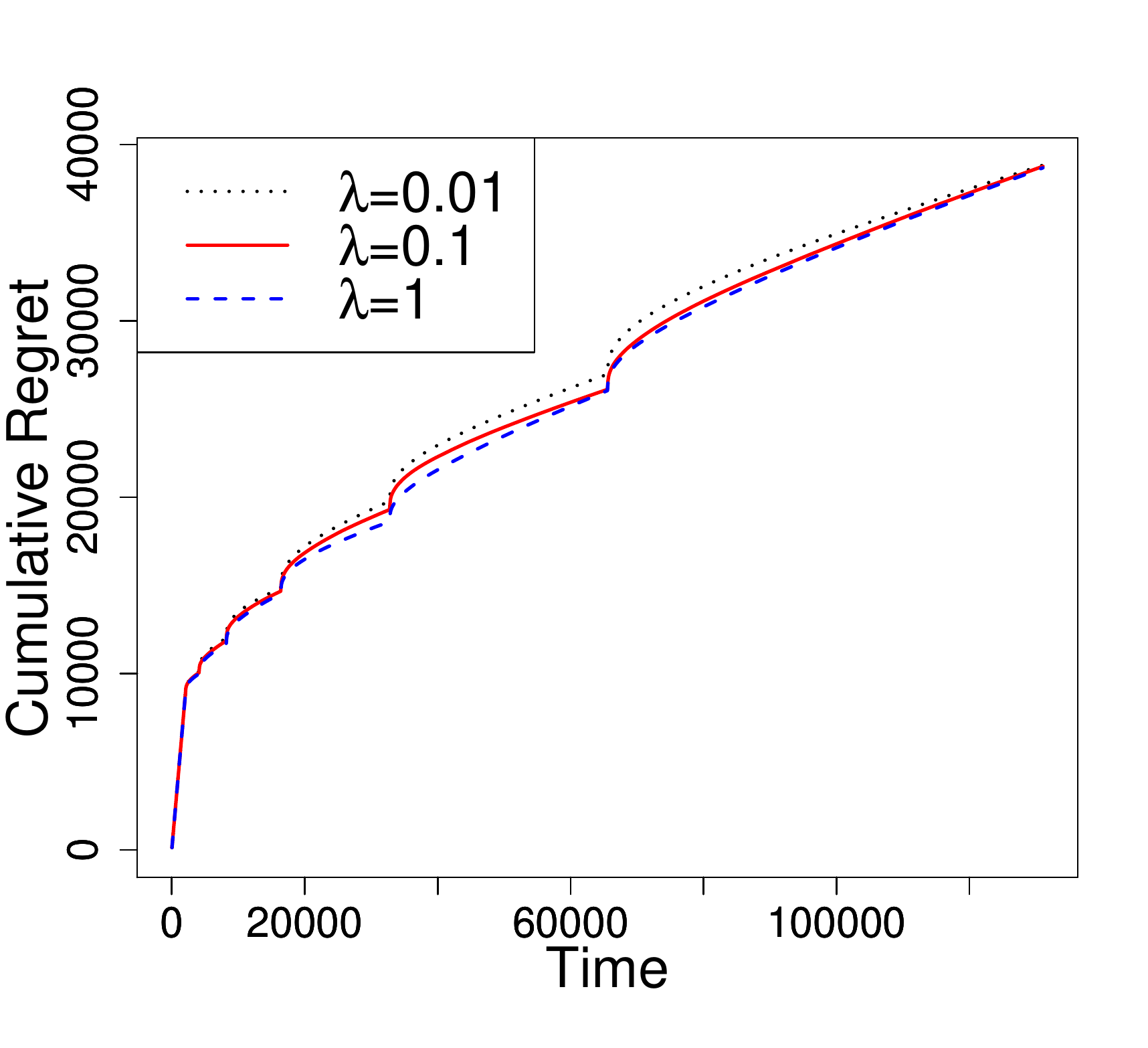}
		\end{minipage}%
	}%
	\subfigure{
		\begin{minipage}[t]{0.32\linewidth}
			\centering
			\includegraphics[width=2in]{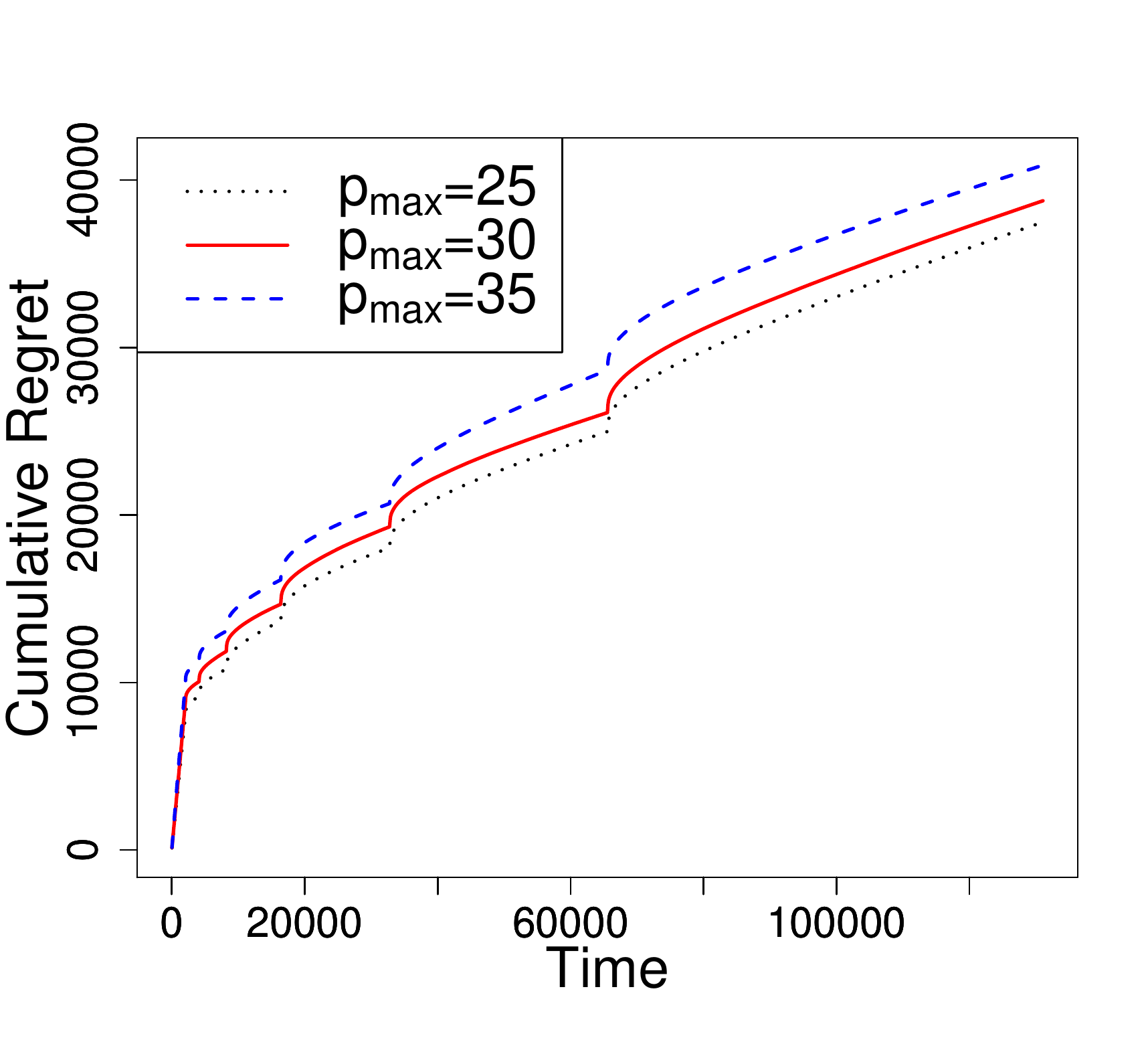}
		\end{minipage}%
	}%
	\subfigure{
		\begin{minipage}[t]{0.32\linewidth}
			\centering
			\includegraphics[width=2in]{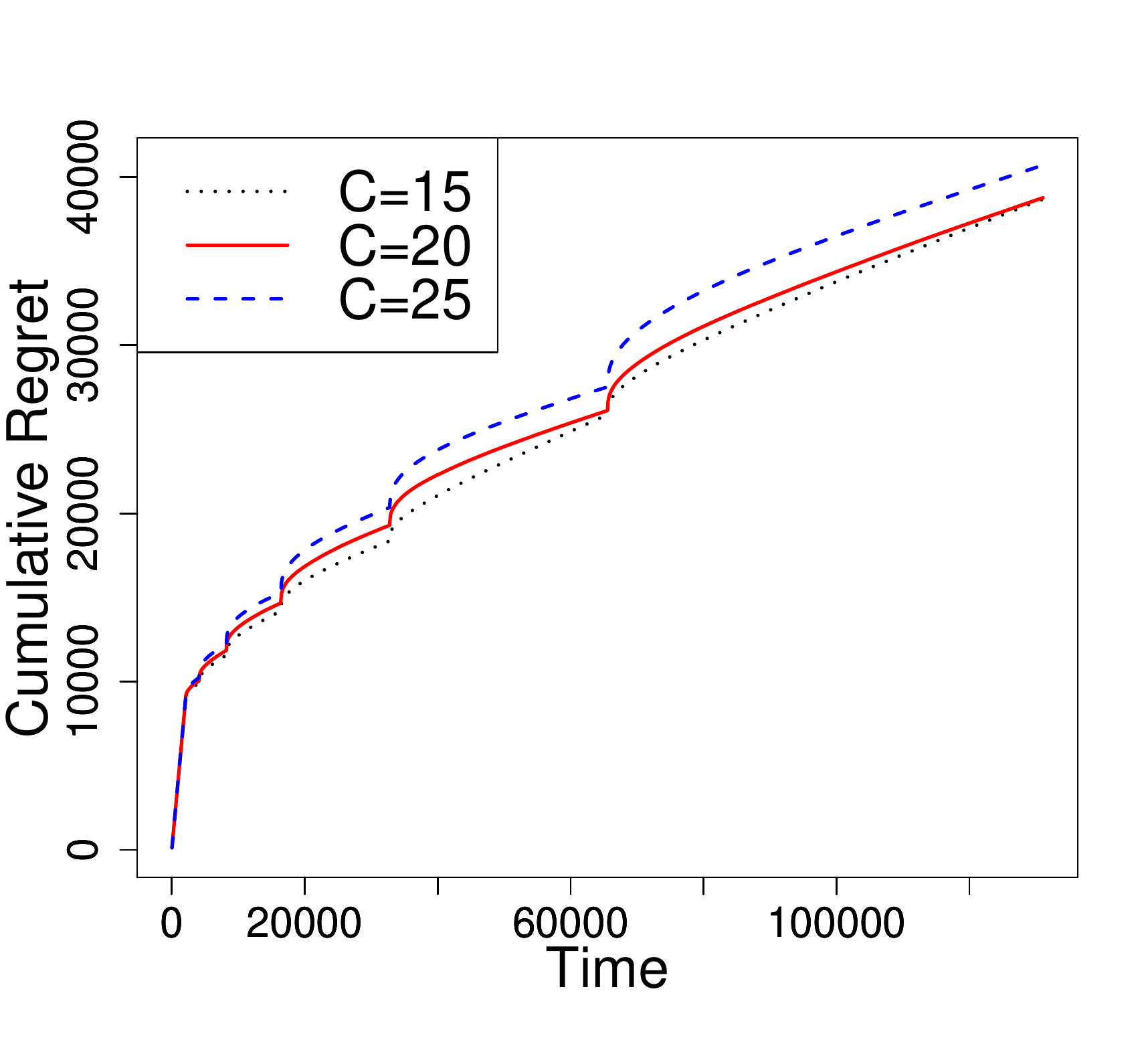}
		\end{minipage}%
	}%
	\centering
	\caption{Sensitivity tests of DIP policy with respect to $\lambda,p_{\max}$ and $C$.}
	\label{fig:sensitive}
\end{figure}

\section{Real Data Analysis}
\label{sec:5}

We explore the efficiency of our proposed DIP policy on a real-life auto loan dataset provided by the Center for Pricing and Revenue Management at Columbia University. This dataset was first studied by \cite{phillips2015effectiveness} and further used by \cite{bastani2019meta} and \cite{ban2020personalized} to evaluate different dynamic pricing algorithms. 

The dataset records 208,085 auto loan applications received by a major online lender in the United States from July 2002 through November 2004. For each application, we observe some loan-specific features such as the date of application, the term and amount of loan requested, and the borrower's personal information. It also includes the monthly payment required by the lender which can be viewed as the pricing decision. Note that it is natural to set prices according to the marketing environment, product features, and customer characteristics in online auto lending. Finally, it records whether or not the price was accepted by the borrower, i.e., the customer's binary purchasing decision in our model. 

We adopt the feature selection result used in \cite{bastani2019meta} and \cite{ban2020personalized} and only consider the following four features: the loan amount approved, FICO score, prime rate, and the competitor's rate. We scale each feature to $[0,1]$ through dividing them by the maximum. The price $p$ of a loan is computed as the net present value of future payment minus the loan amount, i.e., $p = \text{Monthly Payment}\times \sum_{\tau = 1}^{\text{Term}}(1+\text{Rate})^{-\tau}-\text{Loan Amount}$. We use one thousand dollars as a basic unit and $0.12\%$ as the rate value here, an approximate average of the monthly London interbank offered rate for the studied time period. 

Note that it is impossible to obtain customers' real online responses to any dynamic pricing strategy unless it was used in the system while data were collected. Thus we follow the off-policy learning idea used in \cite{bastani2019meta,ban2020personalized} to first estimate the customer choice model using the entire dataset and use it as the grand truth to generate the willingness-to-pay of each customer given any prices. We utilize a two-step estimation procedure to estimate the unknown $\theta_{0}$ and $F$. In particular, we use logistic regression to estimate $\theta_{0}$ and then use the kernel density estimation idea to estimate $F$. The details of this estimation procedure are deferred to Section D of the Supplement. The estimated noise PDF for the US is shown in the left plot of Figure \ref{fig:8}. The estimated $\hat{\theta}_{0}$ and $\hat{F}$ are treated as the true parameters for the customer choice model $y_{t}\sim \text{Ber}(1-\hat{F}(p_{t}-x_{t}^{\top}\hat{\theta}_{0}))$. Note that these true parameters are not used in any dynamic pricing algorithm, but only used to calculate the regret for any set prices and evaluate the performance of any pricing policies. 

\begin{figure}[h!]
	\subfigure{
		\begin{minipage}[t]{0.4\linewidth}
			\centering
			\includegraphics[width=2.5in]{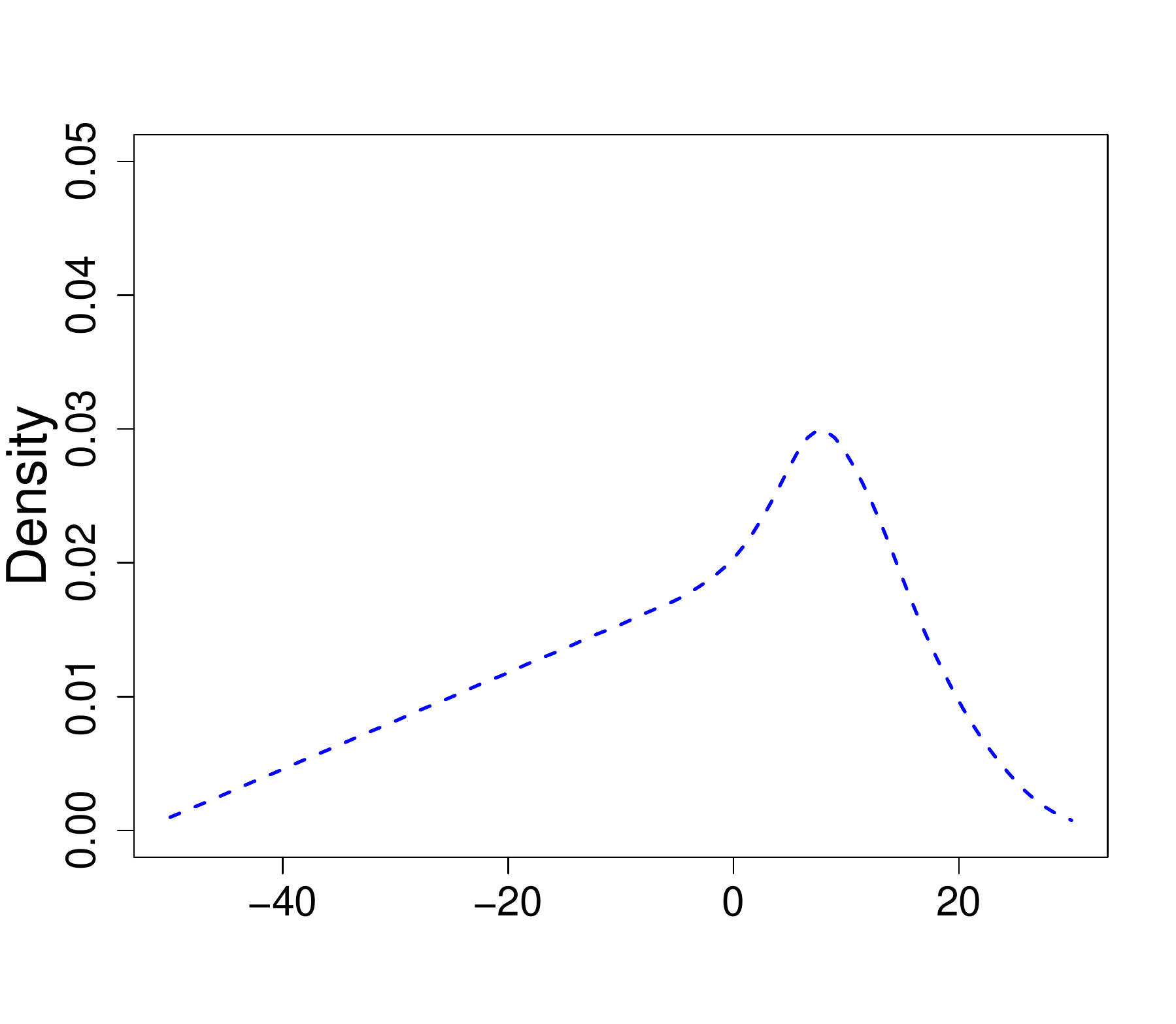}
		\end{minipage}%
	}%
	\subfigure{
		\begin{minipage}[t]{0.4\linewidth}
			\centering
			\includegraphics[width=2.5in]{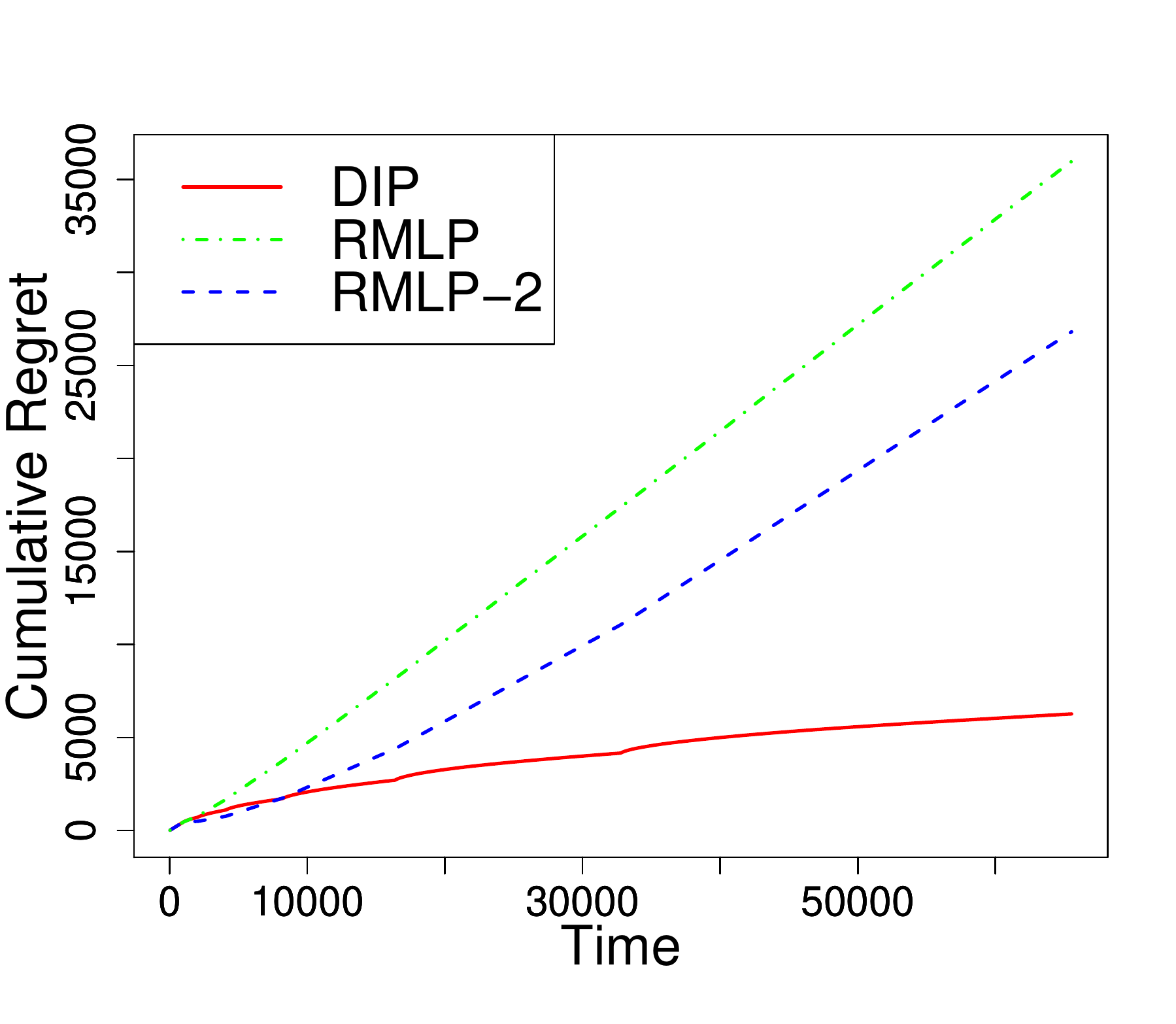}
		\end{minipage}%
	}%
	\centering
	\vspace{-1em}
	\caption{The left plot shows the PDF of the noise distribution for the whole US data and the right plot shows the regret comparison of DIP, RMLP and RMLP-2.}
	\label{fig:8}
\end{figure}

We only compare DIP with RMLP-2 since RMLP-2 is more robust than RMLP as shown in synthetic data in Section \ref{sec:4}. Since the dimension is low and the coefficients are nonsparse, we apply the RMLP-2 policy without regularization. As required by DIP, a known upper bound $p_{\max}$ of the best prices for all applications is set as $30$. We randomly sample $2^{16}$ applications from the total $208085$ for $50$ times and apply DIP and RMLP-2 policy to each of the $50$ replications and then record the average cumulative regrets.

As shown in the right plot of Figure \ref{fig:8}, DIP outperforms RMLP-2 when the time period passes above $10^{4}$. It enjoys more advantages as the time period grows larger. Moreover, DIP shows a clear sub-linear cumulative regret while RMLP-2 displays a linear pattern. This is because DIP can gradually learn the unknown distribution $F$. Furthermore, DIP enjoys a more accurate and stable $\theta_{0}$ estimation since it invests a certain amount in price explorations and generates a more well-distributed dataset. The RMLP-2 sets the prices by applying a deterministic mapping function to a linear combination of the covariates, which might yield a singular data structure leading to unsatisfactory estimates. This phenomenon is similar to that shown in Examples 7-8 of the synthetic experiments.

Next we evaluate the performance of DIP and RMLP-2 by focusing on data in California which has nearly 30000 applications. We apply the same estimation procedure for $\theta_0$ and $F$ on the California dataset to obtain the true customer choice model for California. The estimated PDF of the noise distribution for the California data is shown in the left panel of Figure \ref{fig:9}. It has a multimodal pattern and does not satisfy the log-concave condition required by RMLP-2. This illustrates our motivation that the noise distribution could be complex in real applications. We record the average cumulative regrets for 50 random samplings of $2^{14}$ applications. As shown in the right panel of Figure \ref{fig:9}, DIP again achieves a sub-linear regret, which outperforms that of RMLP-2 eventually. 
\begin{figure}[h!]
	\subfigure{
		\begin{minipage}[t]{0.4\linewidth}
			\centering
			\includegraphics[width=2.5in]{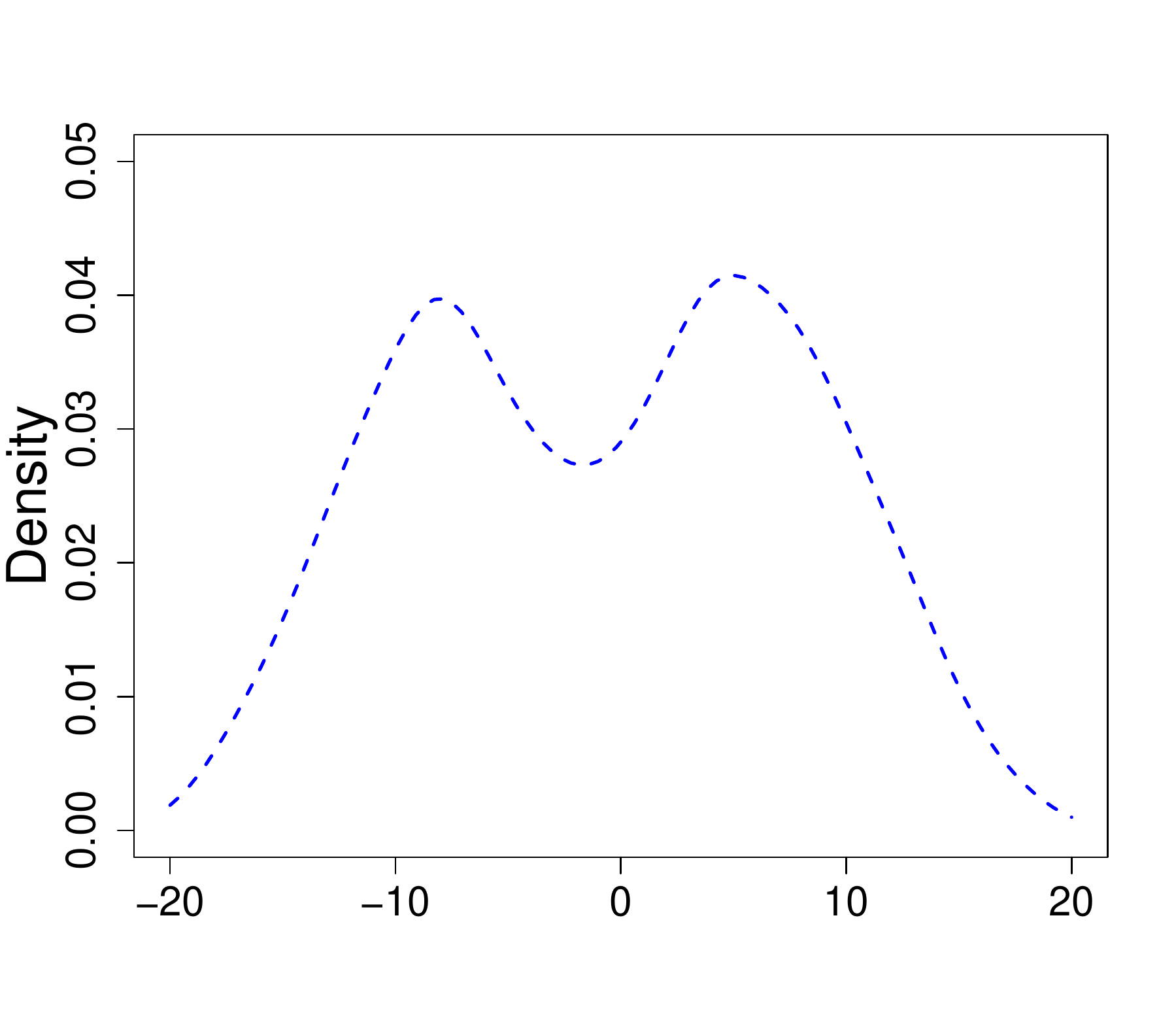}
		\end{minipage}%
	}%
	\subfigure{
		\begin{minipage}[t]{0.4\linewidth}
			\centering
			\includegraphics[width=2.5in]{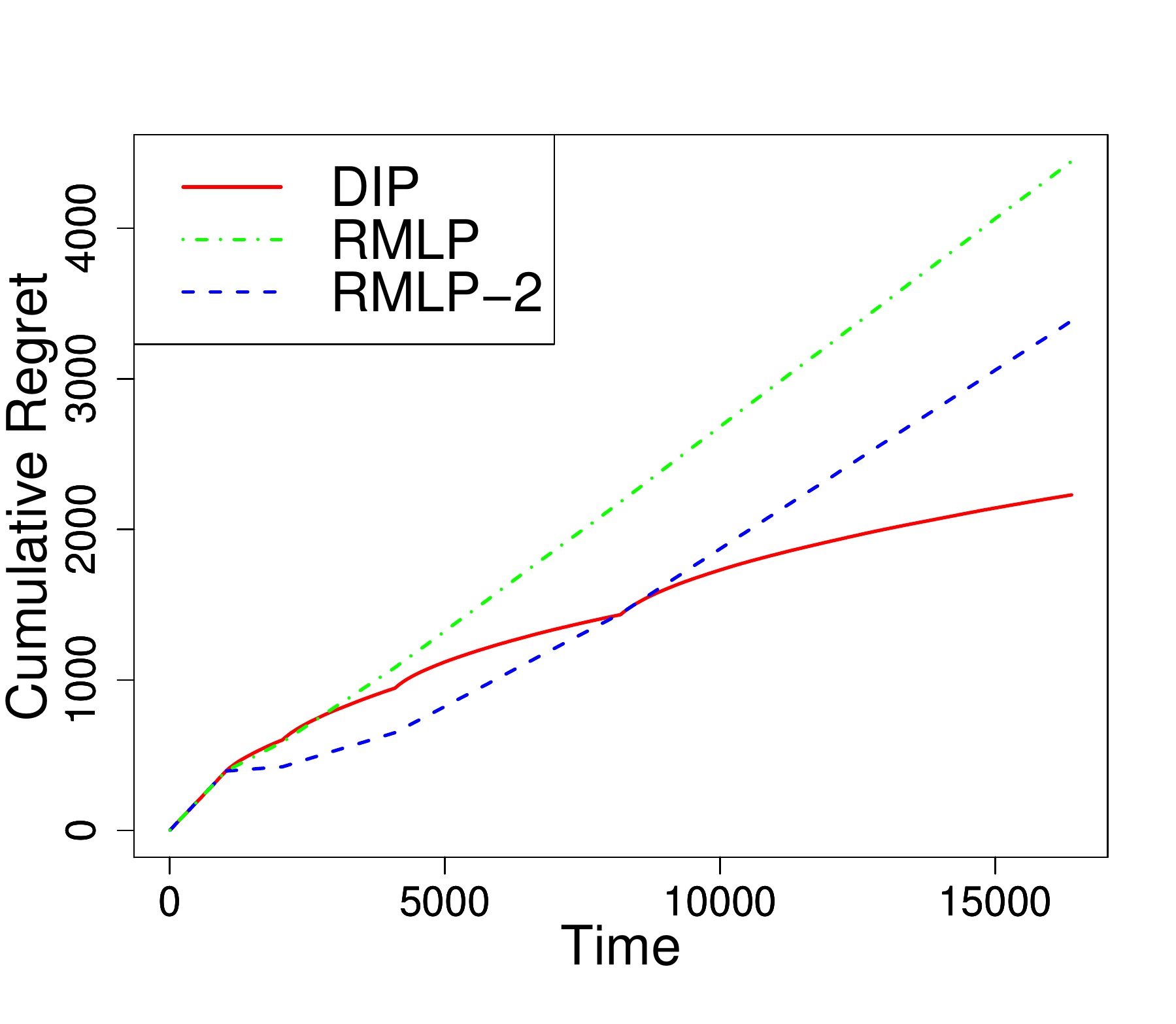}
		\end{minipage}%
	}%
	\centering
	\vspace{-1em}
	\caption{The left plot shows the PDF of the noise distribution for the California data and the right plot shows the regret comparison of DIP, RMLP and RMLP-2.}
	\label{fig:9}
\end{figure}

\section{Conclusion}
\label{sec:6}

In this paper, we consider a customer choice model generated by a linear valuation function with the unknown coefficient parameter and unknown noise distribution. A new pricing policy DIP is proposed to tackle this problem through simultaneously learning both the unknown parameter and the unknown distribution. In theory, we show that even when the noise distribution is unknown, our DIP policy is still able to achieve a sub-linear regret bound. We apply DIP on various synthetic datasets and a real online Auto Lending dataset and demonstrate its superior performance when compared with state-of-the-art pricing algorithms. 

There are a few interesting future directions. In this paper, we focus on non-sparse coefficients with an unknown noise distribution. It would be interesting to extend our policy to the high-dimensional setting with a sparse linear choice model. We can also extend the linear choice model to a more flexible semiparametric model \citep{bickel1993efficient} to allow both a parametric component and a nonparametric component on the covariates. Furthermore, it would be interesting to incorporate the considerations of fairness and welfare \citep{kallus2020fairness} into our dynamic pricing regime. 

\newpage
\appendix
\begin{center}
	\begin{Huge}
		\textbf{Appendix}
	\end{Huge}
\end{center}

\section{Technical Proofs of Lemmas, Propositions and Theorems}

\setcounter{theorem}{0}
\setcounter{lemma}{0}

\begin{lemma}\label{supp_lem:1}
	Under Assumption 1, $||\xi_{t}-\xi^{*}||_{\infty}\leq L||\hat{\theta}-\theta_{0}||_{1},\forall t\in [T_{0}]$. Moreover, under the price-action coupling $A_{t} = Q_{t}(p_{t})$, the reward $Z_{t} = p_{t}1_{\{v_{t}\geq p_{t}\}}$, the parameter $\xi_{t}$ and the action set $\mathcal{A}_{t}$ form a perturbed linear bandit with a perturbation constant $2L||\hat{\theta}-\theta_{0}||_{1}$.
\end{lemma}

\begin{proof}
	For any $t\in[T_{0}]$ and $j\in[d]$, we have
	\begin{equation*}
	\begin{aligned}
	|\xi_{tj}-\xi^{*}_{j}|& = |(1-F(m_{j}+x_{t}^{\top}\hat{\theta}-x_{t}^{\top}\theta_{0}))-(1-F(m_{j}))|\\
	& \leq L|x_{t}^{\top}\hat{\theta}-x_{t}^{\top}\theta_{0}|\leq L||x_{t}||_{\infty}||\hat{\theta}-\theta_{0}||_{1}\leq L||\hat{\theta}-\theta_{0}||_{1}.\\
	\end{aligned}
	\end{equation*}
	Thus we obtain $||\xi_{t}-\xi^{*}||_{\infty}\leq L||\hat{\theta}-\theta_{0}||_{1},\forall t\in [T_{0}]$.
	
	On the other hand, for any selected action $A_{t}\in\mathcal{A}_{t}$ with coupled set price $p_{t}=Q_{t}^{-1}(A_{t})\in\mathcal{S}_{t}$, we aim to prove that \[
	Z_{t} = p_{t}1_{\{v_{t}\geq p_{t}\}}=\langle \xi_{t},A_{t}\rangle + \eta_{t},
	\] where $\eta_{t}$ is $p_{\max}$-subgaussian conditional on $\mathcal{F}_{t-1} = \sigma(\xi_{1},A_{1},Z_{1},\dots,\xi_{t},A_{t})$. 
	
	Let $\eta_{t} = Z_{t}-p_{t}(1-F(p_{t}-x_{t}^{\top}\theta_{0}))$. Then we have $Z_{t} = p_{t}(1-F(p_{t}-x_{t}^{\top}\theta_{0}))+\eta_{t}$. By our definition, $\xi_{t} = (1-F(m_{1}+x_{t}^{\top}\hat{\theta}-x_{t}^{\top}\theta_{0}),\dots,1-F(m_{d}+x_{t}^{\top}\hat{\theta}-x_{t}^{\top}\theta_{0}))^{\top}$. Since $p_{t}\in\mathcal{S}_{t}$, there exists $j_{t}\in [d]$ such that $p_{t} = m_{j_{t}}+x_{t}^{\top}\hat{\theta}$. Since $A_{t}$ is a $d$-dimensional vector with $(A_{t})_{j_{t}} = p_{t}$ and $(A_{t})_{i}  = 0,\forall i\neq j_{t}$, we have \[
	\langle \xi_{t},A_{t}\rangle  = (A_{t})_{j_{t}}(\xi_{t})_{j_{t}} = p_{t}(1-F(m_{j_{t}}+x_{t}^{\top}\hat{\theta}-x_{t}^{\top}\theta_{0})) = p_{t}(1-F(p_{t}-x_{t}^{\top}\theta_{0})).
	\]
	Therefore, we have \[
	Z_{t} = p_{t}(1-F(p_{t}-x_{t}^{\top}\theta_{0}))+\eta_{t} = \langle \xi_{t},A_{t}\rangle +\eta_{t}. 
	\]
	
	Then we prove that $\eta_{t}$ is $p_{\max}$-subgaussian conditional on the filtration $\mathcal{F}_{t-1}$. Remember that $v_{t} = v(x_{t}) = x_{t}^{\top}\theta_{0}+z_{t}$ with $z_{t}$ drawn i.i.d. from distribution $F$. As the reward $Z_{t} = p_{t}1_{\{v_{t}\geq p_{t} \}}$, we have\[
	\mathbb{P}(Z_{t} = p_{t}| x_{1},p_{1},Z_{1},\dots,x_{t},p_{t}) = 1-F(p_{t}-x_{t}^{\top}\theta_{0}),\] \[\mathbb{P}(Z_{t} = 0| x_{1},p_{1},Z_{1},\dots,x_{t},p_{t}) = F(p_{t}-x_{t}^{\top}\theta_{0}).
	\]
	Note that we view $x_{t}$ and $p_{t}$ as potential random variables. For example, $x_{t}$ can be randomly independent of each other. Moreover, both random or deterministic choices of $x_{t}$ by an adversary according to the past random data can yield randomness in $x_{t}$. On the other hand, $p_{t}$ can be random regardless of whether it is selected by a random or deterministic pricing policy because of the randomness in past observations. 
	
	Denote $\tilde{\mathcal{F}}_{t-1} = \sigma(x_{1},p_{1},Z_{1},\dots,x_{t},p_{t})$. Then we have
	\begin{equation*}
	\begin{aligned}
	\mathbb{E}(\eta_{t}|\tilde{\mathcal{F}}_{t-1})& =\mathbb{E}(Z_{t}-p_{t}(1-F(p_{t}-x_{t}^{\top}\theta_{0}))|x_{1},p_{1},Z_{1},\dots,x_{t},p_{t})\\
	& = \mathbb{E}(Z_{t}|x_{1},p_{1},Z_{1},\dots,x_{t},p_{t})-p_{t}(1-F(p_{t}-x_{t}^{\top}\theta_{0}))\\
	& = p_{t}(1-F(p_{t}-x_{t}^{\top}\theta_{0}))-p_{t}(1-F(p_{t}-x_{t}^{\top}\theta_{0})) = 0.\\
	\end{aligned}
	\end{equation*}
	Namely, conditional on $\tilde{\mathcal{F}}_{t-1} = \sigma(x_{1},p_{1},Z_{1},\dots,x_{t},p_{t})$, $\eta_{t} = Z_{t}-p_{t}(1-F(p_{t}-x_{t}^{\top}\theta_{0}))$ has mean $0$. On the other hand, any $p\in\mathcal{S}_{t}$ satisfies $p\in(0,p_{\max})$ by the construction of $\mathcal{S}_{t}$. Thus we have $p_{t}\in(0,p_{\max})$ since $p_{t}$ is selected from $\mathcal{S}_{t}$. Thus $\eta_{t}$ is bounded in the interval $[-p_{\max},p_{\max}]$. By basic argument for subgaussian variables, $\eta_{t}$ is $\frac{p_{\max}-(-p_{\max})}{2} = p_{\max}$-subgaussian conditional on $\tilde{\mathcal{F}}_{t-1}$. Namely,
	\[
	\mathbb{E}(\exp(u\eta_{t})|\tilde{\mathcal{F}}_{t-1})\leq \exp(\frac{p_{\max}^{2}u^{2}}{2}),\forall u\in \mathbb{R}.
	\]
	
	We then derive the relationship between $\tilde{\mathcal{F}}_{t-1} $ and $\mathcal{F}_{t-1} = \sigma(\xi_{1},A_{1},Z_{1},\dots,\xi_{t},A_{t})$. By our definition, $\xi_{s} = (1-F(m_{1}+x_{s}^{\top}\hat{\theta}-x_{s}^{\top}\theta_{0}),\dots,1-F(m_{d}+x_{s}^{\top}\hat{\theta}-x_{s}^{\top}\theta_{0}))^{\top}\in \sigma(x_{s})$. Since $p_{s}\in\mathcal{S}_{s}$, there exists $j_{s}\in [d]$ such that $p_{s} = m_{j_{s}}+x_{s}^{\top}\hat{\theta}$. Thus $j_{s}\in\sigma(x_{s},p_{s})$. By the price-action coupling, $A_{s} = Q_{s}(p_{s})$ is a $d$-dimensional vector with $(A_{s})_{j_{s}} = p_{s}$ and $(A_{s})_{i}  = 0,\forall i\neq j_{s}$. Thus $A_{s}\in \sigma(x_{s},p_{s})$. Therefore, $\sigma(\xi_{s},A_{s})\in \sigma(x_{s},p_{s})$ and \[
	\mathcal{F}_{t-1} = \sigma(\xi_{1},A_{1},Z_{1},\dots,\xi_{t},A_{t})\subseteq \sigma(x_{1},p_{1},Z_{1},\dots,x_{t},p_{t})=\tilde{\mathcal{F}}_{t-1}.
	\]
	
	Therefore, we can write for any $u\in\mathbb{R}$, 
	\[\mathbb{E}(\exp(u\eta_{t})|\mathcal{F}_{t-1})
	=\mathbb{E}(\mathbb{E}(\exp(u\eta_{t})|\tilde{\mathcal{F}}_{t-1})|\mathcal{F}_{t-1})\leq\mathbb{E}(\exp(\frac{p_{\max}^{2}u^{2}}{2})|\mathcal{F}_{t-1}) = \exp(\frac{p_{\max}^{2}u^{2}}{2}).\]
	Thus $\eta_{t} = Z_{t}-p_{t}(1-F(p_{t}-x_{t}^{\top}\theta_{0}))$ is $p_{\max}$-subgaussian conditional on $\mathcal{F}_{t-1}$. 
	
	Furthermore, simple calculation yields $\forall t,s\in [T_{0}]$, $$||\xi_{t}-\xi_{s}||_{\infty}\leq ||\xi_{t}-\xi^{*}||_{\infty}+||\xi^{*}-\xi_{s}||_{\infty}\leq 2L||\hat{\theta}-\theta_{0}||_{1}.$$
	
	Therefore, under the price-action coupling, the reward $Z_{t} = p_{t}1_{\{v_{t}\geq p_{t}\}}$, $\xi_{t}$ and the action set $\mathcal{A}_{t}$ form a perturbed linear bandit with perturbation constant $2L||\hat{\theta}-\theta_{0}||_{1}$. 
\end{proof}

\begin{lemma}\label{supp_lem:2}
	Applying Algorithm 5 to the PLB formulation of our single-episode pricing problem with $\beta_{t} = \beta_{t}^{*}= p_{\max}^{2}(1\vee (\frac{1}{p_{\max}}\sqrt{\lambda d}+\sqrt{2\log(\frac{1}{\delta})+d\log(\frac{d\lambda+(t-1)p_{\max}^{2}}{d\lambda}) })^{2})$ yields Algorithm 3 using the UCB construction $\text{UCB}_{t}(1-F(m_{j})) = \frac{\sum_{s\in\mathcal{U}_{t-1,j}}p_{s}^{2}y_{s}}{\lambda+\sum_{s\in\mathcal{U}_{t-1,j}}p_{s}^{2}} + \sqrt{\frac{\beta_{t}}{\lambda + \sum_{s\in\mathcal{U}_{t-1,j}}p_{s}^{2}}}$ with $\beta_{t} = \beta_{t}^{*}$.
\end{lemma}

\begin{proof}
	We split all time periods into three cases and prove the equivalence of Algorithm 3 and Algorithm 5 for each case. 
	\begin{itemize}
		\item When $t = 1$. 
		
		In this case, Algorithm 5 chooses arbitrary $A_{t}\in\mathcal{A}_{t}$. By price-action coupling, it selects arbitrary price $p_{t}\in\mathcal{S}_{t}$. 
		
		On the other hand, $\mathcal{U}_{t-1,j} = \emptyset$ for any $j\in \mathcal{B}_{t}$ since $t = 1$. Then $\text{EST}_{t-1}(1-F(m_{j})) = \frac{\sum_{s\in\mathcal{U}_{t-1,j}}p_{s}^{2}y_{s}}{\sum_{s\in\mathcal{U}_{t-1,j}}p_{s}^{2}}= \frac{0}{0} = +\infty$ for any $j\in \mathcal{B}_{t}$. As $\text{CR}_{t-1}(1-F(m_{j})) = \sqrt{\frac{1}{\lambda + \sum_{s\in\mathcal{U}_{t-1,j}}p_{s}^{2}}}\in \mathbb{R}$ for any $j\in \mathcal{B}_{t}$, we have $\text{UCB}_{t-1}(1-F(m_{j})) = +\infty$ for any $j\in \mathcal{B}_{t}$. Thus $(m_{j}+x_{t}^{\top}\hat{\theta})\text{UCB}_{t-1}(1-F(m_{j})) = +\infty$ for any $j\in \mathcal{B}_{t}$. Therefore, Algorithm 3 selects arbitrary $j_{t}\in \arg\max_{j\in \mathcal{B}_{t}}(m_{j}+x_{t}^{\top}\hat{\theta})\text{UCB}_{t-1}(1-F(m_{j}))=\mathcal{B}_{t}$ and sets the price $p_{t} = m_{j_{t}}+x_{t}^{\top}\hat{\theta}$. Namely, Algorithm 3 sets arbitrary price $p_{t}\in\mathcal{S}_{t}$ since $\mathcal{B}_{t} = \{j\in [d]: \exists p\in \mathcal{S}_{t} \text{ such that }  p= m_{j}+x_{t}^{\top}\hat{\theta}\}$.
		
		\item When $t\geq 2$ and $\tilde{\mathcal{B}}_{t}\not\subseteq \tilde{\mathcal{B}}^{'}_{t}$. 
		
		In this case, Algorithm 5 chooses arbitrary $A_{t}\in\mathcal{A}_{t}$ such that $\delta(A_{t})\notin \tilde{\mathcal{B}}^{'}_{t}$. By price-action coupling, it sets arbitrary price $p_{t} = m_{j_{t}}+x_{t}^{\top}\hat{\theta}\in\mathcal{S}_{t}$ with $j_{t}\notin \tilde{\mathcal{B}}^{'}_{t}$. Namely, it selects arbitrary $j_{t}\in \tilde{\mathcal{B}}_{t}\setminus \tilde{\mathcal{B}}^{'}_{t}$ and sets the price $p_{t} = m_{j_{t}} + x_{t}^{\top}\hat{\theta}$. 
		
		On the other hand, we have \begin{equation*}
		\begin{aligned}
		j_{t} = j&\Leftrightarrow p_{t} = m_{j}+x_{t}^{\top}\hat{\theta}\\
		&\Leftrightarrow A_{t} = Q_{t}(p_{t}) \text{ has only one nonzero element with index }j\\
		&\Leftrightarrow \delta(A_{t}) = j.\\
		\end{aligned}
		\end{equation*}Thus we have \begin{equation*}
		\begin{aligned}
		&(m_{j}+x_{t}^{\top}\hat{\theta})\text{UCB}_{t-1}(1-F(m_{j})) = +\infty\Leftrightarrow\text{UCB}_{t-1}(1-F(m_{j}))=+\infty\\
		&\Leftrightarrow \mathcal{U}_{t-1,j} = \{s:1\leq s\leq t-1,j_{s}=j \} = \emptyset\\
		&\Leftrightarrow \{s:1\leq s\leq t-1,\delta(A_{s})=j \} = \emptyset\\
		&\Leftrightarrow j\notin \tilde{\mathcal{B}}^{'}_{t} = \{\delta(A_{s}):s\in[t-1]\}.\\
		\end{aligned}
		\end{equation*}
		Moreover, by price-action coupling, we have\begin{equation*}
		\begin{aligned}
		\mathcal{B}_{t} &= \{j\in [d]: \exists p\in \mathcal{S}_{t} \text{ such that }  p= m_{j}+x_{t}^{\top}\hat{\theta}\} \\
		&= \{j\in [d]: \exists p\in \mathcal{S}_{t} \text{ such that }  \delta(Q_{t}(p))= j\} \\
		&=  \{\delta(a):a\in \mathcal{A}_{t}\}=\tilde{\mathcal{B}}_{t}\\
		\end{aligned}
		\end{equation*}
		Since $\tilde{\mathcal{B}}_{t}\not\subseteq \tilde{\mathcal{B}}^{'}_{t}$, we have $\mathcal{B}_{t}\not\subseteq \tilde{\mathcal{B}}^{'}_{t}$. Namely, $\mathcal{B}_{t}\setminus \tilde{\mathcal{B}}^{'}_{t}\neq \emptyset$. Then for any $j\in \mathcal{B}_{t}\setminus \tilde{\mathcal{B}}^{'}_{t}$, we have $(m_{j}+x_{t}^{\top}\hat{\theta})\text{UCB}_{t-1}(1-F(m_{j})) = +\infty$. However, for any $j\in\mathcal{B}_{t}\setminus(\mathcal{B}_{t}\setminus \tilde{\mathcal{B}}^{'}_{t}) = \mathcal{B}_{t}\cap \tilde{\mathcal{B}}^{'}_{t}$, we have $(m_{j}+x_{t}^{\top}\hat{\theta})\text{UCB}_{t-1}(1-F(m_{j})) < +\infty$.  Therefore, Algorithm 3 selects arbitrary $j_{t}\in \arg\max_{j\in \mathcal{B}_{t}}(m_{j}+x_{t}^{\top}\hat{\theta})\text{UCB}_{t-1}(1-F(m_{j}))=\mathcal{B}_{t}\setminus \tilde{\mathcal{B}}^{'}_{t}$ and sets the price $p_{t} = m_{j_{t}}+x_{t}^{\top}\hat{\theta}$. 
		
		\item When $t\geq 2$ and $\tilde{\mathcal{B}}_{t}\subseteq \tilde{\mathcal{B}}^{'}_{t}$. 
		
		In this case, Algorithm 5 chooses $A_{t} = \arg\max_{a\in\mathcal{A}_{t}}\text{LinUCB}_{t}(a)$ where $\text{LinUCB}_{t}(a) = \max_{\xi\in\mathcal{C}_{t}}\langle \xi,a\rangle$. Note that $\mathcal{C}_{t} = \{\xi\in\mathbb{R}^{d}: ||\xi- \hat{\xi}_{t-1}||_{V_{t-1}(\lambda)}^{2}\leq \beta_{t} \}$ where $\beta_{t} = \beta_{t}^{*}$ used in Algorithm 3. Moreover, $V_{t-1}(\lambda) = \lambda I +\sum_{s=1}^{t-1}A_{s}A_{s}^{\top}$. Since $A_{s}$ has only one nonzero element, $V_{t-1}(\lambda)$ is a diagonal matrix. As proved in the second case, $\delta(A_{s}) = i\Leftrightarrow j_{s} = i$ by price-action coupling. Thus we have $$V_{t-1}(\lambda)_{ii} =\lambda+ \sum_{s\in[t-1],\delta(A_{s}) = i}A_{si}^{2} = \lambda + \sum_{s\in \mathcal{U}_{t-1,i}}p_{s}^{2}.$$ As $\hat{\xi}_{t-1} = V_{t-1}(\lambda)^{-1}\sum_{s=1}^{t-1}A_{s}Z_{s}$, we have $$(\hat{\xi}_{t-1})_{i} =\frac{\sum_{s\in[t-1],\delta(A_{s}) = i}A_{si}^{2}y_{s}}{\sum_{s\in[t-1],\delta(A_{s}) = i}A_{si}^{2}}= \frac{\sum_{s\in\mathcal{U}_{t-1,i}}p_{s}^{2}y_{s}}{\sum_{s\in\mathcal{U}_{t-1,i}}p_{s}^{2}}=\text{EST}_{t-1}(1-F(m_{i})).$$
		Since $V_{t-1}(\lambda)$ is a diagonal matrix, we have \begin{equation*}
		\begin{aligned}
		\max_{\xi\in \mathcal{C}_{t}}\xi_{i} &= (\hat{\xi}_{t-1})_{i} + \frac{\sqrt{\beta^{*}_{t}}}{\sqrt{V_{t-1}(\lambda)_{ii}}} = \text{EST}_{t-1}(1-F(m_{i}))+\text{CR}(1-F(m_{i})) \\&= \text{UCB}_{t-1}(1-F(m_{i}))\\
		\end{aligned}
		\end{equation*}For any $a\in\mathcal{A}_{t}$, there exists a price $p\in\mathcal{S}_{t}$ such that $a = Q_{t}(p)$. There further exists a $j\in \tilde{\mathcal{B}}_{t}$ such that $p = m_{j}+x_{t}^{\top}\hat{\theta}$. Then $a_{j} = p = m_{j}+x_{t}^{\top}\hat{\theta}$ and $a_{i} = 0,\forall i\neq j$. Then $$\text{LinUCB}_{t}(a) = \max_{\xi\in\mathcal{C}_{t}}\langle \xi,a\rangle = (m_{j}+x_{t}^{\top}\hat{\theta})\max_{\xi\in \mathcal{C}_{t}}\xi_{j}  = (m_{j}+x_{t}^{\top}\hat{\theta})\text{UCB}_{t-1}(1-F(m_{j})).$$
		Also, for any $j\in \tilde{\mathcal{B}}_{t}$, there exists an action $a = Q_{t}(m_{j}+x_{t}^{\top}\hat{\theta})\in\mathcal{A}_{t}$. Note that Algorithm 5 chooses arbitrary $A_{t}\in\arg\max_{a\in\mathcal{A}_{t}}\text{LinUCB}_{t}(a)$. By the price-action coupling, it selects any arbitrary $j_{t} \in \arg\max_{j\in \tilde{\mathcal{B}}_{t}}(m_{j}+x_{t}^{\top}\hat{\theta})\text{UCB}_{t-1}(1-F(m_{j}))$ and sets the price $p_{t} = m_{j_{t}} + x_{t}^{\top}\hat{\theta}$. 
		
		On the other hand, Algorithm 3 simply selects any arbitrary $j_{t} \in \arg\max_{j\in \tilde{\mathcal{B}}_{t}}(m_{j}+x_{t}^{\top}\hat{\theta})\text{UCB}_{t-1}(1-F(m_{j}))$ and sets the price $p_{t} = m_{j_{t}} + x_{t}^{\top}\hat{\theta}$. 
	\end{itemize}
	Therefore, by the price-action coupling, applying Algorithm 5 to the PLB formulation of our single-episode pricing problem with $\beta_{t} = \beta_{t}^{*}$ yields Algorithm 3 using the UCB construction $\text{UCB}_{t}(1-F(m_{j})) = \frac{\sum_{s\in\mathcal{U}_{t-1,j}}p_{s}^{2}y_{s}}{\lambda+\sum_{s\in\mathcal{U}_{t-1,j}}p_{s}^{2}} + \sqrt{\frac{\beta_{t}}{\lambda + \sum_{s\in\mathcal{U}_{t-1,j}}p_{s}^{2}}}$ with $\beta_{t} = \beta_{t}^{*}$.
\end{proof}

\setcounter{proposition}{0}
\renewcommand{\thelemma}{S\arabic{lemma}}
\setcounter{lemma}{0}

\begin{lemma}\label{supp_lem:3}
	Let $S_{1}\subseteq S_{2}\subseteq \mathbb{N}^{+}$. Suppose $|a_{i}-a_{j}|\leq \delta,\forall i,j\in S_{2}$, then for any positive numbers $\{b_{i}\}_{i\in S_{1}}$, we have 
	
	\[
	|\frac{\sum_{i\in S_{1}}b_{i}a_{i}}{\sum_{i\in S_{1}}b_{i}}-a_{j}|\leq \delta,\forall j\in S_{2}.
	\]
\end{lemma}

\begin{proof}
	Direct calculation shows for any $j\in S_{2}$, \[
	|\frac{\sum_{i\in S_{1}}b_{i}a_{i}}{\sum_{i\in S_{1}}b_{i}}-a_{j}|= |\frac{\sum_{i\in S_{1}}b_{i}(a_{i}-a_{j})}{\sum_{i\in S_{1}}b_{i}}|\leq\frac{\sum_{i\in S_{1}}b_{i}|a_{i}-a_{j}|}{\sum_{i\in S_{1}}b_{i}} \leq \delta.
	\]
\end{proof}

\begin{lemma}\label{supp_lem:4}
	Let $V_{0}$ be positive definite and $x_{1},\dots,x_{n}\in \mathbb{R}^{d}$ be a sequence of vectors with $||x_{t}||_{2}\leq L<\infty$ for all $t\in[n]$. Then 
	\[
	\sum_{t=1}^{n}(1\wedge||x_{t}||_{V_{t-1}^{-1}})\leq 2\log(\frac{\det V_{n}}{\det V_{0}})\leq 2d\log(\frac{\text{trace}(V_{0})+nL^{2}}{d\det^{1/d}V_{0}})
	\]
	where $V_{t} = V_{0} + \sum_{i=1}^{t}x_{t}x_{t}^{\top}$.
\end{lemma}

\begin{proof}
	This Lemma is exactly Lemma 19.4 in \cite{lattimore2020bandit}. Thus we refer to the proof of Lemma 19.4 in \cite{lattimore2020bandit}. 
\end{proof}

\begin{lemma}\label{supp_lem:5}
	Consider the setting in Lemma \ref{supp_lem:8} where $Z_{t} = \langle \xi_{t},A_{t}\rangle +\eta_{t}$, $V_{t}(\lambda) = \lambda I+\sum_{s=1}^{t}A_{s}A_{s}^{\top}$ and $V_{t} = V_{t}(0)$. Note that we define $\dot{\xi}_{t} = V_{t-1}^{+}\sum_{s=1}^{t-1}A_{s}A_{s}^{\top}\xi_{s}$ where $V_{t-1}^{+}$ is the Moore–Penrose inverse of $V_{t-1}$. 
	
	Then we have $||\dot{\xi}_{t+1}-V_{t}(\lambda)^{-1}\sum_{s=1}^{t}A_{s}A_{s}^{\top}\xi_{s}||_{V_{t}(\lambda)}\leq C_{1}\sqrt{\lambda d}$.
\end{lemma}

\begin{proof}
	Note that $V_{t}(\lambda) = \lambda I+\sum_{s=1}^{t}A_{s}A_{s}^{\top}$ is diagonal and we have defined $\tilde{\mathcal{B}}^{'}_{t} = \{\delta(A_{s}):s\in[t-1]\}$. Thus we have
	\begin{equation*}
	V_{t}(\lambda)_{ii} = \begin{cases}
	\lambda, \text{ for } i\notin \tilde{\mathcal{B}}^{'}_{t+1},\\
	\sum_{s\in [t]:\delta(A_{s}) = i}A^{2}_{si}+\lambda, \text{ for } i\in \tilde{\mathcal{B}}^{'}_{t+1}.
	\end{cases}
	\end{equation*}
	Therefore, 
	\begin{equation*}
	V_{t}^{+} = V_{t}(0)^{+} = \begin{cases}
	0, \text{ for } i\notin \tilde{\mathcal{B}}^{'}_{t+1},\\
	(\sum_{s\in [t]:\delta(A_{s}) = i}A^{2}_{si})^{-1}, \text{ for } i\in \tilde{\mathcal{B}}^{'}_{t+1}.
	\end{cases}
	\end{equation*}
	Let $v = \dot{\xi}_{t+1}-V_{t}(\lambda)^{-1}\sum_{s=1}^{t}A_{s}A_{s}^{\top}\xi_{s} = V_{t}^{+}\sum_{s=1}^{t}A_{s}A_{s}^{\top}\xi_{s}-V_{t}(\lambda)^{-1}\sum_{s=1}^{t}A_{s}A_{s}^{\top}\xi_{s}$. Then
	\begin{equation*}
	v_{i} = \begin{cases}
	0, \text{ for } i\notin \tilde{\mathcal{B}}^{'}_{t+1},\\
	\frac{\sum_{s\in [t]:\delta(A_{s}) = i}A^{2}_{si}\xi_{si}}{\sum_{s\in [t]:\delta(A_{s}) = i}A^{2}_{si}}-\frac{\sum_{s\in [t]:\delta(A_{s}) = i}A^{2}_{si}\xi_{si}}{\sum_{s\in [t]:\delta(A_{s}) = i}A^{2}_{si}+\lambda}, \text{ for } i\in \tilde{\mathcal{B}}^{'}_{t+1}.
	\end{cases}
	\end{equation*}
	For $i\in \tilde{\mathcal{B}}^{'}_{t+1}$, denote $u_{i} = \frac{\sum_{s\in [t]:\delta(A_{s}) = i}A^{2}_{si}\xi_{si}}{\sum_{s\in [t]:\delta(A_{s}) = i}A^{2}_{si}}$. By Condition 2, $||\xi_{s}||_{\infty}\leq C_{1},\forall s\in \mathbb{N}^{+}$. Thus $|u_{i}|\leq C_{1}$ for any $i\in \tilde{\mathcal{B}}^{'}_{t+1}$. 
	\begin{equation*}
	\begin{aligned}
	\Rightarrow ||v||_{V_{t}(\lambda)}^{2} &= \sum_{i\in \tilde{\mathcal{B}}^{'}_{t+1}}(\sum_{s\in [t]:\delta(A_{s}) = i}A^{2}_{si}+\lambda)(u_{i}-u_{i}\frac{\sum_{s\in [t]:\delta(A_{s}) = i}A^{2}_{si}}{\sum_{s\in [t]:\delta(A_{s}) = i}A^{2}_{si}+\lambda})^{2}\\
	& =\sum_{i\in \tilde{\mathcal{B}}^{'}_{t+1}} u_{i}^{2}(\sum_{s\in [t]:\delta(A_{s}) = i}A^{2}_{si}+\lambda)(1-\frac{\sum_{s\in [t]:\delta(A_{s}) = i}A^{2}_{si}}{\sum_{s\in [t]:\delta(A_{s}) = i}A^{2}_{si}+\lambda})^{2}\\
	& = \sum_{i\in \tilde{\mathcal{B}}^{'}_{t+1}} u_{i}^{2}\lambda \frac{\lambda}{\sum_{s\in [t]:\delta(A_{s}) = i}A^{2}_{si}+\lambda}\\
	&\leq \sum_{i\in \tilde{\mathcal{B}}^{'}_{t+1}} u_{i}^{2}\lambda  \leq \lambda dC_{1}^{2}\\
	\end{aligned}
	\end{equation*}
	\begin{equation*}
	\Rightarrow ||\dot{\xi}_{t+1}-V_{t}(\lambda)^{-1}\sum_{s=1}^{t}A_{s}A_{s}^{\top}\xi_{s}||_{V_{t}(\lambda)} = ||v||_{V_{t}(\lambda)}\leq C_{1}\sqrt{\lambda d}.
	\end{equation*}
\end{proof}

\begin{lemma}\label{supp_lem:6}
	Consider the setting in Lemma \ref{supp_lem:8}. Let $S_{t} = \sum_{s=1}^{t}\eta_{s}A_{s}$, then for all $\lambda>0$ and $\delta\in (0,1)$, 
	\[
	\mathbb{P}(\exists t\in \mathbb{N}:||S_{t}||_{V_{t}(\lambda)^{-1}}^{2}\geq 2\log(\frac{1}{\delta})+\log(\frac{\det(V_{t}(\lambda)) }{\lambda^{d}}))\leq \delta.
	\]
\end{lemma}

\begin{proof}
	With $\eta_{t}$ being $1$-subgaussian conditional on $\mathcal{F}_{t-1}$, $A_{t}\in\mathcal{F}_{t-1}$, $S_{t} = \sum_{s=1}^{t}\eta_{s}A_{s}$ and $V_{t}(\lambda) = \lambda I+\sum_{s=1}^{t}A_{s}A_{s}^{\top}$, this Lemma follows exactly the same lines as Theorem 20.4 in \cite{lattimore2020bandit}. Thus we refer to the proof of Theorem 20.4 in \cite{lattimore2020bandit}. 
\end{proof}

\begin{lemma}\label{supp_lem:7}
	Consider the setting in Lemma \ref{supp_lem:8}. Let $\delta\in (0,1)$. Then with probability at least $1-\delta$ it holds that for all $t\geq 2$, 
	\[
	||\dot{\xi}_{t}-\hat{\xi}_{t-1}||_{V_{t-1}(\lambda)}\leq C_{1}\sqrt{\lambda d}+\sqrt{2\log(\frac{1}{\delta})+d\log(\frac{d\lambda + (t-1)a_{\max}^{2}}{d\lambda})}. 
	\]
\end{lemma}

\begin{proof}
	By direct calculation, we have
	\begin{equation*}
	\begin{aligned}
	||\dot{\xi}_{t}-\hat{\xi}_{t-1}||_{V_{t-1}(\lambda)}& = ||\dot{\xi}_{t}-V_{t-1}(\lambda)^{-1}\sum_{s=1}^{t-1}A_{s}A_{s}^{\top}\xi_{s}+V_{t-1}(\lambda)^{-1}\sum_{s=1}^{t-1}A_{s}A_{s}^{\top}\xi_{s}-\hat{\xi}_{t-1}||_{V_{t-1}(\lambda)}\\
	&\leq||\dot{\xi}_{t}-V_{t-1}(\lambda)^{-1}\sum_{s=1}^{t-1}A_{s}A_{s}^{\top}\xi_{s}||_{V_{t-1}(\lambda)}+||V_{t-1}(\lambda)^{-1}\sum_{s=1}^{t-1}A_{s}A_{s}^{\top}\xi_{s}-\hat{\xi}_{t-1}||_{V_{t-1}(\lambda)}\\
	& = ||\dot{\xi}_{t}-V_{t-1}(\lambda)^{-1}\sum_{s=1}^{t-1}A_{s}A_{s}^{\top}\xi_{s}||_{V_{t-1}(\lambda)}+||V_{t-1}(\lambda)^{-1}S_{t-1}||_{V_{t-1}(\lambda)}\\
	& = ||\dot{\xi}_{t}-V_{t-1}(\lambda)^{-1}\sum_{s=1}^{t-1}A_{s}A_{s}^{\top}\xi_{s}||_{V_{t-1}(\lambda)}+||S_{t-1}||_{V_{t-1}(\lambda)^{-1}}\\
	(\text{By Lemma \ref{supp_lem:5}})& \leq C_{1}\sqrt{\lambda d} + ||S_{t-1}||_{V_{t-1}(\lambda)^{-1}}.\\
	\end{aligned}
	\end{equation*}
	Thus by Lemma \ref{supp_lem:6}, $$||\dot{\xi}_{t}-\hat{\xi}_{t-1}||_{V_{t-1}(\lambda)}\leq C_{1}\sqrt{\lambda d}+\sqrt{2\log(\frac{1}{\delta})+\log(\frac{\det(V_{t-1}(\lambda)) }{\lambda^{d}})}$$ holds simultaneously for all $t\geq 2$ with probability at least $1-\delta$. 
	Since $V_{t}(\lambda) = \lambda I + \sum_{i=1}^{t}A_{s}A_{s}^{\top}= V_{0}(\lambda)+ \sum_{i=1}^{t}A_{s}A_{s}^{\top}$ and $||A_{s}||_{2}\leq a_{\max}$ for all $s\in [t-1]$ by Condition 3, we have by Lemma \ref{supp_lem:4}, 
	\[
	\log(\frac{\det(V_{t-1}(\lambda)) }{\lambda^{d}})=\log(\frac{\det(V_{t-1}(\lambda)) }{\det(V_{0}(\lambda))})\leq d\log(\frac{d\lambda + (t-1)a_{\max}^{2}}{d\lambda}). 
	\]
	Therefore, $||\dot{\xi}_{t}-\hat{\xi}_{t-1}||_{V_{t-1}(\lambda)}\leq C_{1}\sqrt{\lambda d}+\sqrt{2\log(\frac{1}{\delta})+d\log(\frac{d\lambda + (t-1)a_{\max}^{2}}{d\lambda})}$ holds simultaneously for all $t\geq 2$ with probability at least $1-\delta$. 
\end{proof}

\renewcommand{\thelemma}{\arabic{lemma}}
\setcounter{lemma}{2}

\begin{lemma}\label{supp_lem:8}
	Consider the PLB satisfying Conditions 1-4 with a perturbation $C_{p}$. With probability at least $1-\delta$, Algorithm 5 with $\beta_{t} = \tilde{\beta}_{t} =1\vee (C_{1}\sqrt{\lambda d}+\sqrt{2\log(\frac{1}{\delta})+d\log(\frac{d\lambda+(t-1)a_{\max}^{2}}{d\lambda}) })^{2}$ has the regret bound
	\[
	R^{PLB}_{T_{0}}\leq 2\sqrt{2dT_{0}\tilde{\beta}_{T_{0}}\log(\frac{d\lambda+T_{0}a_{\max}^{2}}{d\lambda})}+  2a_{\max}C_{p}T_{0}+2d.
	\]
\end{lemma}

\begin{proof}
	For $t\geq 2$, denote $\tilde{\mathcal{C}}_{t} = \mathcal{C}_{t}(\tilde{\beta}_{t}) = \{ \xi\in\mathbb{R}^{d}: ||\xi- \hat{\xi}_{t-1}||_{V_{t-1}(\lambda)}^{2}\leq \tilde{\beta}_{t} \}$. Then by Lemma \ref{supp_lem:7}, the event $\{\dot{\xi}_{t}\in \tilde{\mathcal{C}}_{t},2\leq t\leq T_{0}\}$ holds with probability at least $1-\delta$. Therefore, we only need to prove 
	\[
	R^{PLB}_{T_{0}}\leq 2\sqrt{2dT_{0}\tilde{\beta}_{T_{0}}\log(\frac{d\lambda+T_{0}a_{\max}^{2}}{d\lambda})}+  2T_{0}C_{p}a_{\max}+2d
	\]
	when this event $\{\dot{\xi}_{t}\in \tilde{\mathcal{C}}_{t},2\leq t\leq T_{0}\}$ holds. 
	
	Denote $\dot{A}_{t} = \arg\max_{a\in \mathcal{A}_{t}}\langle \dot{\xi}_{t}, a\rangle$ as the best action in $\mathcal{A}_{t}$ with respect to $\dot{\xi}_{t}$. Let $\tilde{\xi}_{t} \in \tilde{\mathcal{C}}_{t}$ such that $\langle \tilde{\xi}_{t},A_{t}\rangle = \text{LinUCB}_{t}(A_{t})$. Note that $\text{LinUCB}_{t}(a) = \max_{\xi\in \tilde{\mathcal{C}}_{t}}\langle \xi ,a\rangle$ and $A_{t} = \arg\max_{a\in \mathcal{A}_{t}}\text{LinUCB}_{t}(a)$. For any $t\geq 2$, since $\dot{\xi}_{t}\in \tilde{\mathcal{C}}_{t}$, we have $$\langle \dot{\xi}_{t},\dot{A}_{t}\rangle \leq \text{LinUCB}_{t}(\dot{A}_{t})\leq \text{LinUCB}_{t}(A_{t}) = \langle \tilde{\xi}_{t},A_{t}\rangle.$$
	Therefore, the pseudo-regret at time $t$, denoted as $\dot{r}^{PLB}_{t}$, satisfies
	\begin{equation*}
	\Rightarrow \dot{r}^{PLB}_{t} = \langle \dot{\xi}_{t},\dot{A}_{t}-A_{t}\rangle \leq \langle \tilde{\xi}_{t}-\dot{\xi}_{t},A_{t}\rangle\leq ||A_{t}||_{V_{t-1}(\lambda)^{-1}}||\tilde{\xi}_{t}-\dot{\xi}_{t}||_{V_{t-1}(\lambda)}\\
	\leq 2||A_{t}||_{V_{t-1}(\lambda)^{-1}}\sqrt{\tilde{\beta}_{t}}.
	\end{equation*}
	
	On the other hand, let $A_{t}^{*} = \arg\max_{a\in \mathcal{A}_{t}}\langle \xi_{t},a\rangle$. Let $\tilde{Q}_{1} = \{1\}, \tilde{Q}_{2} = \{2\leq t\leq T_{0}:\mathcal{B}^{*}_{t}\not\subseteq \mathcal{B}_{t} \}$, $\tilde{Q}_{3} = \{2\leq t\leq T_{0}: \tilde{\mathcal{B}}_{t}\subseteq \tilde{\mathcal{B}}^{'}_{t} \}$. Then $\tilde{Q}_{1}\cup\tilde{Q}_{2}\cup \tilde{Q}_{3} = [T_{0}]$.
	
	Note that $|\tilde{\mathcal{B}}^{'}_{2}| = 1$. For any $t\in \tilde{Q}_{2}$ and $t\leq T_{0}-1$, we select $A_{t}\in \mathcal{A}_{t}$ such that $\delta(A_{t})\notin \tilde{\mathcal{B}}^{'}_{t}$. Thus for such $t$ we have $|\tilde{\mathcal{B}}^{'}_{t+1}|-|\tilde{\mathcal{B}}^{'}_{t}| = 1$. Also note that $|\tilde{\mathcal{B}}^{'}_{t-1}|\leq |\tilde{\mathcal{B}}^{'}_{t}|\leq d,\forall 2\leq t\leq T_{0}+1$. Thus we can conclude that $|\tilde{Q}_{2}|\leq d-1$. If not, there exists $2\leq s_{1}<s_{2}<\dots<s_{d}\leq T_{0}$ such that $s_{i}\in \tilde{Q}_{2},i\in [d]$. Then we have $|\tilde{\mathcal{B}}^{'}_{s_{1}}|\geq |\tilde{\mathcal{B}}^{'}_{2}| = 1$, $|\tilde{\mathcal{B}}^{'}_{s_{i+1}}|\geq |\tilde{\mathcal{B}}^{'}_{s_{i}+1}|\geq|\tilde{\mathcal{B}}^{'}_{s_{i}}|+1$ for $i\in [d-1]$ and $|\tilde{\mathcal{B}}^{'}_{s_{d}+1}|\geq |\tilde{\mathcal{B}}^{'}_{s_{d}}| + 1$. Thus we can conclude $d\geq |\tilde{\mathcal{B}}^{'}_{T_{0}+1}|\geq |\tilde{\mathcal{B}}^{'}_{s_{d}+1}|\geq d+1$, contradiction. 
	
	By Condition 1, $r^{PLB}_{t} = \langle \xi_{t},A_{t}^{*}-A_{t}\rangle \leq |\langle \xi_{t},A_{t}^{*}\rangle| + |\langle \xi_{t},A_{t}\rangle|\leq 2 $. Therefore, 
	\begin{enumerate}
		\item For $t\in \tilde{Q}_{1}$, $r^{PLB}_{t} = r^{PLB}_{1}\leq2$.
		\item For $t\in \tilde{Q}_{2}$, $r^{PLB}_{t}\leq 2\leq 2 + 2\wedge 2\sqrt{\tilde{\beta}_{t}}||A_{t}||_{V_{t-1}(\lambda)^{-1}}$. 
		\item For $t\in \tilde{Q}_{3}$, $\tilde{\mathcal{B}}_{t}\subseteq \tilde{\mathcal{B}}^{'}_{t}$. Therefore, $\forall i\in \tilde{\mathcal{B}}_{t}, i\in \tilde{\mathcal{B}}^{'}_{t}$. Then $\exists s\in [t-1]$ such that $\delta(A_{s}) = i$. Thus $(V_{t-1})_{ii} = \sum_{s\in [t-1]:\delta(A_{s}) = i}A_{si}^{2}$. Since $\dot{\xi}_{t} = V_{t-1}^{+}\sum_{s=1}^{t-1}A_{s}A_{s}^{\top}\xi_{s}$, we have
		\[
		\dot{\xi}_{ti}  = \frac{\sum_{s\in [t-1]:\delta(A_{s}) = i}A_{si}^{2}\xi_{si}}{\sum_{s\in [t-1]:\delta(A_{s}) = i}A_{si}^{2}}.
		\]
		Since $||\xi_{s_{1}}-\xi_{s_{2}}||_{\infty}\leq C_{p},\forall s_{1},s_{2}\in \mathbb{N}^{+}$. Thus $|\xi_{s_{1}i}-\xi_{s_{2}i}|\leq C_{p},\forall s_{1},s_{2}\in \mathbb{N}^{+}$. Therefore, By Lemma \ref{supp_lem:3}, $|\dot{\xi}_{ti}-\xi_{ti}|\leq C_{p}$. 
		
		Since $A_{t}^{*}\in \mathcal{A}_{t}$, $\delta(A_{t}^{*})\in \tilde{\mathcal{B}}_{t}$. Thus $|(\dot{\xi}_{t})_{\delta(A_{t}^{*})}-(\xi_{t})_{\delta(A_{t}^{*})}|\leq C_{p}$. 
		\begin{equation*}
		\begin{aligned}
		\Rightarrow \langle \xi_{t},A_{t}^{*}\rangle  &= (\xi_{t})_{\delta(A_{t}^{*})}(A_{t}^{*})_{\delta(A_{t}^{*})}\leq (\dot{\xi}_{t})_{\delta(A_{t}^{*})}(A_{t}^{*})_{\delta(A_{t}^{*})}+C_{p}a_{\max} \\ &= \langle \dot{\xi}_{t},A_{t}^{*}\rangle + C_{p}a_{\max} \leq \langle \dot{\xi}_{t},\dot{A}_{t}\rangle + C_{p}a_{\max}. 
		\end{aligned}
		\end{equation*}
		Similarly, $A_{t}\in \mathcal{A}_{t}\Rightarrow \delta(A_{t})\in \tilde{\mathcal{B}}_{t}\Rightarrow |(\dot{\xi}_{t})_{\delta(A_{t})}-(\xi_{t})_{\delta(A_{t})}|\leq C_{p}$.
		\begin{equation*}
		\Rightarrow \langle \dot{\xi}_{t},A_{t}\rangle = (\dot{\xi}_{t})_{\delta(A_{t})}(A_{t})_{\delta(A_{t})}\leq (\xi_{t})_{\delta(A_{t})}(A_{t})_{\delta(A_{t})}+C_{p}a_{\max} = \langle \xi_{t},A_{t}\rangle + C_{p}a_{\max}.
		\end{equation*}
		Therefore, 
		\[
		r^{PLB}_{t} = \langle \xi_{t},A_{t}^{*}-A_{t}\rangle \leq \langle \dot{\xi}_{t},\dot{A}_{t}-A_{t}\rangle + 2C_{p}a_{\max}\leq 2||A_{t}||_{V_{t-1}(\lambda)^{-1}}\sqrt{\tilde{\beta}_{t}} +2C_{p}a_{\max} .
		\]
		Since $r^{PLB}_{t}\leq 2$, we obtain \[r^{PLB}_{t}\leq 2\wedge (2||A_{t}||_{V_{t-1}^{-1}}\sqrt{\tilde{\beta}_{t}} +2C_{p}a_{\max}) \leq 2\wedge 2\sqrt{\tilde{\beta}_{t}}||A_{t}||_{V_{t-1}(\lambda)^{-1}} +2C_{p}a_{\max}.\]
	\end{enumerate}
	Therefore, 
	\begin{equation*}
	\begin{aligned}
	R^{PLB}_{T_{0}}& = \sum_{t=1}^{T_{0}}r^{PLB}_{t} = \sum_{t\in \tilde{Q}_{1}\cap [T_{0}]}r^{PLB}_{t}+\sum_{t\in \tilde{Q}_{2}\cap[T_{0}]}r^{PLB}_{t}+\sum_{t\in \tilde{Q}_{3}\cap [T_{0}]}r^{PLB}_{t}\\
	& \leq 2 + \sum_{t\in \tilde{Q}_{2}\cap [T_{0}]}(2 + 2\wedge 2\sqrt{\tilde{\beta}_{t}}||A_{t}||_{V_{t-1}(\lambda)^{-1}} ) \\
	&+ \sum_{t\in \tilde{Q}_{3}\cap [T_{0}]}(2\wedge 2\sqrt{\tilde{\beta}_{t}}||A_{t}||_{V_{t-1}(\lambda)^{-1}} +2C_{p}a_{\max})\\
	& \leq 2+2(d-1) + 2T_{0}C_{p}a_{\max}+\sum_{2\leq t\leq T_{0}}2\wedge 2\sqrt{\tilde{\beta}_{t}}||A_{t}||_{V_{t-1}(\lambda)^{-1}} \\
	(\tilde{\beta}_{T_{0}}\geq 1\vee \tilde{\beta}_{t})&\leq 2d +  2T_{0}C_{p}a_{\max}+ \sum_{2\leq t\leq T_{0}}2\sqrt{\tilde{\beta}_{T_{0}}}(1\wedge ||A_{t}||_{V_{t-1}(\lambda)^{-1}})\\
	\text{(Cauchy-Schwarz)}&\leq 2d +  2T_{0}C_{p}a_{\max}+\sqrt{(T_{0}-1)\sum_{2\leq t\leq T_{0}}(2\sqrt{\tilde{\beta}_{T_{0}}}(1\wedge ||A_{t}||_{V_{t-1}(\lambda)^{-1}}))^{2}}\\
	(\text{Lemma \ref{supp_lem:4}})&\leq 2d +  2T_{0}C_{p}a_{\max}+2\sqrt{2dT_{0}\tilde{\beta}_{T_{0}}\log(\frac{d\lambda+T_{0}a_{\max}^{2}}{d\lambda})}.
	\end{aligned}
	\end{equation*}
\end{proof}

\setcounter{proposition}{1}

\begin{proposition}\label{supp_prop:1}
	For any PLB algorithm $\mathcal{A}^{*}$, any $\tilde{\xi}$ with all positive elements and $\frac{C_p}{2}<\min_{i\in [d]}\tilde{\xi}_{i}$, there exists a PLB ($Z_{t} = \langle\xi_{t},A_{t}\rangle+\eta_{t}$) with parameters $(\xi_{1},\dots,\xi_{t},\dots)$ and action sets $(\mathcal{A}_{1},\dots,\mathcal{A}_{t},\dots)$ satisfying $\xi_{t}\in PB(\tilde{\xi},C_p),\forall t\in\mathbb{N}^{+}$ and a constant $C_{0}$ only dependent on $\tilde{\xi}$ such that \[
	\mathbb{E}(R^{PLB}_{T_{0}}(\mathcal{A}^{*}))\geq C_{0}C_pT_{0}, \forall T_{0}\in\mathbb{N}^{+}.
	\]
\end{proposition}

\begin{proof}
	Let $C = \frac{C_{p}}{2}$. Let $u,v$ be the smallest and second smallest element in $\tilde{\xi}$ and $j_{u},j_{v}$ be their indexes in $\tilde{\xi}$. Let $U$ be a vector with only one nonzero element $\frac{1}{u+C}$ with index $j_{u}$. Let $V$ be a vector with only one nonzero element $\frac{1}{v+C}$ with index $j_{v}$. Let $\mathcal{A}_{t} = \{U,V \},\forall t\in\mathbb{N}^{+}$. Namely, all action sets are the same with two action vectors $U$ and $V$. 
	
	Let $\xi_{u}\in\mathbb{R}^{d}$ satisfying $\xi_{ui} = \tilde{\xi}_{i},\forall i\neq j_{u},j_{v}$, $\xi_{uj_{u}} = u-C$, $\xi_{uj_{v}} = v+C$. Let $\xi_{v}\in\mathbb{R}^{d}$ satisfying $\xi_{vi} = \tilde{\xi}_{i},\forall i\neq j_{u},j_{v}$, $\xi_{vj_{u}} = u+C$, $\xi_{vj_{v}} = v-C$. Then we have $\xi_{u},\xi_{v}\in PB(\tilde{\xi},C_p)$. Moreover, it holds that $$\langle \xi_{u},U\rangle = \frac{u-C}{u+C}\leq 1 = \langle\xi_{v},U\rangle, \langle \xi_{v},V\rangle = \frac{v-C}{v+C}\leq 1 = \langle\xi_{u},V\rangle.$$
	Now we construct the sequence of $\xi_{t}$. For $t = 1$, $\mathcal{A}^{*}$ determines a distribution of $A_{1}$. 
	\begin{equation*}
	\begin{cases}
	\text{if }A_{1} = U \text{ with probability }\geq 0.5 \Rightarrow \text{ let } \xi_{1} = \xi_{u};\\
	\text{if }A_{1} = U \text{ with probability }<0.5 \Rightarrow \text{ let } \xi_{1} = \xi_{v}.\\
	\end{cases}
	\end{equation*}
	Then $\xi_{1}$ is determined by the nature of $\mathcal{A}^{*}$. First by the randomness of $\mathcal{A}^{*}$ selecting $U$ and $V$, there is distribution of $A_{1}$. According to $\xi_{1}$ and the subsequent randomness of $\eta_{1}$, there is a further distribution of $(A_{1},Z_{1})$. For any $(A_{1},Z_{1})$, $\mathcal{A}^{*}$ determines a probability distribution on $\mathcal{A}_{2}$ for selecting actions. Namely, there is a joint distribution of $(A_{1},Z_{1},A_{2})$. Thus there is a corresponding marginal distribution of $A_{2}$ on $\mathcal{A}_{2}$. Then for $t = 2$, 
	\begin{equation*}
	\begin{cases}
	\text{if }A_{2} = U \text{ with probability }\geq 0.5 \Rightarrow \text{ let } \xi_{2} = \xi_{u};\\
	\text{if }A_{2} = U \text{ with probability }<0.5 \Rightarrow \text{ let } \xi_{2} = \xi_{v}.\\
	\end{cases}
	\end{equation*}
	Suppose $\xi_{k}$ has been determined by the above algorithm based on the joint distribution of $(A_{1},Z_{1},\dots,A_{k})$. Then this $\xi_{k}$ and the corresponding $\eta_{k}$ further form a joint distribution of $(A_{1},Z_{1},\dots,A_{k},Z_{k})$. For each realization of $(A_{1},Z_{1},\dots,A_{k},Z_{k})$, $\mathcal{A}^{*}$ determines a distribution on $\mathcal{A}_{k+1}$ for selecting actions. Thus there is a joint distribution of $(A_{1},Z_{1},\dots,A_{k},Z_{k},A_{k+1})$ and a corresponding marginal distribution of $A_{k+1}$. Then for $t = k+1$, 
	\begin{equation*}
	\begin{cases}
	\text{if }A_{k+1} = U \text{ with probability }\geq 0.5 \Rightarrow \text{ let } \xi_{k+1} = \xi_{u};\\
	\text{if }A_{k+1} = U \text{ with probability }<0.5 \Rightarrow \text{ let } \xi_{k+1} = \xi_{v}.\\
	\end{cases}
	\end{equation*}
	By this appoach we can determine an infinite sequence $(\xi_{1},\xi_{2},\dots,\xi_{t},\dots)$. 
	
	Then we calculate the expected regret $\mathbb{E}(R_{T_{0}})$. Note that the expectation is taken with respect to the joint distribution of $(A_{1},Z_{1},\dots,A_{T_{0}},Z_{T_{0}})$ determined by $\mathcal{A}^{*},\{\eta_{t} \}_{t\in[T_{0}]}$ and our introduced mechanism of determining $\{\xi_{t} \}_{t\in[T_{0}]}$. Let $\mathbb{P}$ denote this distribution. We consider each $1\leq t\leq T_{0}$. 
	\begin{itemize}
		\item 	If $\xi_{t} = \xi_{u}$, this implies $\mathbb{P}(A_{t} = U)\geq 0.5$. Then the best action is $V$ since $$\langle \xi_{u},V\rangle =1\geq \frac{u-C}{u+C} = \langle \xi_{u},U\rangle.$$
		Thus the regret $r_{t}$ satisfies $$r_{t} = 1_{\{A_{t} = U \}}(\langle \xi_{u},V\rangle - \langle \xi_{u},U\rangle).$$
		Therefore, the expected regret $\mathbb{E}(r_{t})$ satisfies
		\[
		\mathbb{E}(r_{t}) = \mathbb{E}(1_{\{A_{t} = U \}}(\langle \xi_{u},V\rangle - \langle \xi_{u},U\rangle)) = \mathbb{P}(A_{t} = U)\frac{2C}{u+C} \geq \frac{C}{u+C}\geq \frac{C}{2u}.
		\]
		If $\xi_{t} = \xi_{v}$, this implies $\mathbb{P}(A_{t} = U)<0.5$. Then the best action is $U$ since $$\langle \xi_{v},U\rangle =1\geq \frac{v-C}{v+C} = \langle \xi_{v},V\rangle.$$
		Thus the regret $r_{t}$ satisfies $$r_{t} = 1_{\{A_{t} = V \}}(\langle \xi_{v},U\rangle - \langle \xi_{v},V\rangle).$$
		Therefore, the expected regret $\mathbb{E}(r_{t})$ satisfies
		\[
		\mathbb{E}(r_{t}) = \mathbb{E}(1_{\{A_{t} = V \}}(\langle \xi_{v},U\rangle - \langle \xi_{v},V\rangle)) = \mathbb{P}(A_{t} = V)\frac{2C}{v+C} \geq \frac{C}{v+C}\geq \frac{C}{2v}.
		\]
	\end{itemize}
	Since $v\geq u$ and $C = \frac{C_{p}}{2}$, $\mathbb{E}(r_{t})\geq C_{0}C_{p}$ for both cases where $C_{0} = \frac{1}{4v}$. Thus we have $\forall T_{0}\in \mathbb{N}^{+}$,
	\[
	\mathbb{E}(R_{T_{0}} ) =\mathbb{E}(\sum_{t=1}^{T_{0}}r_{t})= \sum_{t=1}^{T_{0}}\mathbb{E}(r_{t})\geq C_{0}C_{p}T_{0}.
	\]
\end{proof}

\setcounter{proposition}{3}

\begin{proposition}\label{supp_thm:1}
	Under Assumption 1, with probability at least $1-\delta$, applying Algorithm 3 on the single-episode pricing problem yields a discrete part regret $R_{T_{0},1}$ satisfying
	\[
	R_{T_{0},1}\leq 2\sqrt{2dT_{0}\beta_{T_{0}}^{*}\log(\frac{d\lambda+T_{0}p_{\max}^{2}}{d\lambda})}+  4p_{\max}L||\hat{\theta}-\theta_{0}||_{1}T_{0}+2dp_{\max}.
	\]
\end{proposition}

\begin{proof}
	Lemma \ref{supp_lem:1} shows our single-episode pricing problem can be formulated as a perturbed linear bandit. Furthermore, by Lemma \ref{supp_lem:2}, applying Algorithm 5 with $\beta_{t} = \beta^{*}_{t}$ to the perturbed linear bandit formulation of the pricing problem yields Algorithm 3 with a specific UCB construction by the price-action coupling $p_{t}=Q_{t}^{-1}(A_{t})$. Thus, in order to investigate Algorithm 3, we only need to evaluate the performance of Algorithm 5 with $\beta_{t} = \beta^{*}_{t}$ applied on the equivalent perturbed linear bandit. 
	
	Now we discuss how to formulate the equivalent perturbed linear bandit so that Lemma \ref{supp_lem:8} can be employed to obtain the regret bound. We scale the rewards, linear parameters and noises by $\frac{1}{p_{\max}}$ as 
	\[\tilde{\xi}_{t} = \frac{1}{p_{\max}}\xi_{t}, \tilde{Z}_{t} = \frac{1}{p_{\max}}Z_{t}, \tilde{\eta}_{t} = \frac{1}{p_{\max}}\eta_{t}.\]
	Then we obtain the transformed model
	\begin{equation}\label{supp_eq:1}
	\tilde{Z}_{t} = \langle \tilde{\xi}_{t},A_{t}\rangle + \tilde{\eta}_{t}
	\end{equation}
	with perturbation constant $\tilde{C}_{p} = \frac{2L||\hat{\theta}-\theta_{0}||_{1}}{p_{\max}}$. 
	
	We first check Conditions 1-4 on this transformed model. Since $0\leq \xi_{ti} = 1-F(m_{i}+x_{t}^{\top}\hat{\theta}-x_{t}^{\top}\theta_{0})\leq 1$, we have $||\tilde{\xi}_{t}||_{\infty} \leq \frac{1}{p_{\max}}$. Thus Condition 2 is satisfied with a constant $C_{1} = \frac{1}{p_{\max}}$. By construction of $\mathcal{S}_{t}$ and $\mathcal{A}_{t}$, $||a||_{0} = 1$ and $||a||_{1}\leq p_{\max}$ for any $t\in[T_{0}]$ and $a\in\mathcal{A}_{t}$. Thus $||a||_{2}\leq p_{\max}$ and Condition 3 is satisfied with $a_{\max} = p_{\max}$. For any $t\in [T_{0}]$ and $a\in \mathcal{A}_{t}$, we have $|\langle \tilde{\xi}_{t},a\rangle|\leq ||\tilde{\xi}_{t}||_{\infty}||a||_{1}\leq \frac{1}{p_{\max}}\cdot p_{\max} = 1$. Thus Condition 1 is satisified. Since $\eta_{t}$ is conditionally $p_{\max}$-subgaussian, $\tilde{\eta}_{t}$ is conditionally $1$-subgaussian. Thus Condition 4 is satisfied. 
	
	Then we prove that applying Algorithm 5 with $\beta_{t} = \beta_{t}^{*}$ on the original perturbed linear bandit is equivalent to applying it with $\beta_{t} = \tilde{\beta}_{t}=1\vee (C_{1}\sqrt{\lambda d}+\sqrt{2\log(\frac{1}{\delta})+d\log(\frac{d\lambda+(t-1)p_{\max}^{2}}{d\lambda}) })^{2}=\frac{1}{p^{2}_{\max}}\beta_{t}^{*}$ on the transformed one \ref{supp_eq:1}. Furthermore, their regrets admit a scaling relationship.
	
	Note that $V_{t}(\lambda) = \lambda I+\sum_{s=1}^{t}A_{s}A_{s}^{\top}$ and $\hat{\xi}_{t}=V_{t}(\lambda)^{-1}\sum_{s=1}^{t}A_{s}Z_{s}$. Denote $\mathcal{C}_{t}^{*} = \mathcal{C}_{t}(\beta_{t}^{*}) = \{\xi\in\mathbb{R}^{d}: ||\xi- \hat{\xi}_{t-1}||_{V_{t-1}(\lambda)}^{2}\leq \beta_{t}^{*} \}$
	
	Denote $\tilde{V}_{t}(\lambda)= \lambda I+\sum_{s=1}^{t}A_{s}A_{s}^{\top} = V_{t}(\lambda)$, $\tilde{\hat{\xi}}_{t} = \tilde{V}_{t}(\lambda)^{-1}\sum_{s=1}^{t}A_{s}\tilde{Z}_{s}= \frac{1}{p_{\max}}\hat{\xi}_{t}$ and $\tilde{\mathcal{C}}_{t}= \{\xi\in\mathbb{R}^{d}: ||\xi- \tilde{\hat{\xi}}_{t-1}||_{V_{t-1}(\lambda)}^{2}\leq \tilde{\beta}_{t} \}$. 
	
	Then $\zeta\in \mathcal{C}_{t}^{*}\Leftrightarrow ||\zeta - \hat{\xi}_{t-1}||_{V_{t-1}(\lambda)}^{2}\leq \beta_{t}^{*}\Leftrightarrow ||\frac{\zeta}{p_{\max}} - \tilde{\hat{\xi}}_{t-1}||_{\tilde{V}_{t-1}(\lambda)}^{2}\leq \tilde{\beta}_{t}\Leftrightarrow \frac{\zeta}{p_{\max}}\in \tilde{\mathcal{C}}_{t}$.
	
	Note that $\text{LinUCB}_{t}(a) = \max_{\xi\in \mathcal{C}_{t}^{*}}\langle \xi, a\rangle$. Denote $\tilde{\text{LinUCB}}_{t}(a) = \max_{\xi\in \tilde{\mathcal{C}}_{t}}\langle \xi, a\rangle$. Then $$\text{LinUCB}_{t}(a)  =\max_{\xi\in \mathcal{C}_{t}^{*}}\langle \xi, a\rangle=\max_{\xi\in \tilde{\mathcal{C}}_{t}}\langle p_{\max}\xi, a\rangle= p_{\max}\tilde{\text{LinUCB}}_{t}(a).$$
	
	Denote $A_{t} = \arg\max_{a\in\mathcal{A}_{t}}\text{LinUCB}_{t}(a)$ and $A_{t}^{*} = \arg\max_{a\in\mathcal{A}_{t}}\langle \xi_{t},a\rangle$. Similarly, denote $\tilde{A}_{t} = \arg\max_{a\in\mathcal{A}_{t}}\tilde{\text{LinUCB}}_{t}(a)$ and $\tilde{A}_{t}^{*} = \arg\max_{a\in\mathcal{A}_{t}}\langle \tilde{\xi}_{t},a\rangle$. Then we have 
	\[
	\tilde{A}_{t} = \arg\max_{a\in\mathcal{A}_{t}}\tilde{\text{LinUCB}}_{t}(a) =  \arg\max_{a\in\mathcal{A}_{t}}\frac{1}{p_{\max}}\text{LinUCB}_{t}(a) = A_{t},
	\]
	\[
	\tilde{A}^{*}_{t} = \arg\max_{a\in\mathcal{A}_{t}}\langle \tilde{\xi}_{t},a\rangle=  \arg\max_{a\in\mathcal{A}_{t}}\langle \frac{1}{p_{\max}}\xi_{t},a\rangle = A^{*}_{t}.
	\]
	
	Therefore, applying Algorithm 5 with $\beta_{t} = \tilde{\beta}_{t}$ on transformed data yields the same actions as selected by Algorithm 5 with $\beta_{t} = \tilde{\beta}_{t}$ applied on the original data. Furthermore, the best actions also match at each time step. 
	
	Denote $R^{PLB}_{T_{0}}$ as the regret of applying Algorithm 5 with $\beta_{t} = \beta_{t}^{*}$ on the original perturbed linear bandit, i.e., $R^{PLB}_{T_{0}} = \sum_{t=1}^{T_{0}}\langle \xi_{t},A_{t}^{*}-A_{t}\rangle$. Similarly, denote $\tilde{R}^{PLB}_{T_{0}} = \sum_{t=1}^{T_{0}}\langle \tilde{\xi}_{t},\tilde{A}_{t}^{*}-\tilde{A}_{t}\rangle$, which can be viewed as the regret of applying Algorithm 5 with $\beta_{t} = \tilde{\beta}_{t}$ on the transformed perturbed linear bandit \ref{supp_eq:1}. Then they admit the relationship
	\[
	R^{PLB}_{T_{0}} = \sum_{t=1}^{T_{0}}\langle \xi_{t},A_{t}^{*}-A_{t}\rangle = \sum_{t=1}^{T_{0}}\langle p_{\max}\tilde{\xi}_{t},\tilde{A}_{t}^{*}-\tilde{A}_{t}\rangle = p_{\max}\tilde{R}^{PLB}_{T_{0}}.
	\]
	By Lemma \ref{supp_lem:8}, with probability at least $1-\delta$, 
	\[
	\tilde{R}^{PLB}_{T_{0}}\leq 2\sqrt{2dT_{0}\tilde{\beta}_{T_{0}}\log(\frac{d\lambda+T_{0}p_{\max}^{2}}{d\lambda})} +  2T_{0}\tilde{C}_{p}p_{\max}+2d
	\]
	where $\tilde{C}_{p}  =  \frac{2L||\theta_{0}-\hat{\theta}||_{1}}{p_{\max}}$. Therefore, we have with probability at least $1-\delta$, 
	$$
	R^{PLB}_{T_{0}} =p_{\max}\tilde{R}^{PLB}_{T_{0}} \leq 2\sqrt{2dT_{0}\beta^{*}_{T_{0}}\log(\frac{d\lambda+T_{0}p_{\max}^{2}}{d\lambda})}+  4T_{0}L||\theta_{0}-\hat{\theta}||_{1}p_{\max}+2dp_{\max}.
	$$
	
	Next, we investigate the discrete part regret of the pricing policy yielded by Algorithm 5 with $\beta_{t} = \beta_{t}^{*}$ through the price-action coupling. Note that this policy is just Algorithm 3 by Lemma \ref{supp_lem:2}. Denote $\tilde{p}^{*}_{t} = \arg\max_{p\in \mathcal{S}_{t}}p(1-F(p-x_{t}^{\top}\theta_{0}))$ as the price that yields the largest expected reward for prices in $\mathcal{S}_{t}$. Note that for any $p\in \mathcal{S}_{t}$, there exists $j\in[d]$ such that $p = m_{j}+x_{t}^{\top}\hat{\theta}$. Remember that then $Q_{t}(p)$ is a $d$-dimensional vector with only one nonzero element $p=m_{j}+x_{t}^{\top}\hat{\theta}$ with index $j$. Then we have 
	\begin{equation*}
	\begin{aligned}
	p(1-F(p-x_{t}^{\top}\theta_{0}))& = (m_{j}+x_{t}^{\top}\hat{\theta})(1-F(m_{j}+x_{t}^{\top}\hat{\theta}-x_{t}^{\top}\theta_{0}))\\
	& = (Q_{t}(p))_{j}(\xi_{t})_{j} = \langle \xi_{t},Q_{t}(p)\rangle.
	\end{aligned}
	\end{equation*}
	Since $\mathcal{A}_{t} = \{Q_{t}(p):p\in\mathcal{S}_{t}\}$ and $Q_{t}$ is a one-to-one mapping from $S_{t}$ to $\mathcal{A}_{t}$, we have
	\begin{equation*}
	\begin{aligned}
	\tilde{p}^{*}_{t} &= \arg\max_{p\in \mathcal{S}_{t}}p(1-F(p-x_{t}^{\top}\theta_{0}))\\
	& = \arg\max_{p\in \mathcal{S}_{t}}\langle \xi_{t},Q_{t}(p)\rangle = Q_{t}^{-1}(\arg\max_{A\in\mathcal{A}_{t}}\langle \xi_{t},A\rangle) = Q_{t}^{-1}(A_{t}^{*}).\\
	\end{aligned}
	\end{equation*}
	Therefore, we have
	\begin{equation*}
	\begin{aligned}
	r_{t,1} &=\tilde{p}^{*}_{t}(1-F(\tilde{p}^{*}_{t}-x_{t}^{\top}\theta_{0}))-p_{t}(1-F(p_{t}-x_{t}^{\top}\theta_{0}))\\
	& = \langle \xi_{t},Q_{t}(\tilde{p}^{*}_{t})\rangle -\langle \xi_{t},Q_{t}(p_{t})\rangle\\
	& = \langle \xi_{t},A^{*}_{t}\rangle -\langle \xi_{t},A_{t}\rangle = \langle \xi_{t},A_{t}^{*}-A_{t}\rangle = r^{PLB}_{t}.\\
	\end{aligned}
	\end{equation*}
	Namely, the discrete part regret $R_{T_{0},1}$ of Algorithm 3 with respect to the pricing problem admit 
	\[R_{T_{0},1} = \sum_{t=1}^{T_{0}}r_{t,1} = \sum_{t=1}^{T_{0}}r^{PLB}_{t} = R^{PLB}_{T_{0}}.\]
	Namely, it equals to the regret of Algorithm 5 with $\beta_{t} = \beta_{t}^{*}$ with respect to the equivalent perturbed linear bandit. Therefore, we have with probability at least $1-\delta$, 
	$$
	R_{T_{0},1} =R^{PLB}_{T_{0}} \leq 2\sqrt{2dT_{0}\beta_{T_{0}}^{*}\log(\frac{d\lambda+T_{0}p_{\max}^{2}}{d\lambda})}+  4T_{0}L||\theta_{0}-\hat{\theta}||_{1}p_{\max}+2dp_{\max}.
	$$
\end{proof}

\renewcommand{\thelemma}{S\arabic{lemma}}
\setcounter{lemma}{5}

\begin{lemma}\label{supp_lem:9}
	For $\beta_{T_{0}}^{*} =p^{2}_{\max}(1\vee (C_{1}\sqrt{\lambda d}+\sqrt{2\log(T_{0})+d\log(\frac{d\lambda+(T_{0}-1)p_{\max}^{2}}{d\lambda}) })^{2})$, there exists constants $C_{1}^{'},C_{2}^{'}>0$ such that for any $T_{0}\geq 1$, \[
	\sqrt{\beta_{T_{0}}^{*}\log(\frac{d\lambda+T_{0}p_{\max}^{2}}{d\lambda})}\leq C_{1}^{'}\sqrt{d}\log(C_{2}^{'}T_{0}).
	\]
\end{lemma}

\begin{proof}
	Denote $A_{1} = p_{\max},A_{2} = p_{\max}C_{1}\sqrt{\lambda d},A_{3} = p_{\max}\sqrt{2\log(T_{0})+d\log(\frac{d\lambda+(T_{0}-1)p_{\max}^{2}}{d\lambda}) }$. Then $\sqrt{\beta_{T_{0}}^{*}} = A_{1}\vee(A_{2}+A_{3})$. Let $C_{3}^{'} = \max\{1+\frac{p_{\max}^{2}}{\lambda},3 \}$, then
	\[
	1+\frac{p_{\max}^{2}}{\lambda}T_{0}\leq (1+\frac{p_{\max}^{2}}{\lambda})T_{0}\leq C_{3}^{'}T_{0}
	\Rightarrow \log(1+\frac{p_{\max}^{2}}{\lambda}T_{0})\leq \log(C_{3}^{'}T_{0}).\]
	Since $C_{3}^{'}\geq 3$, $\log(C_{3}^{'}T_{0})\geq 1$. 
	
	Therefore, we have \begin{equation}\label{supp_eq:2}
	\sqrt{\log(\frac{d\lambda+T_{0}p_{\max}^{2}}{d\lambda})}\leq \sqrt{\log(1+\frac{p_{\max}^{2}}{\lambda}T_{0})}\leq \sqrt{\log(C_{3}^{'}T_{0})}\leq \log(C_{3}^{'}T_{0}).
	\end{equation} 
	
	On the other hand, 
	\begin{equation}\label{supp_eq:3}
	\begin{aligned}
	\sqrt{2\log(T_{0})+d\log(\frac{d\lambda+(T_{0}-1)p_{\max}^{2}}{d\lambda}) }&\leq \sqrt{d}\sqrt{2\log(T_{0})+\log(\frac{d\lambda+(T_{0}-1)p_{\max}^{2}}{d\lambda}) }\\
	&\leq \sqrt{d}\sqrt{2\log(T_{0})+\log(1 + \frac{p_{\max}^{2}}{d\lambda}T_{0}) }\\
	&\leq \sqrt{d}\sqrt{2\log(T_{0})+\log(1 + \frac{p_{\max}^{2}}{\lambda}T_{0}) }\\
	(\text{By } \ref{supp_eq:2})&\leq \sqrt{d}\sqrt{2\log(T_{0})+\log(C_{3}^{'}T_{0}) }\\
	&\leq \sqrt{d}\sqrt{3\log(C_{3}^{'}T_{0}) }\\
	\end{aligned}
	\end{equation}
	
	Therefore, we have $$A_{1}\sqrt{\log(\frac{d\lambda+T_{0}p_{\max}^{2}}{d\lambda})}\leq p_{\max}\log(C_{3}^{'}T_{0})\leq C_{4}^{'}\log(C_{3}^{'}T_{0})\sqrt{d} \text{ where }C_{4}^{'} = p_{\max}, $$
	$$A_{2}\sqrt{\log(\frac{d\lambda+T_{0}p_{\max}^{2}}{d\lambda})}\leq p_{\max}C_{1}\sqrt{\lambda d}\log(C_{3}^{'}T_{0})\leq C_{5}^{'}\log(C_{3}^{'}T_{0})\sqrt{d} \text{ where }C_{5}^{'} = p_{\max}C_{1}\sqrt{\lambda}. $$
	\begin{equation*}
	\begin{aligned}
	A_{3}\sqrt{\log(\frac{d\lambda+T_{0}p_{\max}^{2}}{d\lambda})}&=p_{\max} \sqrt{2\log(T_{0})+\log(\frac{d\lambda+(T_{0}-1)p_{\max}^{2}}{d\lambda}) }\sqrt{\log(\frac{d\lambda+T_{0}p_{\max}^{2}}{d\lambda})}\\
	(\text{By }\ref{supp_eq:2}\text{ and }\ref{supp_eq:3})&\leq p_{\max}\sqrt{d}\sqrt{3\log(C_{3}^{'}T_{0}) }\sqrt{\log(C_{3}^{'}T_{0})}\\
	& = \sqrt{3}p_{\max}\sqrt{d}\log(C_{3}^{'}T_{0})\leq C_{6}^{'}\log(C_{3}^{'}T_{0})\sqrt{d} \text{ where }C_{6}^{'} = \sqrt{3}p_{\max}.
	\end{aligned}
	\end{equation*}
	
	Let $C_{7}^{'} = \max\{C_{4}^{'},C_{5}^{'}+C_{6}^{'} \}$, then 
	\begin{equation*}
	\begin{aligned}
	&\sqrt{\beta_{T_{0}}^{*}\log(\frac{d\lambda+T_{0}p_{\max}^{2}}{d\lambda})} = \sqrt{\beta_{T_{0}}^{*}}\sqrt{\log(\frac{d\lambda+T_{0}p_{\max}^{2}}{d\lambda})}\\
	= &(A_{1}\vee(A_{2}+A_{3}))\sqrt{\log(\frac{d\lambda+T_{0}p_{\max}^{2}}{d\lambda})}\\
	=& (A_{1}\sqrt{\log(\frac{d\lambda+T_{0}p_{\max}^{2}}{d\lambda})})\vee(A_{2}\sqrt{\log(\frac{d\lambda+T_{0}p_{\max}^{2}}{d\lambda})}+A_{3}\sqrt{\log(\frac{d\lambda+T_{0}p_{\max}^{2}}{d\lambda})})\\
	\leq &(C_{4}^{'}\log(C_{3}^{'}T_{0})\sqrt{d}) \vee(C_{5}^{'}\log(C_{3}^{'}T_{0})\sqrt{d} +C_{6}^{'}\log(C_{3}^{'}T_{0})\sqrt{d} )\\
	= & (C_{4}^{'}\vee (C_{5}^{'}+C_{6}^{'}))\log(C_{3}^{'}T_{0})\sqrt{d}= C_{7}^{'}\log(C_{3}^{'}T_{0})\sqrt{d} \\
	\end{aligned}
	\end{equation*}
	Let $C_{1}^{'} = C_{7}^{'},C_{2}^{'} = C_{3}^{'}$, we obtain $\sqrt{\beta_{T_{0}}^{*}\log(\frac{d\lambda+T_{0}p_{\max}^{2}}{d\lambda})}\leq C_{1}^{'}\sqrt{d}\log(C_{2}^{'}T_{0})$.
\end{proof}

\setcounter{proposition}{0}

\begin{proposition}\label{supp_prop:2}
	Assumption 2 holds if $F^{''}(\cdot)$ is bounded on $[-||\theta_{0}||_{1},p_{\max}+||\theta_{0}||_{1}]$.
\end{proposition}

\begin{proof}
	Recall that Assumption 2 says there exists a constant $C$ such that for any $q=x^{\top}\theta_{0}$ and $x\in\mathcal{X}$, we have $f_{q}(p^{*}(x))-f_{q}(p)\leq C(p^{*}(x)-p)^{2}, \forall p\in[0,p_{\max}]$. 
	
	Remember that $f_{q}(p) = p(1-F(p-q))$ and we assume a known upper bound optimal prices, i.e., $p^{*}(x) \in(0,p_{\max})$. Thus for any $x\in\mathcal{X}$, $p^{*}(x)$ is a maximizing point of $f_{x^{\top}\theta_{0}}(p) = p(1-F(p-x^{\top}\theta_{0}))$ on $(0,p_{\max})$. By Assumption 2, $F$ is twice differentiable on $[-||\theta_{0}||_{1},p_{\max}+||\theta_{0}||_{1}]$. Thus $f_{x^{\top}\theta_{0}}(p)$ is twice differentiable on $[0,p_{\max}]$ as $|x^{\top}\theta_{0}|\leq ||x||_{\infty}||\theta_{0}||_{1}\leq ||\theta_{0}||_{1}$ for $x\in\mathcal{X}$. Thus $f_{x^{\top}\theta_{0}}^{'}(p^{*}(x))$ exists and equals to $0$. By applying Taylor's Theorem with Lagrange Remainder at $p^{*}(x)$, we obtain for any $p\in [0,p_{\max}]$, there exists $\zeta\in (\min(p,p^{*}(x)),\max(p,p^{*}(x)))\subseteq [0,p_{\max}]$ such that
	\[
	f_{x^{\top}\theta_{0}}(p) = f_{x^{\top}\theta_{0}}(p^{*}(x)) + f_{x^{\top}\theta_{0}}^{'}(p^{*}(x))(p^{*}(x)-p)+\frac{f_{x^{\top}\theta_{0}}^{''}(\zeta)}{2}(p^{*}(x)-p)^{2}.
	\]
	Since $f_{x^{\top}\theta_{0}}^{'}(p^{*}(x))=0$, we obtain
	\[
	f_{x^{\top}\theta_{0}}(p^{*}(x))-f_{x^{\top}\theta_{0}}(p) = \frac{f_{x^{\top}\theta_{0}}^{''}(\zeta)}{2}(p^{*}(x)-p)^{2}.
	\]
	By calculation, we have \[
	f_{q}^{''}(p) = -2F^{'}(p-q)-pF^{''}(p-q).
	\]
	Since $F^{''}(x)$ exists on $[-||\theta_{0}||_{1},p_{\max}+||\theta_{0}||_{1}]$, $F^{'}(x)$ are continuous on $[-||\theta_{0}||_{1},p_{\max}+||\theta_{0}||_{1}]$. Thus they are both bounded on this range and there exists a constant $C>0$ only dependent on $F$ such that \[
	\max_{p\in [0,p_{\max}],q\in [-||\theta_{0}||_{1},||\theta_{0}||_{1}]} f_{q}^{''}(p)\leq 2C.
	\] 
	Since $-||\theta||_{1} \leq -||x||_{\infty}||\theta_{0}||_{1} \leq x^{\top}\theta_{0}\leq ||x||_{\infty}||\theta_{0}||_{1} \leq ||\theta||_{1}$ and $\zeta\in[0,p_{\max}]$, we have $f_{x^{\top}\theta_{0}}^{''}(\zeta)\leq 2C$. Thus we obtain
	\[
	f_{x^{\top}\theta_{0}}(p^{*}(x))-f_{x^{\top}\theta_{0}}(p)=\frac{f_{x^{\top}\theta_{0}}^{''}(\zeta)}{2}(p^{*}(x)-p)^{2}\leq C(p^{*}(x)-p)^{2}.
	\]
	Note that this holds for any $p\in[0,p_{\max}]$. Thus we complete our proof. 
\end{proof}

\setcounter{proposition}{2}

\begin{proposition}\label{supp_thm:2}
	Under Assumptions 1 -- 2, with probability at least $1-\delta$, applying Algorithm 3 on the single-episode pricing problem yields the total regret $R_{T_{0}}$ satisfying
	\[
	R_{T_{0}}\leq 2\sqrt{2dT_{0}\beta_{T_{0}}^{*}\log(\frac{d\lambda+T_{0}p_{\max}^{2}}{d\lambda})} +  4p_{\max}L||\hat{\theta}-\theta_{0}||_{1}T_{0} + C\frac{T_{0}}{d^{2}}+2dp_{\max}.
	\]
	Moreover, by setting $\delta = \frac{1}{T_{0}},d = C\lceil T_{0}^{1/6}\rceil$ and taking the expectation, we have $\mathbb{E}(R_{T_{0}}) = \tilde{O}(T_{0}^{2/3}) + 4p_{\max}L||\hat{\theta}-\theta_{0}||_{1}T_{0}$.
\end{proposition}

\begin{proof}
	In Proposition \ref{supp_thm:1}, we have proved a high probability upper bound for the discrete part regret. Thus we focus on the continuous part regret $R_{T_{0},2}$ which admits the form
	\begin{equation*}
	\begin{aligned}
	R_{T_{0},2} &= \sum_{t=1}^{T_{0}}(p^{*}_{t}(1-F(p^{*}_{t}-x_{t}^{\top}\theta_{0}))-\tilde{p}^{*}_{t}(1-F(\tilde{p}^{*}_{t}-x_{t}^{\top}\theta_{0})))\\
	& = \sum_{t=1}^{T_{0}}(f_{x_{t}^{\top}\theta_{0}}(p_{t}^{*})-f_{x_{t}^{\top}\theta_{0}}(\tilde{p}_{t}^{*})).
	\end{aligned}
	\end{equation*}
	
	By our discretization approach, $\{m_{i}+x_{t}^{\top}\hat{\theta} \}_{i\in[d]}$ are a sequence of points with a special pattern that any two consecutive points have a difference $$|(m_{i+1}+x_{t}^{\top}\hat{\theta})-(m_{i}+x_{t}^{\top}\hat{\theta})|=\frac{|G(\hat{\theta})|}{d}.$$ Moreover, the first point satisfies $m_{1}+x_{t}^{\top}\hat{\theta}\leq \frac{1}{2}\cdot\frac{|G(\hat{\theta})|}{d}$ while the last points satisfies $m_{d}+x_{t}^{\top}\hat{\theta}\geq p_{\max}-\frac{1}{2}\cdot\frac{|G(\hat{\theta})|}{d}$. Since $\mathcal{S}_{t} = \{m_{j}+x_{t}^{\top}\hat{\theta}|j\in [d],m_{j}+x_{t}^{\top}\hat{\theta}\in(0,p_{\max})\}$ and $p_{t}^{*}\in(0,p_{\max})$, there must be some price $\dot{p}_{t}\in\mathcal{S}_{t}$ whose distance with $p_{t}^{*}$ is less than $\frac{|G(\hat{\theta})|}{d}$, i.e., $|p_{t}^{*}-\dot{p}_{t}|\leq \frac{|G(\hat{\theta})|}{d}$. Thus by Assumption 2, there exists $C_{2}>0$ only dependent on $F$ such that
	\[
	f_{x_{t}^{\top}\theta_{0}}(p_{t}^{*})-f_{x_{t}^{\top}\theta_{0}}(\dot{p}_{t})\leq C_{2}(p_{t}^{*}-\dot{p}_{t})^{2}\leq \frac{C_{2}C_{3}^{2}}{d^{2}}
	\]
	where $C_{3} = p_{\max}+2W\geq p_{\max}+2||\hat{\theta}||_{1} = |G(\hat{\theta})|$. The last inequality follows from the fact that $\hat{\theta}$ represents the input $\hat{\theta}_{k-1}$ for the $k$-th episode, which is a projection on $\Theta = \{\theta\in\mathbb{R}^{d_{0}}:||\theta||_{1}\leq W \}$.
	
	Since $\tilde{p}_{t}^{*}$ yields the maximum reward for prices in $\mathcal{S}_{t}$, we have $f_{x_{t}^{\top}\theta_{0}}(\tilde{p}_{t}^{*})\geq f_{x_{t}^{\top}\theta_{0}}(\dot{p}_{t})$. Thus we further derive 
	\[
	f_{x_{t}^{\top}\theta_{0}}(p_{t}^{*})-f_{x_{t}^{\top}\theta_{0}}(\tilde{p}^{*}_{t})\leq f_{x_{t}^{\top}\theta_{0}}(p_{t}^{*})-f_{x_{t}^{\top}\theta_{0}}(\dot{p}_{t})\leq \frac{C_{2}C_{3}^{2}}{d^{2}} = \frac{C}{d^{2}}
	\]
	where $C = C_{2}C_{3}^{2}$ which depends on $F$ and $W$. Therefore, the continous part regret $R_{T_{0},2}$ can be bounded as
	\[ 
	R_{T_{0},2} = \sum_{t=1}^{T_{0}}(f_{x_{t}^{\top}\theta_{0}}(p_{t}^{*})-f_{x_{t}^{\top}\theta_{0}}(\tilde{p}^{*}_{t}))\leq C\frac{T_{0}}{d^{2}}. 
	\]
	
	By Proposition \ref{supp_thm:1}, the discrete part regret $R_{T_{0},1}$ satisfies
	\[
	R_{T_{0},1}\leq 2\sqrt{2dT_{0}\beta_{T_{0}}^{*}\log(\frac{d\lambda+T_{0}p_{\max}^{2}}{d\lambda})} +  4T_{0}L||\theta_{0}-\hat{\theta}||_{1}p_{\max}+2dp_{\max}
	\]
	with probability at least $1-\delta$. Thus the overall regret $R_{T_{0}} =R_{T_{0},1} + R_{T_{0},2}$ satisfies 
	\[
	R_{T_{0}}\leq 2\sqrt{2dT_{0}\beta_{T_{0}}^{*}\log(\frac{d\lambda+T_{0}p_{\max}^{2}}{d\lambda})}+  4T_{0}L||\theta_{0}-\hat{\theta}||_{1}p_{\max} + C\frac{T_{0}}{d^{2}}+2dp_{\max}
	\]
	with probability at least $1-\delta$.
	
	Let $\delta = \frac{1}{T_{0}}$, then $\beta_{T_{0}}^{*} =p^{2}_{\max}(1\vee (C_{1}\sqrt{\lambda d}+\sqrt{2\log(T_{0})+d\log(\frac{d\lambda+(T_{0}-1)p_{\max}^{2}}{d\lambda}) })^{2})$. By Lemma \ref{supp_lem:9}, there exists constants $C_{1}^{'},C_{2}^{'}>0$ such that $\sqrt{\beta_{T_{0}}^{*}\log(\frac{d\lambda+T_{0}p_{\max}^{2}}{d\lambda})}\leq C_{1}^{'}\sqrt{d}\log(C_{2}^{'}T_{0})$. Thus with probability at least $1-\delta = 1-\frac{1}{T_{0}}$, 
	\begin{equation*}
	\begin{aligned}
	R_{T_{0}}&\leq 2\sqrt{2dT_{0}\beta_{T_{0}}\log(\frac{d\lambda+T_{0}p_{\max}^{2}}{d\lambda})} +  4T_{0}L||\theta_{0}-\hat{\theta}||_{1}p_{\max} + C\frac{T_{0}}{d^{2}}+2dp_{\max}\\
	&\leq 2\sqrt{2}C_{1}^{'}d\log(C_{2}^{'}T_{0})\sqrt{T_{0}}+4T_{0}L||\theta_{0}-\hat{\theta}||_{1}p_{\max} + C\frac{T_{0}}{d^{2}}+2dp_{\max}.\\
	\end{aligned}
	\end{equation*}
	Since $r_{t}\leq p_{\max}$, we have $R_{T_{0}} = \sum_{t=1}^{T_{0}}r_{t} \leq p_{\max}T_{0}$. Therefore, we obtain 
	\begin{equation*}
	\begin{aligned}
	\mathbb{E}(R_{T_{0}})&\leq (1-\frac{1}{T_{0}})(2\sqrt{2}C_{1}^{'}d\log(C_{2}^{'}T_{0})\sqrt{T_{0}}+4T_{0}L||\theta_{0}-\hat{\theta}||_{1}p_{\max} + C\frac{T_{0}}{d^{2}})\\
	&+2dp_{\max}+p_{\max}T_{0}\frac{1}{T_{0}}\\
	&\leq 2\sqrt{2}C_{1}^{'}d\log(C_{2}^{'}T_{0})\sqrt{T_{0}}+4T_{0}L||\theta_{0}-\hat{\theta}||_{1}p_{\max} + C\frac{T_{0}}{d^{2}}+(2d+1)p_{\max} \\
	&\leq 2\sqrt{2}C_{1}^{'}d\log(C_{2}^{'}T_{0})\sqrt{T_{0}}+4T_{0}L||\theta_{0}-\hat{\theta}||_{1}p_{\max} + C\frac{T_{0}}{d^{2}}+3dp_{\max}\\
	\end{aligned}
	\end{equation*}
	Since $\lceil T_{0}^{1/6}\rceil\leq 2T_{0}^{1/6}$, $d = C^{'}\lceil T_{0}^{1/6}\rceil$ satisfies $C^{'}T_{0}^{1/6}\leq d\leq 2C^{'}T_{0}^{1/6}$. Thus we obtain
	\begin{equation*}
	\begin{aligned}
	\mathbb{E}(R_{T_{0}})&\leq 2\sqrt{2}C_{1}^{'}d\log(C_{2}^{'}T_{0})\sqrt{T_{0}}+4T_{0}L||\theta_{0}-\hat{\theta}||_{1}p_{\max}+C\frac{T_{0}}{d^{2}}+3dp_{\max}\\
	&\leq 4\sqrt{2}C^{'}C_{1}^{'}\log(C_{2}^{'}T_{0})T_{0}^{2/3}+\frac{C}{(C^{'})^{2}}T_{0}^{2/3}+6p_{\max}C^{'}T_{0}^{1/6}+4L||\theta_{0}-\hat{\theta}||_{1}p_{\max}T_{0}\\
	& \leq C_{3}^{'}\log(C_{2}^{'}T_{0})T_{0}^{2/3}+4L||\theta_{0}-\hat{\theta}||_{1}p_{\max}T_{0}\\
	\end{aligned}
	\end{equation*}
	where $C_{3}^{'} = 4\sqrt{2}C^{'}C_{1}^{'}+\frac{C}{(C^{'})^{2}}+6p_{\max}C^{'}$. Thus we obtain
	\[
	\mathbb{E}(R_{T_{0}}) = \tilde{O}(T_{0}^{2/3}) + 4T_{0}L||\theta_{0}-\hat{\theta}||_{1}p_{\max}.
	\]
\end{proof}

\begin{theorem}\label{supp_thm:3}
	Under Assumptions 1 -- 2, the DIP policy yields the expected regret
	\[
	\mathbb{E}(R_{T}) =  \tilde{O}(T^{2/3}) +4p_{\max}L\sum_{k=2}^{n}2^{k-2}\ell_{2}\mathbb{E}||\hat{\theta}_{k-1}-\theta_{0}||_{1}.
	\]
\end{theorem}

\begin{proof}
	Note that the lengths of the first two episodes are fixed as $\ell_{1}=\alpha_{1}$ and $\ell_{2} = \alpha_{2}$ and the regret order is with respect to $T\to\infty$. Given $T$, the doubling trick yields the number of episodes $n = n(T,\alpha_{1},\alpha_{2})$ and the lengths of the $k$-th episode $\ell_{k} = \ell_{k}(T,\alpha_{1},\alpha_{2})$.We then find a mathematical formulation of these functions $n(T,\alpha_{1},\alpha_{2})$ and $\ell_{k}(T,\alpha_{1},\alpha_{2})$ for $k\in[n]$. 
	
	First we find a relationship between $n$ and $T$. It is obvious that
	\[
	\alpha_{1}+\alpha_{2}(1+2+\dots+2^{n-3})< T\leq \alpha_{1}+\alpha_{2}(1+2+\dots+2^{n-2}).
	\]
	Therefore, we have 
	\begin{equation}\label{supp_eq:4}
	\log_{2}(\frac{T-\alpha_{1}}{\alpha_{2}}+1)+1\leq n<\log_{2}(\frac{T-\alpha_{1}}{\alpha_{2}}+1)+2.
	\end{equation}
	Thus $n = \lceil \log_{2}(\frac{T-\alpha_{1}}{\alpha_{2}}+1)+1\rceil = \lceil\log_{2}(\frac{T-\alpha_{1}}{\alpha_{2}}+1)\rceil + 1=n(T,\alpha_{1},\alpha_{2})$. On the other hand, we have $\ell_{1} = \alpha_{1} = \ell_{1}(T,\alpha_{1},\alpha_{2})$. Moreover, $\ell_{k} = 2^{k-2}\ell_{2} = 2^{k-2}\alpha_{2} = \ell_{k}(T,\alpha_{1},\alpha_{2})$ for $2\leq k\leq n-1=\lceil\log_{2}(\frac{T-\alpha_{1}}{\alpha_{2}}+1)\rceil$. Finally, $\ell_{n} = T-\alpha_{1}-\alpha_{2}(2^{\lceil\log_{2}(\frac{T-\alpha_{1}}{\alpha_{2}}+1)\rceil-1}-1) = \ell_{n}(T,\alpha_{1},\alpha_{2})$.
	
	Denote the regret in episode $k$ as $R_{k}^{'}$. Then $R_{T} = \sum_{k=1}^{n}R^{'}_{k}$. We price randomly for the first episode. Thus we obtain
	\[
	R^{'}_{1}\leq p_{\max}\ell_{1}\text{ and }\mathbb{E}(R^{'}_{1})\leq p_{\max}\ell_{1}.
	\]
	For the regret $R_{k}^{'}$ in episode $2\leq k\leq n$, we have by Proposition \ref{supp_thm:2}, \[
	\mathbb{E}(R_{k}^{'}|\hat{\theta}_{k-1})\leq C_{3}^{'}\log(C_{2}^{'}2^{k-2}\ell_{2})(2^{k-2}\ell_{2})^{2/3}+4L||\theta_{0}-\hat{\theta}_{k-1}||_{1}p_{\max}2^{k-2}\ell_{2}.
	\]
	Therefore, we obtain
	\[
	\mathbb{E}(R_{k}^{'})\leq C_{3}^{'}\log(C_{2}^{'}2^{k-2}\ell_{2})(2^{k-2}\ell_{2})^{2/3}+4p_{\max}L\mathbb{E}(||\theta_{0}-\hat{\theta}_{k-1}||_{1})2^{k-2}\ell_{2}.
	\]
	By \ref{supp_eq:4}, we have $2^{n-1}\leq 2(\frac{T-\alpha_{1}}{\alpha_{2}}+1)\leq 2T$. Thus we obtain
	\[
	\sum_{k=2}^{n}(2^{k-2}\ell_{2})^{2/3}=\ell_{2}^{2/3}\frac{(2^{2/3})^{n-1}-1}{2^{2/3}-1} \leq C_{4}T^{2/3} \text{ where }C_{4} = \frac{2^{2/3}\ell_{2}^{2/3}}{2^{2/3}-1}. 
	\]
	Since we assume $T>\alpha_{1}+\alpha_{2}$, we have $n\geq 3$. Thus $2^{n-3}\alpha_{2} = \ell_{n-1}\leq T$ and we obtain $2^{k-2}\ell_{2} \leq 2^{n-2}\ell_{2} = 2^{n-2}\alpha_{2}\leq 2T$ for $2\leq k\leq n$. Therefore, we obtain
	\begin{equation*}
	\begin{aligned}
	\mathbb{E}(R_{T}) &=\mathbb{E}(\sum_{k=1}^{n}R_{k}^{'}) = \sum_{k=1}^{n}\mathbb{E}(R_{k}^{'})\\
	&\leq \ell_{1}p_{\max}+\sum_{k=2}^{n}4p_{\max}L\mathbb{E}(||\theta_{0}-\hat{\theta}_{k-1}||_{1})2^{k-2}\ell_{2}+\sum_{k=2}^{n}C_{3}^{'}\log(C_{2}^{'}2^{k-2}\ell_{2})(2^{k-2}\ell_{2})^{2/3}\\ 
	&\leq \ell_{1}p_{\max}+\sum_{k=2}^{n}4p_{\max}L\mathbb{E}(||\theta_{0}-\hat{\theta}_{k-1}||_{1})2^{k-2}\ell_{2}+\sum_{k=2}^{n}C_{3}^{'}\log(2C_{2}^{'}T)(2^{k-2}\ell_{2})^{2/3}\\ 
	&\leq \ell_{1}p_{\max}+\sum_{k=2}^{n}4p_{\max}L\mathbb{E}(||\theta_{0}-\hat{\theta}_{k-1}||_{1})2^{k-2}\ell_{2}+C_{3}^{'}C_{4}\log(2C_{2}^{'}T)T^{2/3}\\ 
	&=\tilde{O}(T^{2/3}) + 4p_{\max}L\sum_{k=2}^{n}2^{k-2}\ell_{2}\mathbb{E}(||\theta_{0}-\hat{\theta}_{k-1}||_{1}).\\
	\end{aligned}
	\end{equation*}
\end{proof}

\renewcommand{\thelemma}{S\arabic{lemma}}
\setcounter{lemma}{6}




\def\thesection{B}

\section{Implementation Details}

In all problem instances, we consider low-dimensional and nonsparse coefficients. Thus we apply RMLP and RMLP-2 without regularization. We now specify some algorithm inputs.
\begin{itemize}
	\item Across all the simulation settings and the real-data applications, we use the same value of $\lambda = 0.1, p_{\max} = 30, C = 20$ and $W = 10^{4}$. The known bound $W$ of $||\theta_{0}||_{1}$ for DIP and RMLP, and of $||\theta_{0}||_{1}+\frac{1}{\sigma}$ for RMLP-2 is mainly for theoretical purpose. Thus we select a random and loose choice $W = 10^{4}$ to minimize its affect on the real performance of the three methods. 
	\item In the simulation settings, we use the same value of the first and the second episode lengths $\alpha_{1} = \alpha_{2}= 2^{11}$. In the real-data applications, the number of data samples is not as many as in the simulation settings and thus we shrink the first and second episode lengths to $\alpha_{1} = \alpha_{2} = 2^{10}$ or $2^{9}$ according to the horizon length. 
	\item For DIP, we use $\beta_{t}^{*} = s^{*}(1\vee (\frac{1}{p_{\max}}\sqrt{\lambda d}+\sqrt{2\log(T_{0})+d\log(\frac{d\lambda+(t-1)p_{\max}^{2}}{d\lambda}) })^{2})$ in the UCB construction of Inner Algorithm B with a scaling parameter $s^{*} = \frac{1}{40}$. Such a scaling is common in UCB practices.
	\item The specified distribution of RMLP is the logistic distribution with CDF $\frac{e^{x}}{e^{x}+1}$ and the specified class of distributions of RMLP-2 is the class of logistic distributions $\{\frac{e^{(x-\mu)/s}}{e^{(x-\mu)/s}+1}|\mu\in \mathbb{R},s\in \mathbb{R}^{+}\}$. We choose logistic distributions since the corresponding optimization problem of RMLP and RMLP-2 can be formulated into modified versions of the conventional logistic regression. 
\end{itemize}

\def\thesection{C}

\section{Detailed Discussion of Instability of RMLP-2}

We provide a more detailed discussion of why large estimation errors sometimes happen with RMLP-2 as indicated by the simulation Examples 7 -- 8. Note that RMLP-2 uses a mapping function to map the customer covariates to the set prices. Therefore, the set prices together with the customer covariates lie in a lower-dimensional subspace. It turns out that such subspace is likely to lie almost on only one side of the Bayes decision boundary when the distribution of $x_{t}^{\top}\theta_{0}$ is positively away from $0$ by the nature of the mapping function. This is exactly the case for Examples 7 -- 8. Then with such singular data structure and the wrongly specified distribution class, the RMLP-2 estimation process, which borrows a maximum likelihood idea, becomes unstable especially for higher dimensions. On the other hand, the distribution of $x_{t}^{\top}\theta_{0}$ is not positively away from 0 in Example 9, thus yielding relatively stable estimations.

\def\thesection{D}

\section{Detailed Estimation Procedure of $\theta_{0}$ and $F$ for Real Auto-Loan Data}

We take the dataset of the entire US as an example to illustrate how to estimate the noise CDF $F$ and $\theta_{0}$. The estimation of state-specific $F$ and $\theta_{0}$ are similar. Let $T_{0} = 208085$ be the number of data points for the US. 

The model we consider can be simply written as
\[
y_{t} \sim \text{Ber}(1-F(p_{t}-x_{t}^{\top}\theta_{0})).
\]
Note that
\[
F(p_{t}-x_{t}^{\top}\theta_{0}) = \tilde{F}(p_{t}-x_{t}^{\top}\theta_{0}-F^{-1}(\frac{1}{2})) = \tilde{F}(p_{t}-\tilde{x}_{t}^{\top}\tilde{\theta}_{0})
\]
where $\tilde{F}(\cdot) = F(\cdot+F^{-1}(\frac{1}{2}))$ and $\tilde{x}_{t} = (1,x_{t}^{\top})^{\top},\tilde{\theta}_{0} = (F^{-1}(\frac{1}{2}),\theta_{0}^{\top})^{\top}$. Then $\tilde{F}^{-1}(-\frac{1}{2}) = 0$. Thus without loss of generality, we can assume $F^{-1}(\frac{1}{2}) = 0$ by adding an all-$1$ column to $x_{t}$ and an intercept term to $\theta_{0}$. Then we have
\[
\begin{cases}
\mathbb{P}(y_{t} = 1)  > \frac{1}{2}, \text{ if }x_{t}^{\top}\theta_{0}-p_{t}>0;\\
\mathbb{P}(y_{t} = 1)  = \frac{1}{2}, \text{ if }x_{t}^{\top}\theta_{0}-p_{t}=0;\\
\mathbb{P}(y_{t} = 1)  < \frac{1}{2}, \text{ if }x_{t}^{\top}\theta_{0}-p_{t}<0.\\
\end{cases}
\]
Therefore, we can form a classification problem with responses $\{y_{t}\}_{t\in[T_{0}]}$ and covariates $\{(x_{t}^{\top},p_{t})^{\top}\}_{t\in[T_{0}]}$. It admits the Bayes decision boundary $\{u: (\theta_{0}^{\top},-1)u = 0 \}$ involving $\theta_{0}$. We use logistic regression to obtain an estimate $(\hat{\beta}^{\top},\hat{b})$ of a multiple of $(\theta_{0}^{\top},-1)$. Then we use $\hat{\theta}_{0}=-\frac{\hat{\beta}}{b}$ as the estimate of $\theta_{0}$. 

After obtaining the estimate $\hat{\theta}_{0}$ of $\theta_{0}$, we step forward to estimate $F$. Denote the corresponding PDF function as $f$. Our final aim would be generating an appropriate estimate of $f$ as a weighted average of normal distribution PDFs. Then we are able to construct a suitable closed form of an estimate $\hat{F}$ of $F$ as a weighted average of normal distribution CDFs. The data closely related to the behavior of $F$ can be formulated as $\{(p_{t}-x_{t}^{\top}\theta_{0},y_{t})\}_{1\leq t\leq T_{0}}$. However, we do not know the true values of $p_{t}-x_{t}^{\top}\theta_{0}$ since $\theta_{0}$ is unknown. By substituting $\hat{\theta}_{0}$ for $\theta_{0}$, we obtain the approximate dataset $\{(p_{t}-x_{t}^{\top}\hat{\theta}_{0},y_{t})\}_{1\leq t\leq T_{0}}$. Since $y_{t} \sim \text{Ber}(1-F(p_{t}-x_{t}^{\top}\theta_{0}))$, we have for any $v\in \mathbb{R}$, \begin{equation*}
\begin{aligned}
F(v) &= \mathbb{E}(1_{\{y_{t}=0 \}}\ |\ p_{t}-x_{t}^{\top}\theta_{0} = v) \\
&\approx \mathbb{E}(1_{\{y_{t}=0 \}}\ |\ p_{t}-x_{t}^{\top}\hat{\theta}_{0} = v)\\
&\approx \mathbb{E}(1_{\{y_{t}=0 \}}\ |\ p_{t}-x_{t}^{\top}\hat{\theta}_{0}\in [v-w,v+w])\\
&\approx \frac{\sum_{t=1}^{T_{0}}1_{\{y_{t} = 0,p_{t}-x_{t}^{\top}\hat{\theta}_{0}\in[v-w,v+w]  \}}}{\sum_{t=1}^{T_{0}}1_{\{p_{t}-x_{t}^{\top}\hat{\theta}_{0}\in[v-w,v+w] \}}} = \tilde{F}_{w}(v),\\
\end{aligned}  
\end{equation*}
where $w$ is some pre-specified window size parameter. Thus by this approximation strategy we are able to estimate the value of $F$ at the point $v$. Then we use the difference quotient $\frac{\tilde{F}_{w}(v+w)-F_{w}(v-w)}{2w}$ as an estimate $\tilde{f}_{w}(v)$ of $f$ at the point $v$. We find that most of the data in $\{(p_{t}-x_{t}^{\top}\hat{\theta}_{0},y_{t})\}_{1\leq t\leq T_{0}}$ belong to the range $[-5,30]$. Thus we formulate a discrete set $V = \{v_{1},\dots,v_{7} \}$ on this range where $v_{i} = -7.5 + 5i$. Here we use a moderate window size $w = 2$ with respect to the dataset size and the length of the objective range to obtain the satisfactory estimates $\text{Est}(V) = \{\tilde{f}_{w}(v_{1}),\dots, \tilde{f}_{w}(v_{7})) \}$ of $f$ on $V$. The left image in Figure \ref{supp_fig:1} shows the discrete estimates $\text{Est}(V)$. Based on these discrete estimates, we are able to formulate a closed form estimate of $f$ on the entire real line by an approach similar in idea with the kernel density estimation. Namely, we can construct the estimate $\hat{f}_{\text{raw}} = \sum_{i=1}^{7}\frac{\tilde{f}_{w}(v_{i})}{\sum_{j=1}^{7}\tilde{f}_{w}(v_{j})}\phi(v_{i},\sigma^{2})$ where $\phi(v_{i},\sigma^{2})$ is the PDF function of $N(v_{i},\sigma^{2})$ distribution. Here we use a choice of $\sigma = 3$ to obtain an estimate $\hat{f}_{\text{raw}}$ with appropriate smoothness. 

\begin{figure}[h]
	\centering
	\subfigure{
		\begin{minipage}[t]{0.32\linewidth}
			\centering
			\includegraphics[width=2.05in]{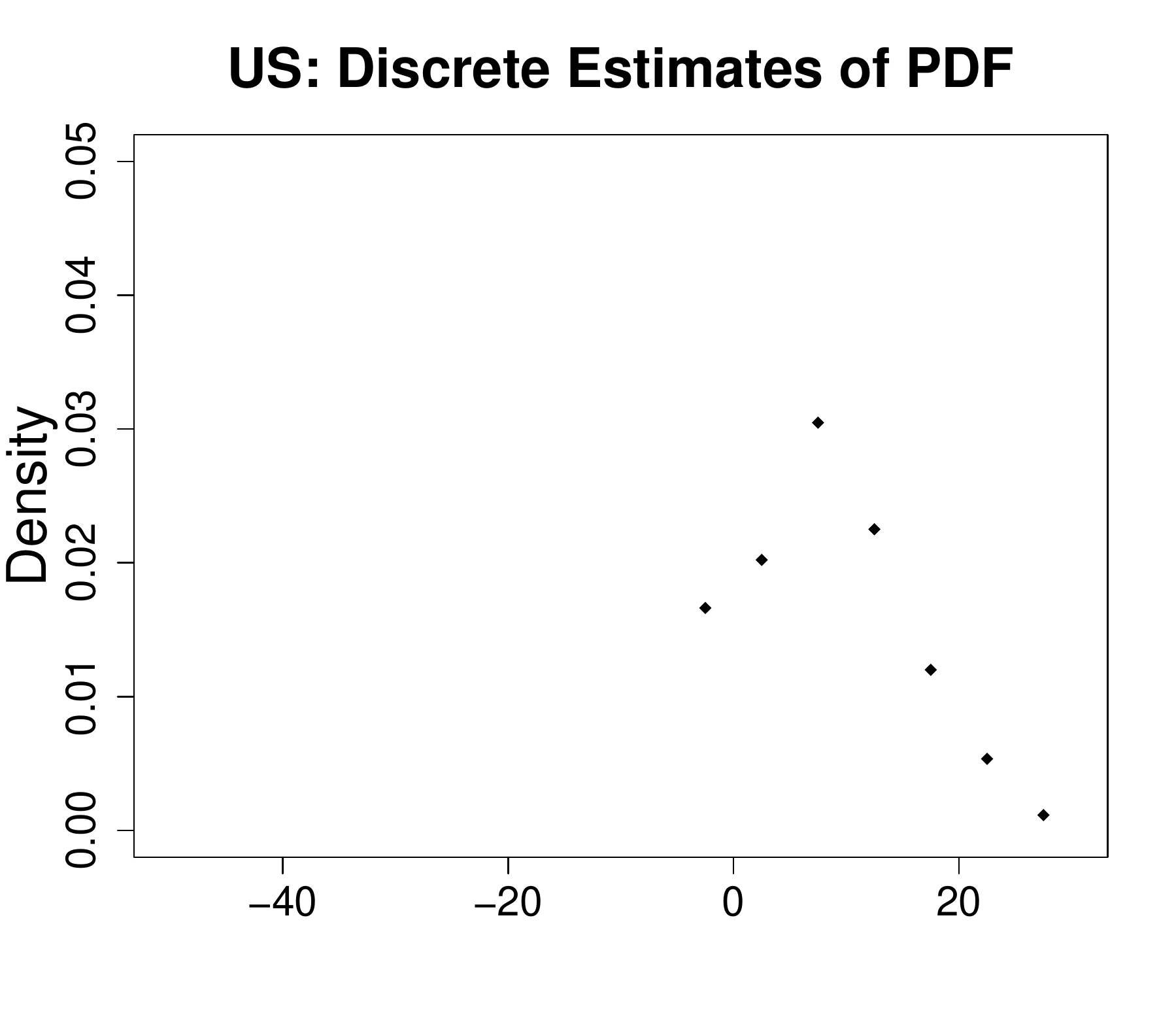}
		\end{minipage}%
	}%
	\subfigure{
		\begin{minipage}[t]{0.32\linewidth}
			\centering
			\includegraphics[width=2.05in]{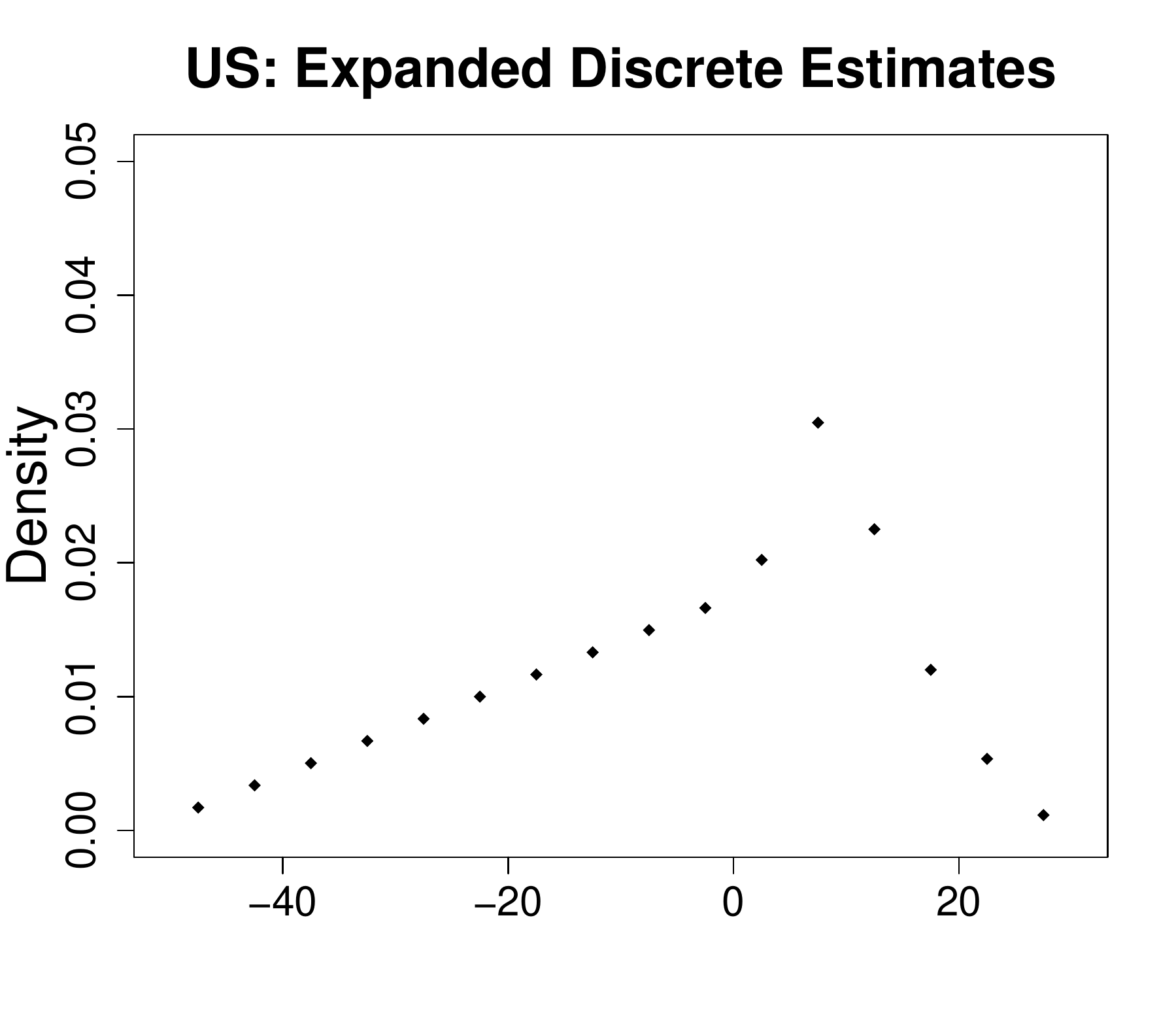}
		\end{minipage}%
	}%
	\subfigure{
		\begin{minipage}[t]{0.32\linewidth}
			\centering
			\includegraphics[width=2.05in]{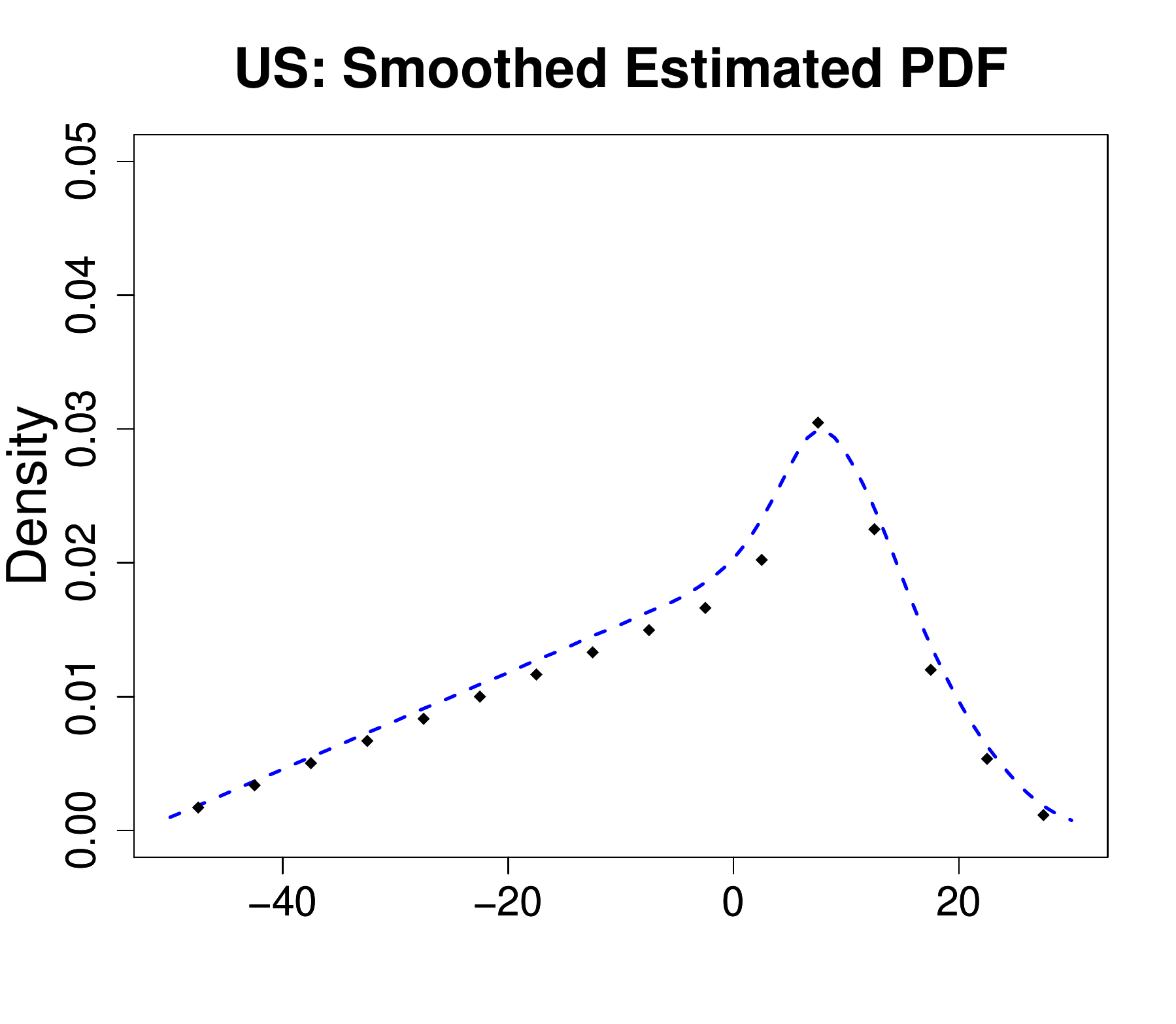}
		\end{minipage}%
	}%
	\centering
	\caption{Estimation procedure of the noise PDF for the US data.}
	\label{supp_fig:1}
\end{figure}

Note that the true $F$ satisfies $F^{-1}(0) = \frac{1}{2}$ and thus has the same mass on both sides of $0$. However, by our choice of $V\subseteq [-5,30]$ and the kernel smoothing approach, $\hat{f}_{\text{raw}}$ places much more mass on the positive side of $0$. The simplest way to remedy this problem is to involve more discrete estimates of the pdf $f$ on the negative side of $0$. This helps make the sum of discrete estimates $\tilde{f}_{w}(v)$ on both sides of $0$ equal. Then by the nature of our kernel smoothing approach, the resulting smoothed estimate $\hat{f}$ yields an $\hat{F}$ approximately satisfying $\hat{F}^{-1}(\frac{1}{2}) = 0$. However, using the same way to estimate $f(v)$ for $v\notin [-5,30]$ is not reliable because of high variance and instability caused by lack of data outside the range $[-5,30]$. Thus we simply expand $V,\text{Est}(V)$ into $V^{*},\text{Est}(V^{*})$ in a smooth way. Specifically, let $V^{*} = \{v_{-8},\dots,v_{0},v_{1},\dots,v_{9} \}$ where $v_{i} = -7.5 + 5i$. Then $\tilde{f}_{w}(v_{-8}),\dots,\tilde{f}_{w}(v_{0})$ are well chosen for smoothness and to satisfy $\sum_{v_{i}<0}\tilde{f}_{w}(v_{i}) = \sum_{v_{i}>0}\tilde{f}_{w}(v_{i})$. The expanded discrete estimates $\text{Est}(V^{*})$ are shown in the middle image in Figure \ref{supp_fig:1}. Our final estimates of $f$ and $F$ admit the forms
\[
\hat{f} = \sum_{i=-8}^{7}\frac{\tilde{f}_{w}(v_{i})}{\sum_{j=-8}^{7}\tilde{f}_{w}(v_{j})}\phi(v_{i},\sigma^{2}), \hat{F} = \sum_{i=-8}^{7}\frac{\tilde{f}_{w}(v_{i})}{\sum_{j=-8}^{7}\tilde{f}_{w}(v_{j})}\Phi(v_{i},\sigma^{2}),
\]
where $\Phi(v_{i},\sigma^{2})$ is the CDF function of $N(v_{i},\sigma^{2})$ distribution. The right image in Figure \ref{supp_fig:1} shows the final smoothed estimate of the noise PDF. 

In Figure \ref{supp_fig:2}, we present the similar estimation procedure of the noise PDF for the California data.

\begin{figure}[h]
	\centering
	\subfigure{
		\begin{minipage}[t]{0.32\linewidth}
			\centering
			\includegraphics[width=2.05in]{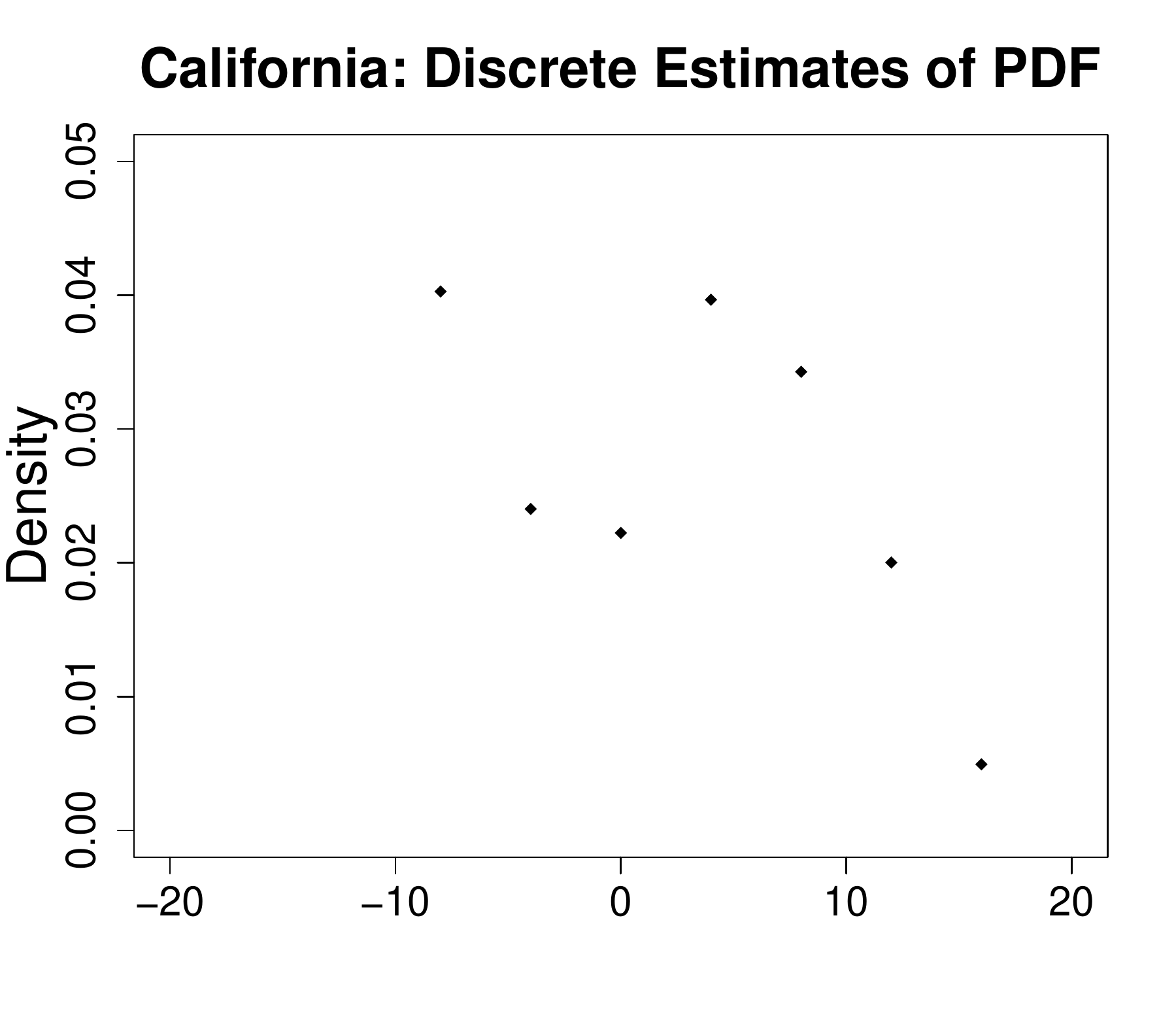}
		\end{minipage}%
	}%
	\subfigure{
		\begin{minipage}[t]{0.32\linewidth}
			\centering
			\includegraphics[width=2.05in]{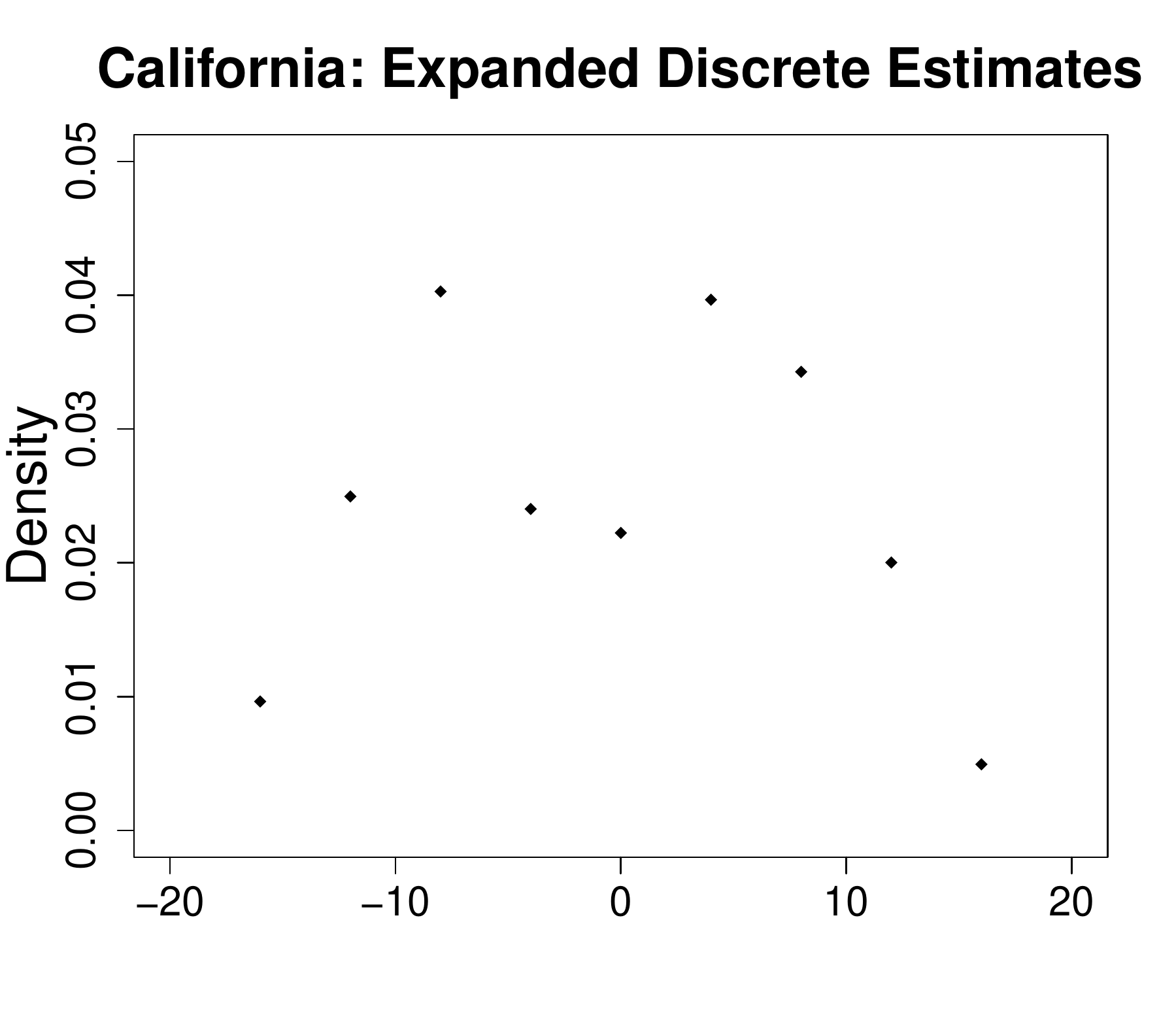}
		\end{minipage}%
	}%
	\subfigure{
		\begin{minipage}[t]{0.32\linewidth}
			\centering
			\includegraphics[width=2.05in]{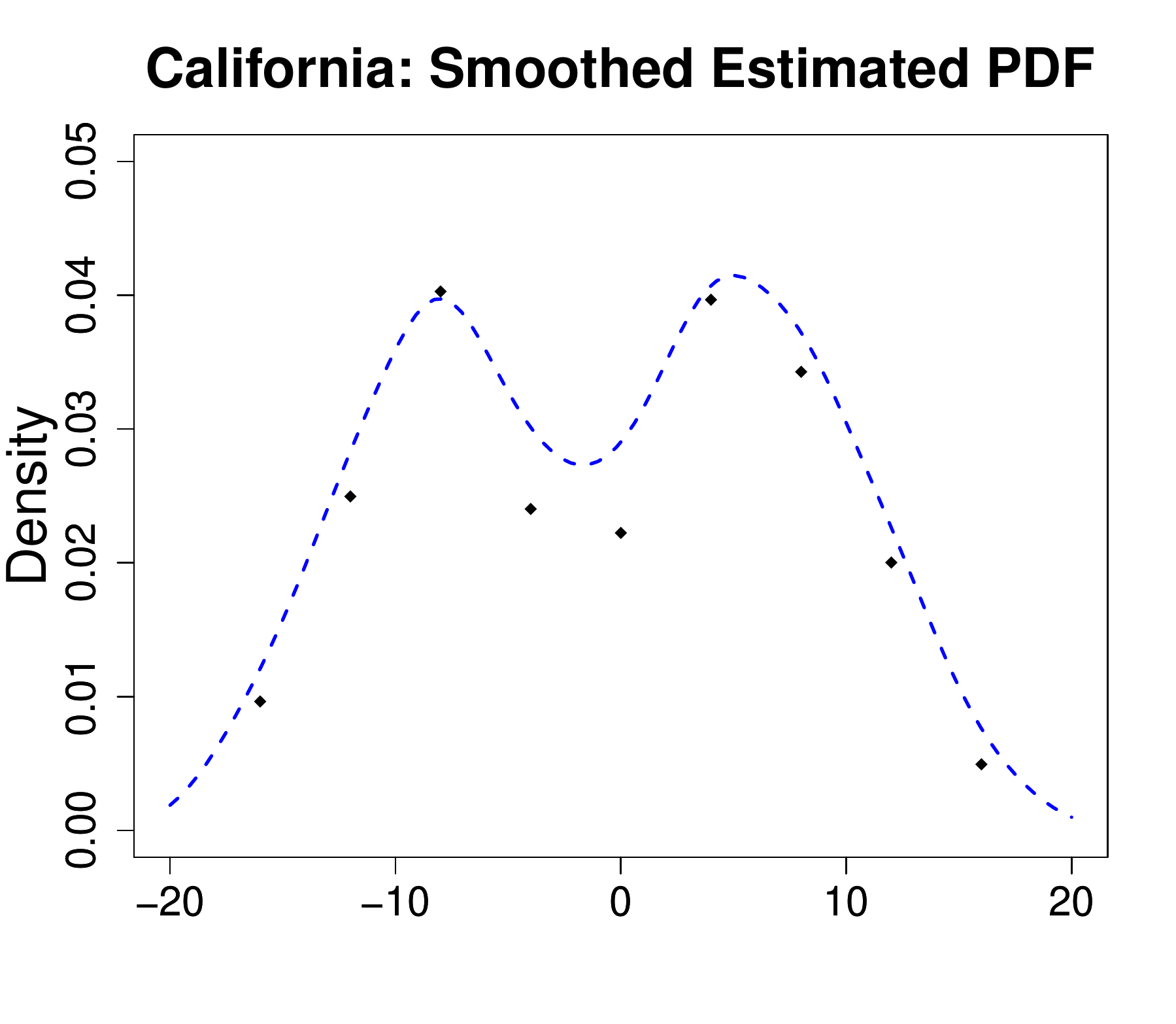}
		\end{minipage}%
	}%
	\centering
	\caption{Estimation procedure of the noise PDF for the California data.}
	\label{supp_fig:2}
\end{figure}

\def\thesection{E}
\section{DIP-SVM: A Variant of DIP}
We propose a variant DIP-SVM of our DIP algorithm in the following Algorithm \ref{supp_alg:6}. Instead of applying logistic regression, it uses Support Vector Machine (SVM) as the linear classifier in Inner Algorithm A to estimate $\theta_{0}$. For the simulation setting in Example 1, we conduct the regret comparison between DIP and DIP-SVM, along with RMLP and RMLP-2. Figure \ref{supp_fig:DIP-SVM} shows that DIP-SVM performs similar to DIP, and both outperform RMLP-2 and RMLP. This suggests that the performance of DIP policy is relatively robust to the choice of the linear classifier. In this paper, we suggest to use logistic regression in DIP policy due to its stable performance and cheap computation.

\setcounter{algorithm}{5}
\begin{algorithm}[h!]
	\caption{DIP-SVM}\label{supp_alg:6}
	\begin{algorithmic}[1]
		\STATE \textbf{Input:} \textbf{(at time 0)} $\alpha_{1},\alpha_{2},p_{\max},C,\lambda,W$
		\STATE \textbf{Input:} \textbf{(arrives over time)} covariates $\{x_{t}\}_{t\in[T]}$
		\STATE \textbf{For} time $t=1,\dots,\ell_{1} (= \alpha_{1})$, \textbf{do}
		\STATE \hspace{0.075in} Set a price $p_{t}$ randomly from $(0,p_{\max})$ and receive a binary response $y_{t}$. 
		\STATE \textbf{For} episodes $k=2,3,\dots,n(=n(T,\alpha_{1},\alpha_{2}))$, \textbf{do}
		\STATE \hspace{0.075in} Apply the SVM version of Inner Algorithm A with the input data\\ \hspace{0.075in} $\{(x_{t},p_{t},y_{t}) \}_{\sum_{i=1}^{k-2}\ell_{i}+1\leq t\leq \sum_{i=1}^{k-1}\ell_{i}}$ and $W$ to obtain the estimate $\hat{\theta}_{k-1}$.
		\STATE \hspace{0.075in} Apply Inner Algorithm B on the coming sequential covariates\\
		\hspace{0.075in} $\{x_{t}\}_{\sum_{i=1}^{k-1}\ell_{i}+1\leq t\leq \sum_{i=1}^{k}\ell_{i}}$, with the estimate $\hat{\theta}_{k-1}$, discretization number \\
		\hspace{0.075in} $d_{k} = C\lceil (2^{k-2}\ell_{2})^{\frac{1}{6}}\rceil$ and the UCB construction in (2) with $\beta_{t} = \beta_{t}^{*}$ and $\delta = \frac{1}{2^{k-2}\ell_{2}}$.
	\end{algorithmic}
\end{algorithm}

\begin{figure}[h]
	\centering
	\includegraphics[width=0.5\textwidth]{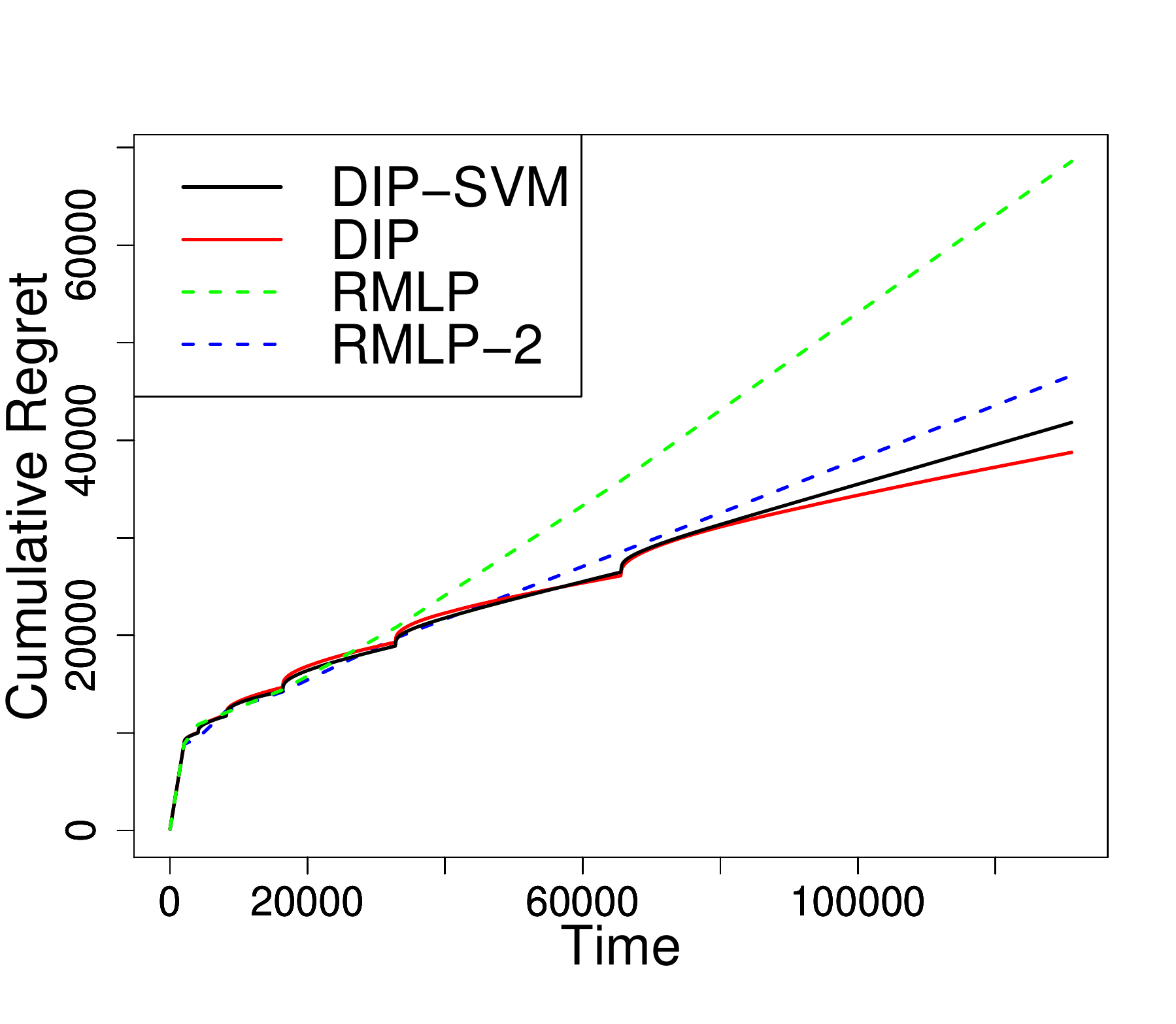}
	\caption{Regret comparison with DIP-SVM.}
	\label{supp_fig:DIP-SVM}
\end{figure}

\def\thesection{F}
\section{Real data Analysis in Four More US States}
We provide additional real data analysis on the Texas, Florida, Pennsylvania and Virginia data. As shown in Figure \ref{supp_fig:four_states}, DIP shows a clear sublinear trend while RMLP and RMLP-2 do not. Specifically, DIP outperforms both RMLP and RMLP-2 for the Texas and Pennsylvania state data just like that in the original cases of the US and California state data. For the Florida and Virginia state data, RMLP-2 performs better than DIP within the limited time horizon. However, considering the increasing trend of RMLP-2 and DIP, we expect that DIP would eventually outperform RMLP-2 as the time horizon gets longer.

\begin{figure}[h]
	\centering
	\subfigure{
		\begin{minipage}[t]{0.24\linewidth}
			\centering
			\includegraphics[width=1.45in]{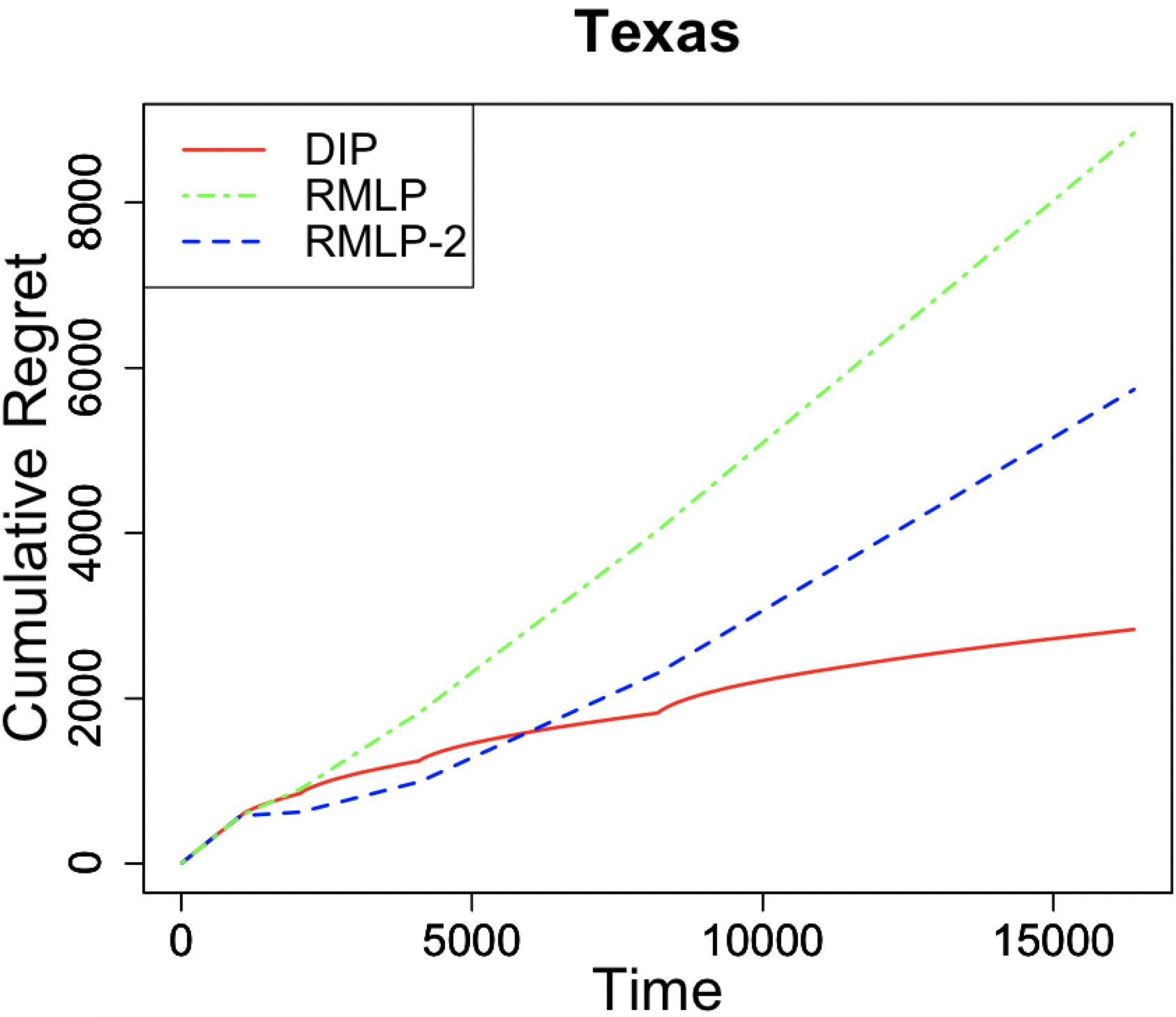}
		\end{minipage}%
	}%
	\subfigure{
		\begin{minipage}[t]{0.24\linewidth}
			\centering
			\includegraphics[width=1.45in]{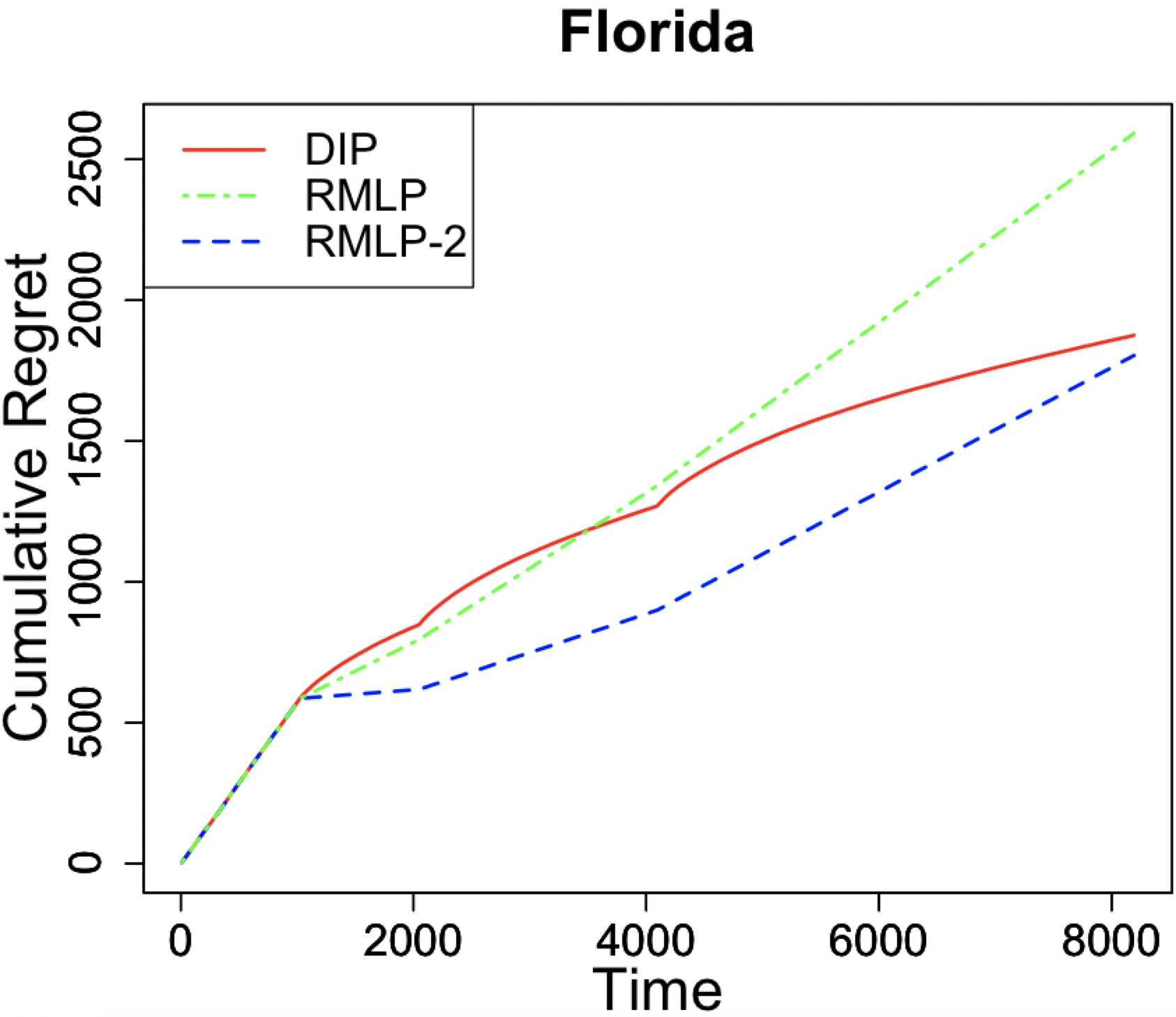}
		\end{minipage}%
	}%
	\subfigure{
		\begin{minipage}[t]{0.24\linewidth}
			\centering
			\includegraphics[width=1.45in]{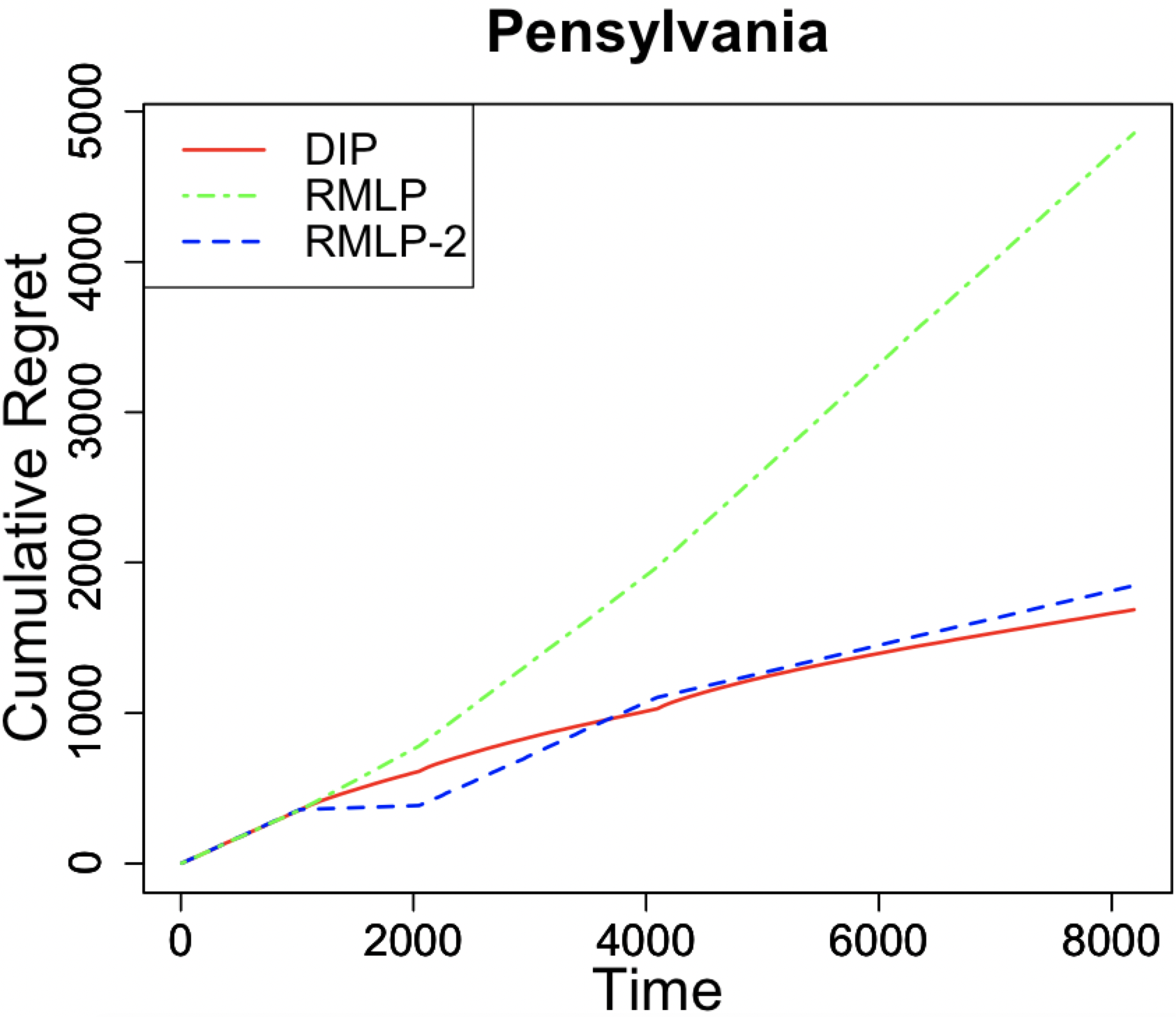}
		\end{minipage}%
	}%
	\subfigure{
		\begin{minipage}[t]{0.24\linewidth}
			\centering
			\includegraphics[width=1.45in]{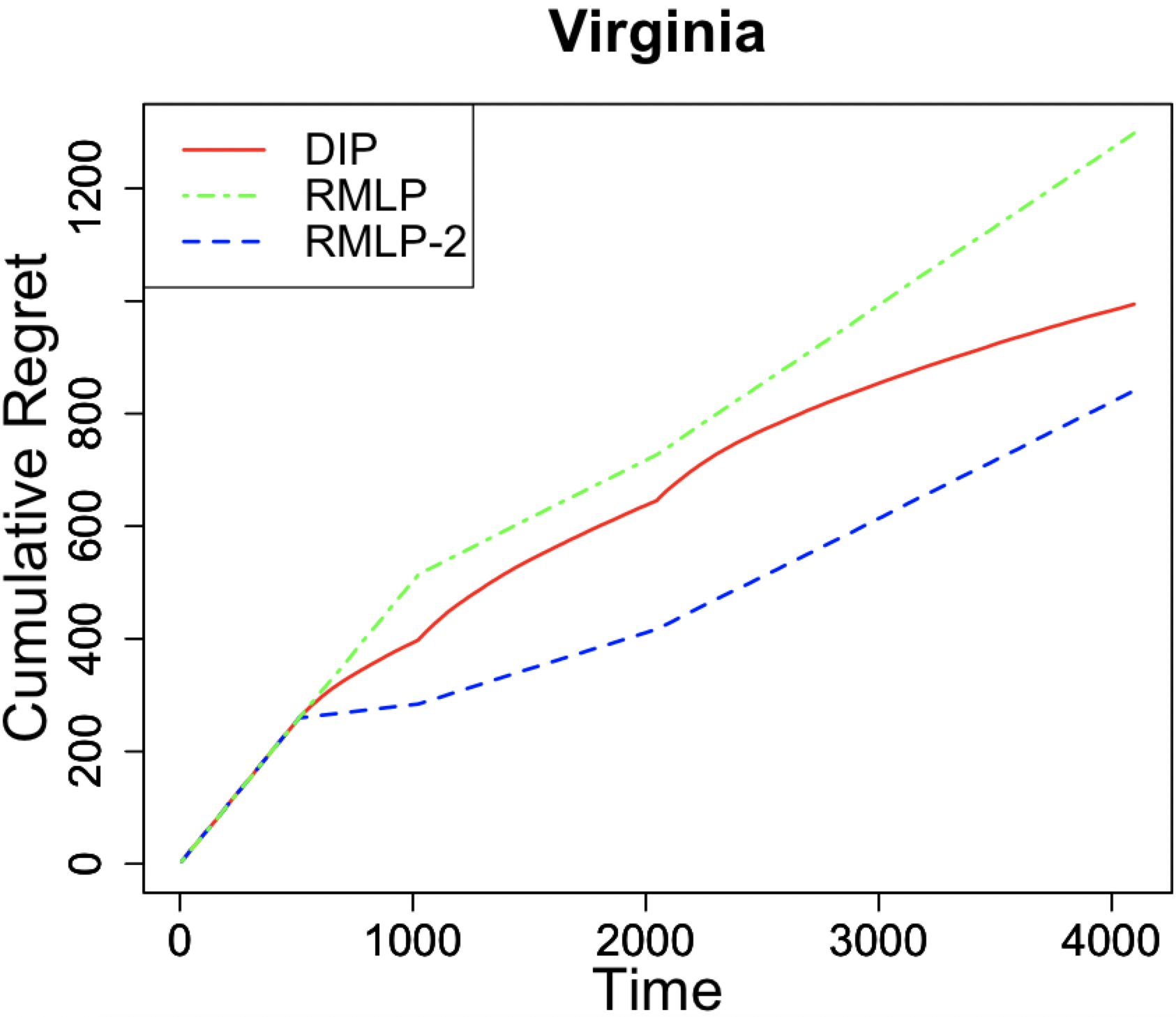}
		\end{minipage}%
	}%
	\centering
	\caption{Additional real data analysis in four more US states.}
	\label{supp_fig:four_states}
\end{figure}

\baselineskip=18pt
\bibliographystyle{biom}
\bibliography{references}

\end{document}